%% file: iclr2024_conference.tex
\documentclass{article} 
\usepackage{iclr2024_conference,times}

\input{math_commands.tex}

\usepackage{microtype}
\usepackage{graphicx}
\usepackage{subfigure}
\usepackage{booktabs} 
\usepackage{enumerate}
\usepackage{arydshln}
\usepackage{threeparttable}
\usepackage{longtable}
\usepackage{tablefootnote}
\usepackage{float}
\usepackage[hidelinks]{hyperref}

\usepackage{hyperref}
\usepackage{url}
\usepackage{amsmath}
\usepackage{amssymb}
\usepackage{mathtools}
\usepackage{amsthm}
\usepackage{scalerel}
\usepackage{bbding}
\usepackage{multirow}
\usepackage{bm}

\usepackage[capitalize,noabbrev]{cleveref}

\theoremstyle{plain}
\newtheorem{theorem}{Theorem}
\theoremstyle{claim}

\theoremstyle{fact}
\newtheorem{fact}{Fact}
\theoremstyle{definition}
\newtheorem{definition}{Definition}
\theoremstyle{lemma}
\newtheorem{lemma}{Lemma}
\usepackage{wrapfig,lipsum,booktabs}
\theoremstyle{proper}
\newtheorem{proper}{Property}
\newtheorem{assumption}[theorem]{Assumption}
\theoremstyle{remark}

\newcommand{\m}[1]{\mathbf{#1}}

\renewcommand{\hat}{\widehat}

\def\T{{ \mathrm{\scriptscriptstyle T} }} 
\newcommand{\tr}{\mathop{\mathrm{tr}}}

\renewcommand{\triangleq}{:=}

\def\##1\#{\begin{align}#1\end{align}}
\def\$ #1\${\begin{align*}#1\end{align*}}
\newcommand{\websiteURL}[0]{https://github.com/sunset-clouds/WassersteinUniformityMetric}

\usepackage{minitoc}

\title{Rethinking The Uniformity Metric in Self-Supervised Learning}

\iclrfinalcopy

\author{Xianghong Fang\\
The Chinese University of Hong Kong, Shenzhen\\
\texttt{fangxianghong2@gmail.com}\\
\And 
Jian Li\\
Tencent AI Lab \\
\texttt{lijianjack@gmail.com}\\
\AND
Qiang Sun\footnotemark[1]\\
University of Toronto \& MBZUAI \\
\texttt{qsunstats@gmail.com} \\
\And
Benyou Wang\thanks{Qiang Sun and Benyou Wang are joint corresponding authors.} \\
The Chinese University of Hong Kong, Shenzhen \& SRIBD \\
\texttt{wangbenyou@cuhk.edu.cn}\\
}

\begin{document}

\maketitle

\begin{abstract}

Uniformity plays an important role in evaluating learned representations, providing insights into self-supervised learning. In our quest for effective uniformity metrics, we pinpoint four principled properties that such metrics should possess. Namely, an effective uniformity metric should remain invariant to instance permutations and sample replications while accurately capturing feature redundancy and dimensional collapse. Surprisingly, we find that the uniformity metric proposed by \citet{Wang2020UnderstandingCR} fails to satisfy the majority of these properties. Specifically, their metric is sensitive to sample replications, and can not account for feature redundancy and  dimensional collapse correctly. To overcome these limitations, we introduce a new uniformity metric based on the Wasserstein distance, which satisfies all the aforementioned properties. Integrating this new metric in existing self-supervised learning methods effectively mitigates dimensional collapse and consistently improves their performance on downstream tasks involving CIFAR-10 and CIFAR-100 datasets. Code is available at \url{https://github.com/statsle/WassersteinSSL}.

\end{abstract}

\doparttoc 
\faketableofcontents 
\part{} 
\vspace{-45pt}

\input{sections/introduction2.tex}
\input{sections/background.tex}

\input{sections/criteria.tex}

\input{sections/theoretical.tex}

\input{sections/method.tex}

\input{sections/empirical.tex}

\input{sections/experiment.tex}

\input{sections/conclusion.tex}

\newpage
\section*{Acknowledgement}

Benyou Wang was partially supported by the Shenzhen Science and Technology Program (JCYJ20220818103001002), Shenzhen Doctoral Startup Funding (RCBS20221008093330065), and Tianyuan Fund for Mathematics of National Natural Science Foundation of China (NSFC) (12326608). Qiang Sun was partially supported in part by the Natural Sciences and Engineering Research Council of Canada under Grant RGPIN-2018-06484 and a Data Sciences Institute Catalyst Grant.  

\bibliography{iclr2024_conference}
\bibliographystyle{iclr2024_conference}

\newpage
\appendix
\input{sections/appendix_new.tex}

\end{document}

%% file: math_commands.tex
\usepackage{amsmath,amsfonts,bm}

\def\eqref#1{equation~\ref{#1}}

\def\1{\bm{1}}

\DeclareMathAlphabet{\mathsfit}{\encodingdefault}{\sfdefault}{m}{sl}
\SetMathAlphabet{\mathsfit}{bold}{\encodingdefault}{\sfdefault}{bx}{n}

\newcommand{\KL}{D_{\mathrm{KL}}}



%% file: sections/introduction2.tex
\section{Introduction}
\label{sec:introduction}

Self-supervised learning excels in acquiring invariant representations to various augmentations~\citep{Chen2020ASF,He2020MomentumCF,Caron2020UnsupervisedLO,Grill2020BootstrapYO,Zbontar2021BarlowTS}. It has been outstandingly successful across a wide range of domains, such as multimodality learning, object detection, and segmentation~\citep{Radford2021LearningTV,li2022supervision,Xie2021DetCoUC,Wang2021DenseCL,Yang2021InstanceLF,Zhao2021ContrastiveLF}. To gain a deeper understanding of self-supervised learning, thoroughly evaluating the learned representations is necessary~\citep{Wang2020UnderstandingCR,Gao2021SimCSESC,Tian2021UnderstandingSL,Jing2021UnderstandingDC}.

\begin{wrapfigure}[11]{r}{0.42\textwidth}
\small
\vspace{-6mm}
\begin{center}
\includegraphics[width=0.40\textwidth]{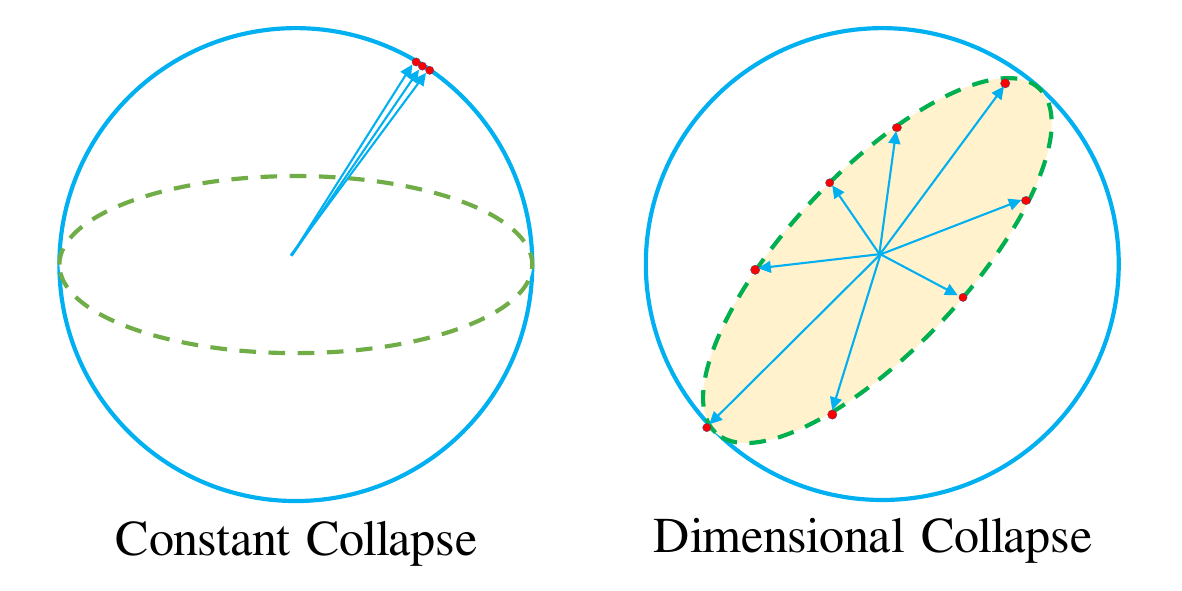}
\vspace{-2mm}
\caption{The left figure presents constant collapse, and the right figure visualizes dimensional collapse.}
\label{fig:two collapse types}
\end{center}
\end{wrapfigure} 

Alignment, a metric quantifying the similarities between positive pairs, holds significant importance in the evaluation of learned representations~\citep{Wang2020UnderstandingCR}. It ensures that positive pairs are mapped to similar features,  making them invariant  to unnecessary details~\citep{Hadsell2006DimensionalityRB,Chen2020ASF}. However, relying solely on alignment proves inadequate for effectively assessing the representations.
This limitation becomes evident in the presence of extremely small alignment values in collapsing solutions, as observed in Siamese networks~\citep{Hadsell2006DimensionalityRB}, where all outputs collapse to a single point~\citep{Chen2021ExploringSS}, as illustrated in Figure~\ref{fig:two collapse types}. In such cases, the learned representations exhibit optimal alignment but fail to provide meaningful information for any downstream tasks. This underscores the necessity of incorporating additional metrics  when evaluating learned representations.

To further evaluate the learned representations, 
\citet{Wang2020UnderstandingCR} formally introduced a \emph{uniformity} metric based on the logarithm of the average pairwise Gaussian potential~\citep{Cohn2006UniversallyOD}. 
Uniformity assesses how feature embeddings are distributed uniformly across the unit hypersphere, and higher uniformity indicates more information from the data is preserved. Since its introduction, uniformity has played a pivotal role in understanding  self-supervised learning and mitigating constant collapse~\citep{Arora2019ATA,Wang2020UnderstandingCR,Gao2021SimCSESC}. Nevertheless, the effectiveness of this particular uniformity metric warrants further examination.

To delve deeper into the existing uniformity metric proposed by \cite{Wang2020UnderstandingCR}, we introduce four principled  properties that an effective uniformity metric should possess. Guided by these properties, we conduct a theoretical analysis, unveiling  key limitations of this metric, particularly its inability to capture feature redundancy and dimensional collapse~\citep{hua2021feature}. Dimensional collapse refers to the scenario where representations occupy a lower-dimensional subspace rather than the entire embedding space \citep{Jing2021UnderstandingDC}; see Figure~\ref{fig:two collapse types}.  We reinforce our theoretical findings with empirical evidence, demonstrating, for instance, the existing metric's inability to differentiate between different degrees of dimensional collapse. Subsequently, we propose a novel uniformity metric based on the quadratic Wasserstein distance that satisfies  all four properties, thereby surpassing the existing one. Finally, integrating  the proposed uniformity metric as an auxiliary loss within existing self-supervised learning methods consistently enhances their performance in downstream tasks.

Our main contributions are summarized as follows. 
(i) We identify four principled properties that an effective uniformity metric should possess, providing new guidelines on designing such metrics. (ii) Surprisingly, we find that the existing uniformity metric \citep{Wang2020UnderstandingCR} fails to meet the majority of these properties. For example, it can not correctly capture dimensional collapse. 
(iii) We propose a new uniformity metric based on the Wasserstein distance that satisfies all four properties, addressing key limitations of the existing metric.
(iv) Our proposed uniformity metric can seamlessly integrate as an auxiliary loss in various self-supervised learning methods, resulting in improved performance in downstream tasks.

%% file: sections/background.tex
\section{Background}
\label{sec:background} 

\subsection{Self-Supervised Representation Learning}\label{sec:ssrl}


Self-supervised learning  leverages the idea  that similar samples should have similar representations that are  invariant to unnecessary details \citep{Wang2020UnderstandingCR}.  For instance, the Siamese network~\citep{Hadsell2006DimensionalityRB} takes as input  positive pairs $(\m x^a, \m x^b)$, often obtained by taking two augmented views of the same sample $\m x$. These positive pairs are then  processed by an encoder network $f$ consisting of a backbone (e.g., ResNet~\citep{He2016DeepRL}) and a projection MLP head~\citep{Chen2020ASF}, yielding representations $(\m z^a = f(\m x^a), \m z^b = f(\m x^b)$\footnote{For simplicity, we also refer to $(\m z^a, \m z^b)$ as positive pairs.}.  To enforce invariance, a natural approach  is to minimize the following alignment loss, defined as the expected distance between positive pairs: 
\begin{equation}\label{eq:align loss}
\mathcal{L}_{\mathcal{A}}
:= \mathbb{E}_{(\m z^a, \m z^b)\sim p_{\textrm{pos}}} \left\Vert {\m z^a_i} - {\m z^b_i}\right\Vert_2^{2},
\end{equation}
where $p_{\textrm{pos}}(\cdot, \cdot)$ is the distribution of positive pairs. 

However, optimizing the above alignment loss alone may lead to an  undesired collapsing solution,  where all representations collapse into a single point, as shown in Figure~\ref{fig:two collapse types}.

\subsection{Existing Solutions to Constant Collapse}
\label{sec:solutions}

To prevent  constant collapse, existing  solutions include contrastive learning, asymmetric model architecture, and redundancy reduction.

\paragraph{Contrastive Learning} 
Contrastive learning offers a potent solution to  mitigate constant collapse. The key idea is to leverage negative pairs.  For example, SimCLR~\citep{Chen2020ASF}  introduced an in-batch negative sampling strategy that utilizes samples within a batch as negative samples. However, its effectiveness is contingent on the use of a large batch size. To address this limitation, MoCo~\citep{He2020MomentumCF} used a memory bank, which stores additional representations as negative samples. 
Recent research endeavors have also explored clustering-based contrastive learning, which combines a clustering objective with contrastive learning techniques~\citep {Li2021PrototypicalCL,Caron2020UnsupervisedLO}.

\paragraph{Asymmetric Model Architecture}
The use of asymmetric model architecture represents another strategy  to combat constant collapse. One plausible explanation for its effectiveness is that such an asymmetric design encourages encoding more information~\citep {Grill2020BootstrapYO}.
To maintain this asymmetry, BYOL~\citep{Grill2020BootstrapYO} introduces the concept of using an additional predictor in one branch of the Siamese network while employing momentum updates and stop-gradient operators in the other branch.  DINO~\citep{Caron2021EmergingPI}, takes this asymmetry a step further by applying it to two encoders, distilling knowledge from the momentum encoder into the other one~\citep{Hinton2015DistillingTK}.  SimSiam~\citep{Chen2021ExploringSS}  removes the momentum update from BYOL, and shows that the momentum update may not be essential in preventing constant collapse. However, Mirror-SimSiam~\citep{Zhang2022HowDS} swaps the stop-gradient operator to the other branch. Its failure challenges the assertion made in SimSiam~\citep{Chen2021ExploringSS} that the stop-gradient operator is the key component for preventing constant collapse.  \citet{Tian2021UnderstandingSL} provides a theoretical examination to elucidate why an asymmetric model architecture can effectively avoid constant collapse.

\paragraph{Redundancy Reduction}
The fundamental principle behind redundancy reduction to mitigate constant collapse is to maximize the information preserved by the representations. The key idea is to decorrelate the learned representations.  
Barlow Twins~\citep{Zbontar2021BarlowTS} aims to achieve decorrelation by focusing on the cross-correlation matrix, while VICReg~\citep{Bardes2021VICRegVR} focuses on the covariance matrix. 
Zero-CL~\citep{zhang2022zerocl} takes a hybrid approach, combining instance-wise and feature-wise whitening techniques.

\subsection{The existing uniformity metric}
\label{sec:collapse analysis}

While the aforementioned solutions effectively prevent constant collapse, they are not as effective in preventing dimensional collapse,  wherein  representations occupy a lower-dimensional subspace instead of the entire  space. 
This phenomenon has been  observed in contrastive learning by visualizing the singular value spectra of representations~\citep{Jing2021UnderstandingDC,Tian2021UnderstandingSL}.

To quantitatively measure the degree of collapse, \cite{Wang2020UnderstandingCR} introduced  a  uniformity loss based on the logarithm of the average pairwise Gaussian potential. Given (normalized) feature representations $\{\m z_1, \m z_2, ..., \m z_n\}$, their proposed empirical uniformity loss is:
\begin{equation}\label{eq:uniform loss}
\mathcal{L_U} \triangleq \log \frac{1}{n(n-1)/2} \sum_{i=2}^{n} \sum_{j=1}^{i-1} e^{-t \left\Vert {\m z_i} - {\m z_j}\right\Vert_2^{2}},
\end{equation}
where $t>0$ is a fixed parameter, often set to $2$. Then $-\mathcal{L}_{\mathcal{U}}$ serves as the corresponding uniformity metric, with a higher value indicating greater uniformity.

We demonstrate in this work that this metric is insensitive to dimensional collapse, both theoretically in Section~\ref{sec:theoretical analysis} and empirically in Section~\ref{sec:empirical analysis}.

%% file: sections/criteria.tex
\section{What makes an effective uniformity metric?}\label{sec:criteria}

In this section, we begin by presenting four fundamental properties that an effective uniformity metric should possess. Leveraging these properties as a lens, we then scrutinize the existing uniformity metric $-\mathcal{L_U}$, shedding light on its limitations.

\subsection{Four Properties for Uniformity}\label{sec:desiderata}

A uniformity metric $\mathcal{U}: {\mathbb{R}^{m}} ^n \rightarrow \mathbb{R}$ is a function that maps a set of learned representations to a scalar indicator of uniformity. In the following section, we introduce four principled properties that  an effective uniformity metric should possess. Let $\mathcal{D}= {\mathbf{z}_1, \ldots, \mathbf{z}_n} \in {\mathbb{R}^{m}} ^n$ represent the learned representations. To avoid the trivial case, we assume that $\mathbf{z}_1, \ldots, \mathbf{z}_n$ are not all equal, meaning that not all points collapse to a single constant point.

First, an effective uniformity metric should be invariant to the permutation of instances, as the distribution of representations should  not  be affected by permutations.

\begin{proper}[Instance Permutation Constraint (IPC)]\label{pro:ipc}

An effective uniformity metric $\mathcal{U}$ should satisfy 
\begin{equation}\label{eq:ipc}
\mathcal{U}(\pi (\mathcal{D})) = \mathcal{U}(\mathcal{D}),
\end{equation}
where $\pi$ is a permutation over the instances.

\end{proper}

Second, an effective uniformity metric should be invariant to instance clones, as instance cloning does not vary the distribution of representations.

\begin{proper}[Instance Cloning Constraint (ICC)]
\label{pro:icc}
An effective uniformity metric $\mathcal{U}$ should satisfy
\begin{equation}\label{eq:icc}
\mathcal{U}(\mathcal{D} \uplus \mathcal{D}) = \mathcal{U}(\mathcal{D}), 
\end{equation}
{where  $\mathcal{D} \uplus \mathcal{D} := \{\m z_1, \m z_2, ..., \m z_n, \m z_1, \m z_2, ..., \m z_n \}$.}
\end{proper}

Third, an effective  uniformity metric should strictly decrease as feature-level cloning for each instance occurs, as this duplication introduces redundancy, which corresponds to dimensional collapse~\citep{Zbontar2021BarlowTS,Bardes2021VICRegVR}.

\begin{proper}[Feature Cloning Constraint (FCC)]\label{pro:fcc}
An effective uniformity metric $\mathcal{U}$ should satisfy
\begin{equation}\label{eq:fcc}
\mathcal{U}(\mathcal{D} \oplus \mathcal{D})  < \mathcal{U}(\mathcal{D}),
\end{equation}
where $\mathcal{D} \oplus \mathcal{D} := \{\m z_1 \oplus \m z_1, \m z_2 \oplus \m z_2, ..., \m z_n \oplus \m z_n\}$ and  $\m z_i \oplus \m z_i : =(z_{i1}, \cdots, z_{im}, z_{i1}, \cdots, z_{im})^{\T} \in \mathbb{R}^{2m}$.
\end{proper}

Fourth, an effective uniformity metric should strictly decrease with the addition of constant features for each instance, as this introduces uninformative and thus redundant features, which again corresponds to dimensional collapse.

\begin{proper}[Feature Baby Constraint (FBC)]\label{pro:fbc}
An effective uniformity metric $\mathcal{U}$ should satisfy
\begin{equation}
\label{eq:fbc}
\mathcal{U}(\mathcal{D} \oplus \m 0^{k})  < \mathcal{U}(\mathcal{D}), \quad k \in \mathbb{N}^+,
\end{equation}
where $\oplus$ is defined in Property \ref{pro:fcc}, that is, $\mathcal{D} \oplus \m 0^{k} = \{\m z_1 \oplus \m 0^{k}, \m z_2 \oplus \m 0^{k}, ..., \m z_n \oplus \m 0^{k}\}$ and $\m z_i \oplus \m 0^{k}=(z_{i1}, z_{i2}, ..., z_{im}, 0, 0, ..., 0)^{\T} \in \mathbb{R}^{m+k}$.
\end{proper}

Intuitively, Properties \ref{pro:ipc} and \ref{pro:icc}  ensure that the uniformity metric should remain insensitive to instance permutations and sample replications, respectively. Meanwhile, Properties \ref{pro:fcc} and \ref{pro:fbc} ensure that feature redundancy and dimensional collapse reduce the uniformity metric, as they make the distribution of the representations less uniform. These four properties constitute intuitive yet principled characteristics of an effective uniformity metric.

%% file: sections/theoretical.tex
\subsection{Examining the uniformity metric $-\mathcal{L_U}$}\label{sec:theoretical analysis}

We employ the four properties introduced earlier to analyze the uniformity metric $-\mathcal{L_U}$ defined in Eqn.~(\ref{eq:uniform loss}). The following theorem summarizes our findings.

\begin{theorem}\label{theorem:lu}
The uniformity metric  $-\mathcal{L_U}$ satisfies {Property~\ref{pro:ipc}}, but  violates {Properties~\ref{pro:icc}}, {\ref{pro:fcc}}, and {\ref{pro:fbc}}.
\end{theorem}

The proof of the above theorem is provided in Appendix~\ref{appendix: claim}. The violation of Property \ref{pro:icc} indicates that the uniformity metric $-\mathcal{L_U}$ is sensitive to sample replications, while the violations of Properties~\ref{pro:fcc} and {\ref{pro:fbc}} suggest that feature redundancy and dimensional collapse do not reduce the uniformity metric $-\mathcal{L_U}$, making this uniformity metric unable to correctly reflect feature redundancy and dimensional collapse. Therefore, there is a pressing need to develop a new uniformity metric.

%% file: sections/method.tex
\section{A New Uniformity Metric}\label{sec:method}

In this section, we introduce a new uniformity metric to address the limitations of $-\mathcal{L}_{\mathcal{U}}$.

\subsection{The uniform spherical distribution  and an approximation} \label{sec:zero-mean isotropic Gaussian distribution}

As pointed out by \citep{Wang2020UnderstandingCR}, feature vectors should be roughly uniformly distributed  on the unit hypersphere $\mathcal{S}^{m-1}$, preserving  as much information of the data as possible. Therefore, we adopt the uniform spherical distribution as our target distribution.

Our approach utilizes the quadratic Wasserstein distance, a form of statistical distance, between the feature distribution and the target distribution as the new uniformity loss. However, computing any statistical distances involving the uniform spherical distribution can be challenging. To address this, we  first establish an asymptotic equivalence between the uniform spherical distribution and the isotropic Gaussian distribution. By adopting a Gaussian distribution for the representations, we then exploit the fact that the quadratic Wasserstein distance between two Gaussian distributions has a closed form involving only the means and covariance matrices, leading to a new and simple uniformity loss. We need the following fact.

\begin{fact}\label{proof:maximum uniformity}
If $ \m Z \sim \mathcal{N}(\m 0, \sigma^2 \m I_m)$, then $\m Y:=\m Z/\Vert \m Z \Vert_2$ is  uniformly distributed on the unit  hypersphere $\mathcal{S}^{m-1}$.
\end{fact}

Because  the average length of $\|\m Z\|_2$ is roughly $\sigma \sqrt{m}$ \citep{chandrasekaran2012convex}, that is,
\begin{align}\label{eq:align loss}
\frac{m}{\sqrt{m+1}} \leq  \Vert \m Z \Vert_2/\sigma \leq \sqrt{m}, \nonumber
\end{align}
we expect that $\m Z/ (\sigma \sqrt{m})\sim \mathcal{N}(\m 0,  \m I_m/m)$ provides a reasonable approximation to $\m Z/ \|\m Z\|_2$, and thus to the uniform spherical distribution.  This is  partially justified by the following theorem.

\begin{theorem}
\label{proof:the kl divergence}
Let $Y_i$ be the $i$-th coordinate of $\m Y=\m Z/\|\m Z \|_2\in \mathbb{R}^m$,  where $\m Z \sim \mathcal{N}(\m 0, \sigma^2 \m I_m)$. Then the quadratic Wasserstein
distance between $Y_i$ and  $\hat{Y}_i \sim \mathcal{N}(0, 1/m)$ converges to zero as $m \to \infty$, that is,
\begin{equation}\label{eq:limit}
\nonumber \lim_{m \to \infty} \mathcal{W}_2(Y_i, \hat Y_i)  = 0.
\end{equation}
\end{theorem}

Theorem \ref{proof:the kl divergence} suggests that $\mathcal{N}(\m 0, \m I_m/m)$ approximates the distribution of each coordinate of the uniform spherical distribution as $m\rightarrow \infty$. 
It can be proven by first  employing the Talagrand $T_2$ inequality \citep{van2016probability} to upper bound the quadratic Wasserstein distance using the  Kullback-Leibler (KL) divergence, and then establishing that the Kullback-Leibler (KL) divergence converges to 0. The proof is provided in Appendix~\ref{app:KL_distance}.




We empirically compare the distributions of $Y_i$ and $\hat Y_i$
across various dimensions $m\in {2, 4, 8, 16, 32, 64, 128, 256}$.
For each $m$, we sample 200,000 data points from both $Y_i$ and $\hat{Y}_i$, bin them into 51 groups, and calculate the empirical KL divergence and Wasserstein distance. Figure~\ref{fig:two distances} plots both distances versus increasing dimensions. We observe that both distances converge to 0 as $m$ increases. Specifically, these results indicate that the distribution of $\widehat Y_i$ provides a reasonable approximation to that of $Y_i$ when $m\geq 2^4 = 16$. Further comparisons between $\m Y$ and $\hat{\m Y}$ can be found in Appendix \ref{sec:further comparisons}.


\begin{figure*}[t]
\vspace{-8pt}
\centering
\subfigure[KL Divergence] { 
\label{fig:kl distance}
\includegraphics[width=0.36\textwidth]{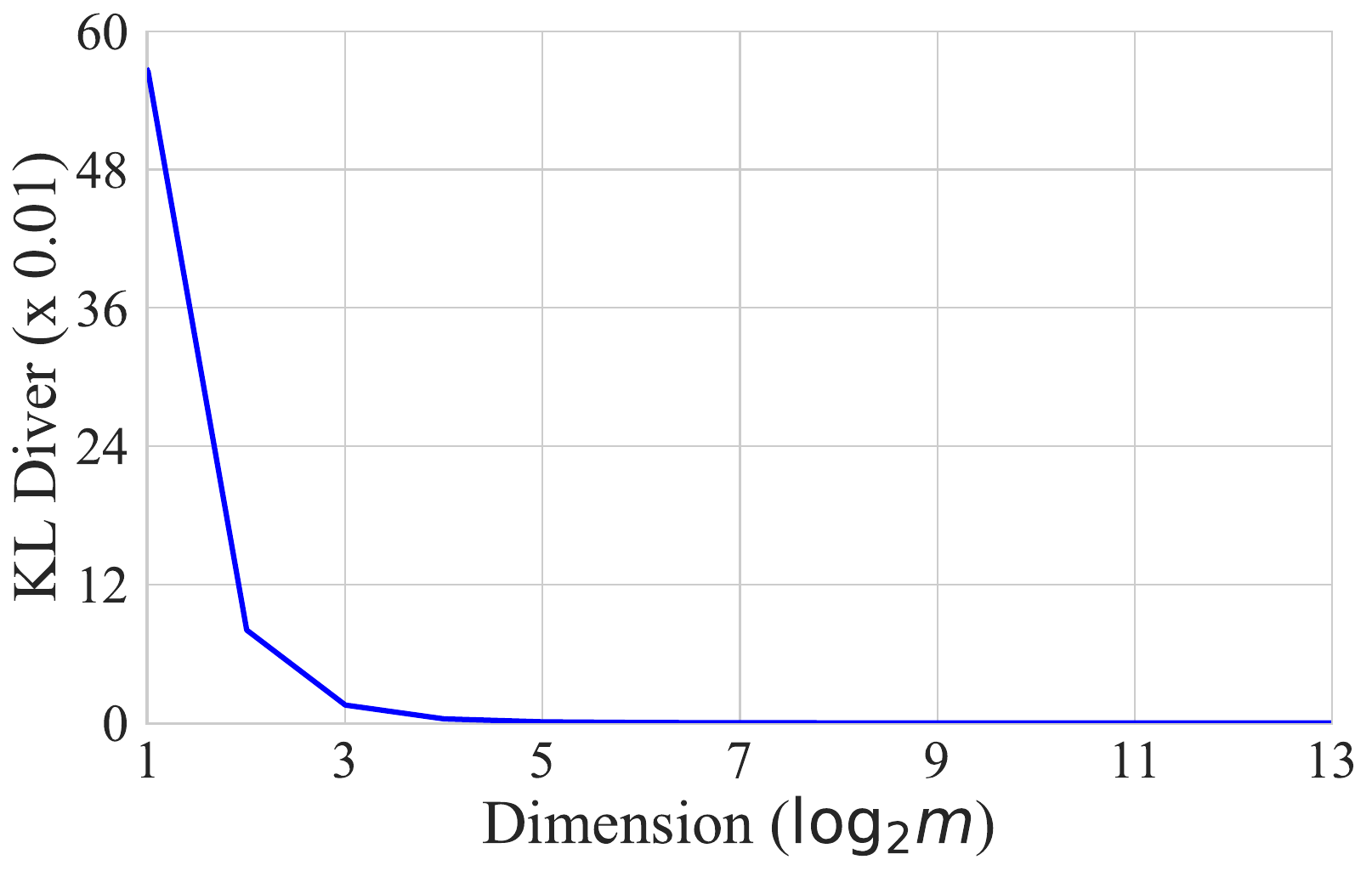}}
\hspace{0.6cm}
\subfigure[Wasserstein Distance]{
\label{fig:wasserstein distance}
\includegraphics[width=0.37\textwidth]{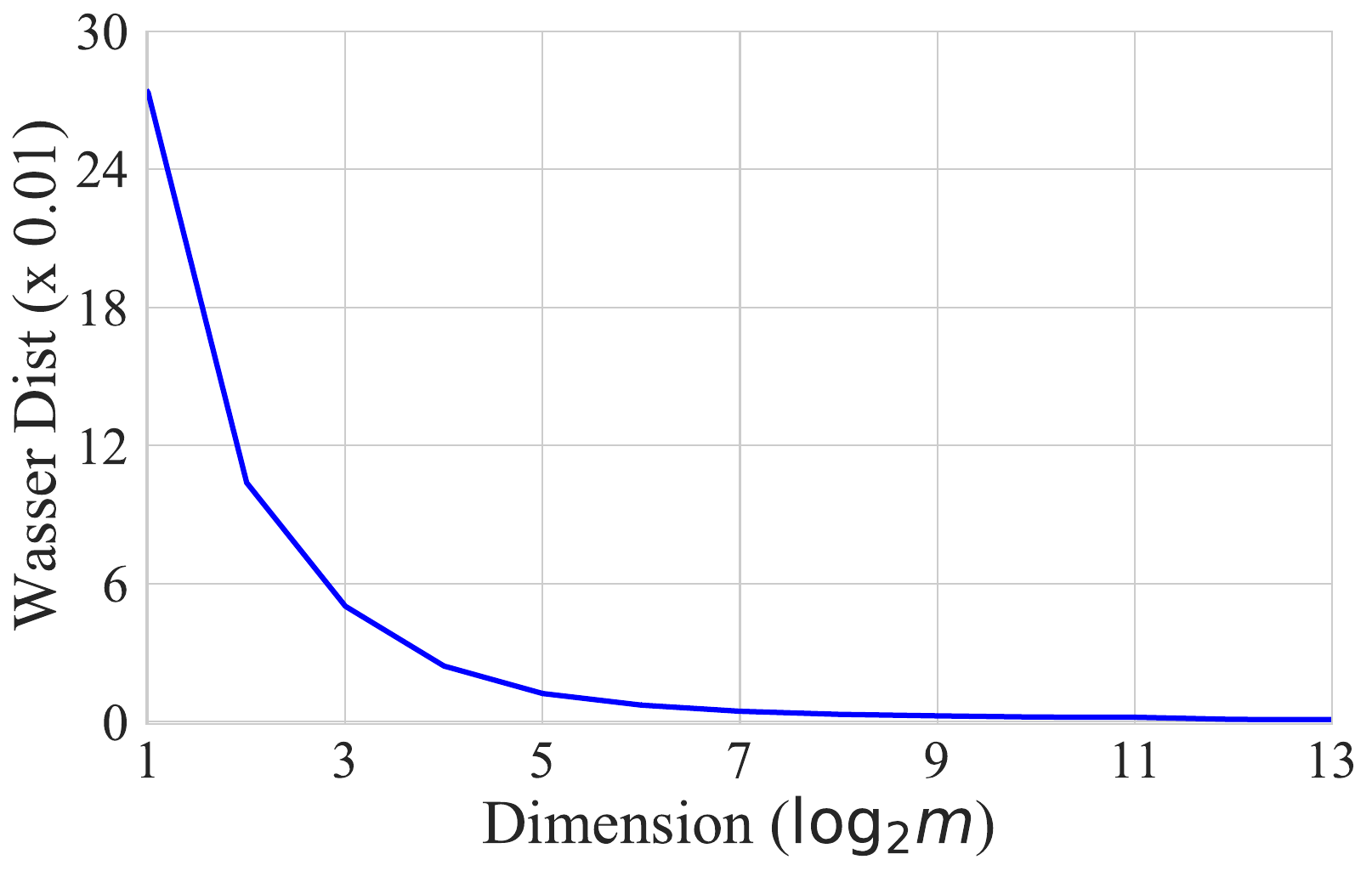}}
\vspace{-8pt}
\caption{The KL divergence and Wasserstein distance between $Y_i$ and $\hat{Y}_i$ w.r.t. various dimensions.}
\label{fig:two distances}
\vspace{-12pt}
\end{figure*}

\subsection{A New Metric for Uniformity} \label{sec:wasserstein Distance}

In this section, we discuss how to use  the quadratic Wasserstein distance between the distribution of learned representations and $\mathcal{N}(\m 0, \m I_m/m)$, in place of the uniform spherical distribution $\textrm{Unif}(\mathcal{S}^{m-1})$,  as our new uniformity loss.

To facilitate  computation, we adopt a Gaussian  hypothesis for the learned representations and assume they follow $\mathcal{N}(\bm \mu, \m \Sigma)$.
With this assumption, we employ the quadratic  Wasserstein distance\footnote{We discuss using other statistical distances as uniformity losses, such as the Kullback-Leibler divergence and Bhattacharyya distance, in Appendix~\ref{sec:other distribution distance}.}
to measure the distance between two distributions. We need the following well-known lemma \citep{Olkin1982TheDB}.

\begin{lemma}\label{theorem:Wasserstein Distance}
 Then the quadratic Wasserstein distance between $\mathcal{N}(\bm \mu, \m \Sigma)$ and $\mathcal{N}(\m 0, \m I/m)$  is
\begin{equation}\label{eq:wasserstein distance definition}
\sqrt{\Vert {\bm \mu} \Vert^2_{2} + 1 + \tr( {\m \Sigma}) -\frac{2}{\sqrt{m}} \tr( {\m \Sigma}^{ \frac{1}{2}})}.
\end{equation}
\end{lemma}

The lemma above indicates that the quadratic Wasserstein distance  can be easily computed using the population mean and covariance of the representations. In practice, we estimate the population mean and covariance by using the sample mean $\hat {\bm \mu} $ and covariance matrix $\hat {\m \Sigma}$, respectively.  Specifically, the empirical quadratic Wasserstein distance serves as the new empirical uniformity loss: 
\begin{equation}\label{eq:wasserstein distance}
\mathcal{W}_{2} \triangleq \sqrt{\Vert \widehat {\bm \mu} \Vert^2_{2} + 1 + \tr(\widehat {\m \Sigma}) -\frac{2}{\sqrt{m}} \tr(\widehat {\m \Sigma}^{ \frac{1}{2}})}.
\end{equation}  
Thus, $-\mathcal{W}_{2}$ can be utilized as the new uniformity metric, with larger values indicating greater uniformity. Moreover, our new uniformity loss can be seamlessly integrated into various existing self-supervised learning methods to enhance their performance.

%% file: sections/empirical.tex
\section{Comparing Two Metrics}
\subsection{Theoretical Comparison}
\label{sec:theoretical comparison}

We examine the proposed metric $-\mathcal{W}_{2}$ in terms of the four properties introduced earlier. The following theorem summarizes our findings.

\begin{theorem}\label{theorem:w2}
The uniformity metric  $-\mathcal{W}_2$ satisfies all four properties, that is,  {Properties~\ref{pro:ipc}--\ref{pro:fbc}}.
\end{theorem}

The proof of the above theorem is similar to that of Theorem \ref{theorem:lu}, and is provided in Appendix~\ref{appendix: claim w2}. Table~\ref{table:theoretical analysis} compares $-\mathcal{L_U}$ and $-\mathcal{W}_2$. It is important to highlight that our new uniformity metric is invariant to instance permutations and sample replications, while effectively capturing feature redundancy and dimensional collapse.

\begin{wraptable}[6]{r}{0.5\textwidth}
\vspace{-20pt}
\caption{Comparing the two uniformity metrics.}
\vspace{4pt}
\centering
\begin{tabular}{lcccc}\hline
Properties & IPC  & ICC & FCC & FBC \\\hline
$-\mathcal{L_U}$ & \CheckmarkBold  & \XSolidBrush 
&\XSolidBrush & \XSolidBrush \\
$-\mathcal{W}_{2}$ & \CheckmarkBold  &  \CheckmarkBold & \CheckmarkBold & \CheckmarkBold\\\hline
\end{tabular}
\label{table:theoretical analysis}
\end{wraptable}

Taking dimensional collapse as an example,  we consider $\mathcal{D} \oplus \mathbf{0}^{k}$ versus $\mathcal{D}$. Here, a larger $k$ indicates a more severe dimensional collapse. However, $-\mathcal{L_U}$ fails to identify this issue, as $-\mathcal{L_U}(\mathcal{D} \oplus \mathbf{0}^{k}) = -\mathcal{L_U}(\mathcal{D})$. In stark contrast, our proposed metric can accurately detect this dimensional collapse, as $-\mathcal{W}_{2}(\mathcal{D} \oplus \mathbf{0}^{k}) < -\mathcal{W}_{2}(\mathcal{D})$.

\subsection{Empirical Comparisons via Synthetic Studies}\label{sec:empirical analysis}

We perform synthetic experiments to investigate the two uniformity metrics. An empirical examination of the correlation between these metrics shows that data points following an isotropic Gaussian distribution exhibit better uniformity compared to those from other distributions; see Appendix~\ref{appendix: supplementary empirical study} for detailed results. Additionally, we generate data vectors from this distribution to enable a thorough comparison between the two metrics.

\begin{figure*}[h]
\vspace{-8pt}
\centering
\subfigure[Sensitivity to collapse degrees for $-\mathcal{L_U}$]{ 
\label{fig:lu collapse level}  
\includegraphics[width=0.42\textwidth]{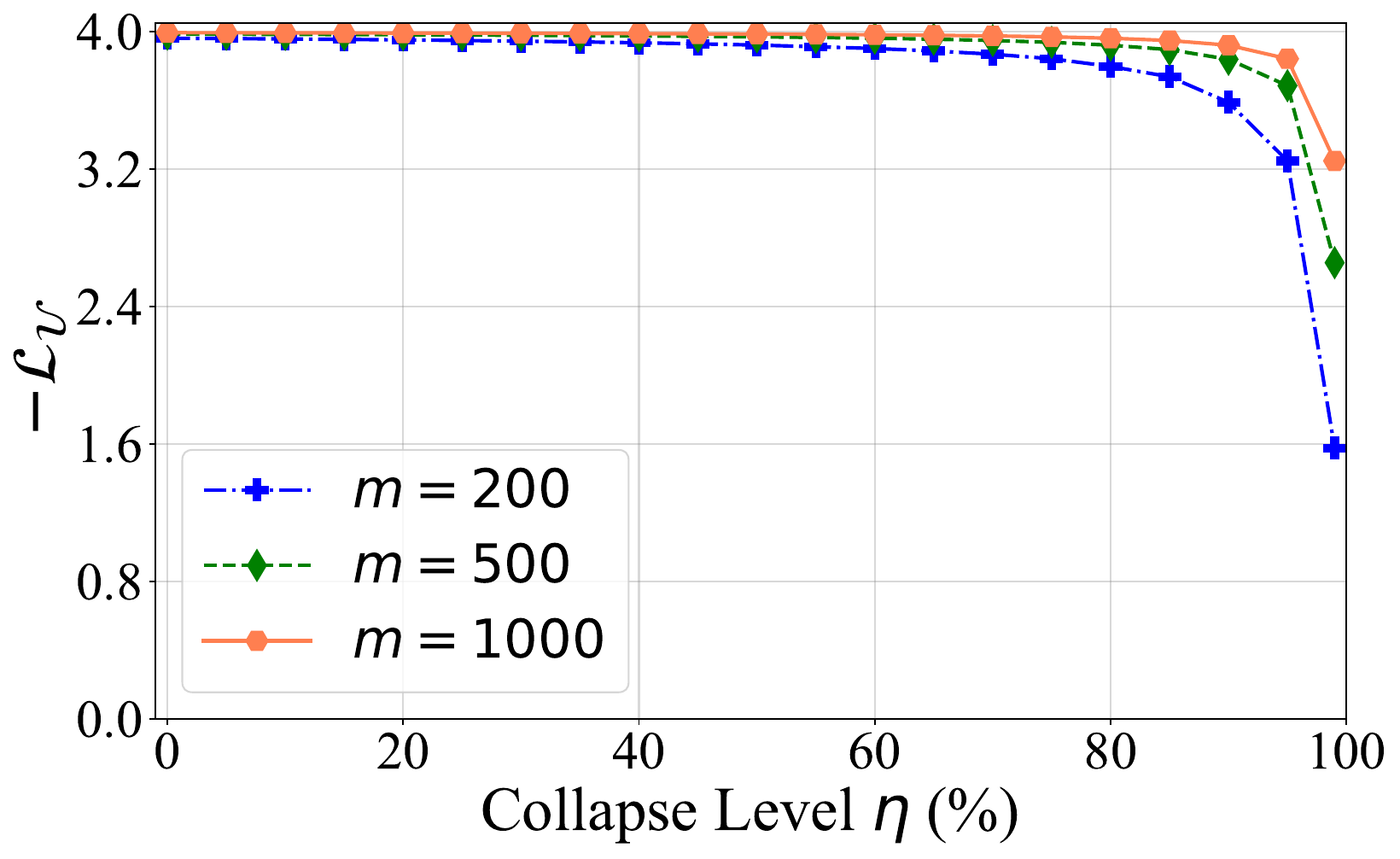}}
\hspace{0.6cm}
\subfigure[Sensitivity to collapse degrees for $-\mathcal{W}_{2}$]{
\label{fig:Wp collapse level}
\includegraphics[width=0.42\textwidth]{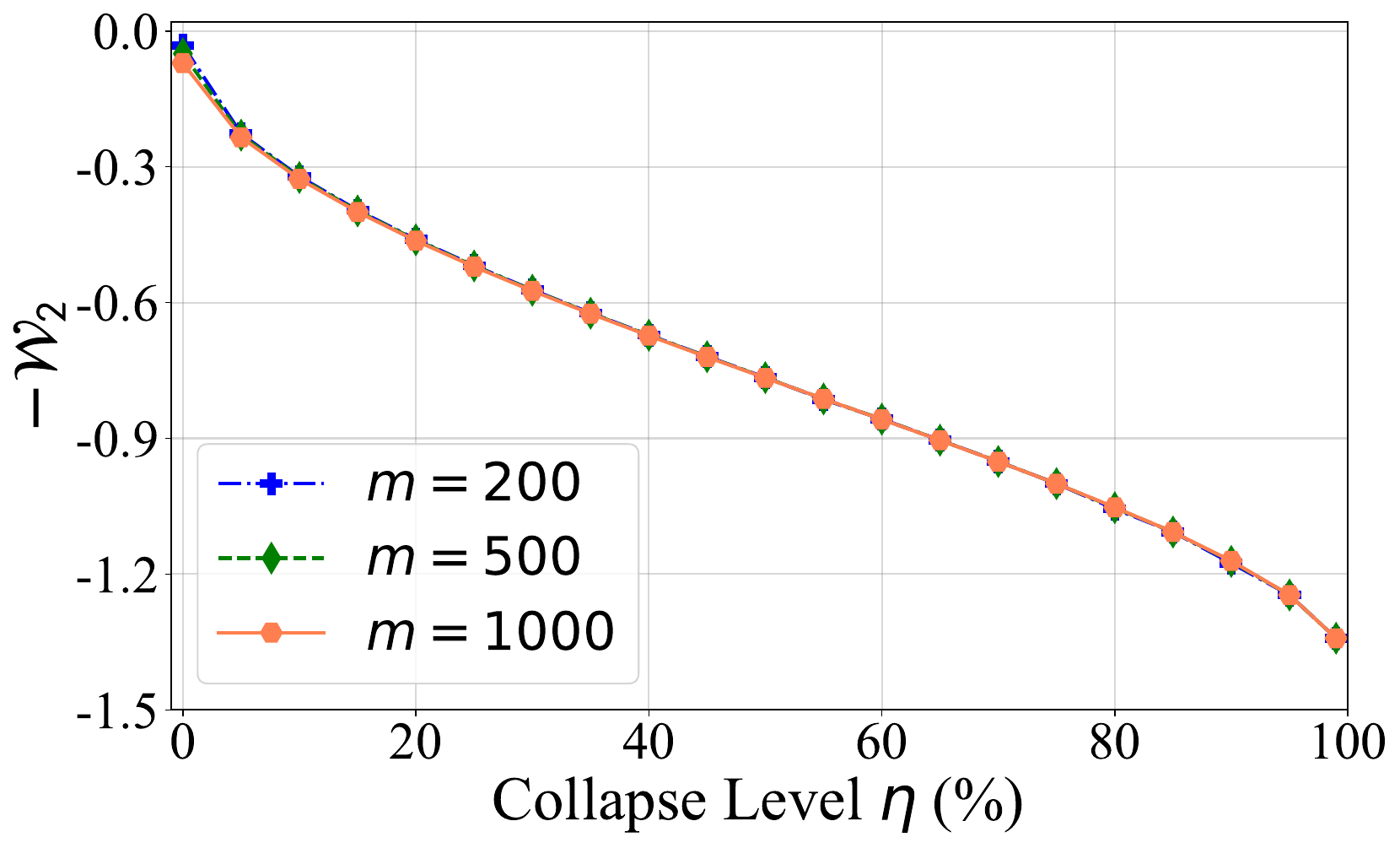}}
\vspace{-4pt}
\caption{{Sensitivity to dimensional collapse degrees: $-\mathcal{W}_{2}$ is more sensitive than $-\mathcal{L_U}$.}}
\label{fig:collapse analysis on various collapse level}
\vspace{-8pt}
\end{figure*}

\paragraph{On Dimensional Collapse  Degrees}
To generate data reflecting varying degrees of dimensional collapse, we sample data vectors from an isotropic Gaussian distribution, normalize them to have $\ell_2$ norms\footnote{In this paper, we always first normalize the representations to have unit $\ell_2$ norms.}, and then zero out a proportion of the coordinates. As the proportion of zero-value coordinates, denoted by $\eta$, increases, dimensional collapse becomes more pronounced, while the proportion of non-zero coordinates is $1-\eta$. In Figure~\ref{fig:lu collapse level} and Figure~\ref{fig:Wp collapse level}, we observe that $-\mathcal{W}_{2}$ effectively captures different collapse degrees, whereas $-\mathcal{L_U}$ remains almost unchanged even with $80\%$ collapse ($\eta=80\%$), indicating that $-\mathcal{L_U}$ is insensitive to the degrees of dimensional collapse.

\begin{figure*}[h]
\vspace{-4pt}
\centering
\subfigure[Effectiveness of $-\mathcal{L_U}$ when increasing $m$] { 
\label{fig:dimension lu}
\includegraphics[width=0.42\textwidth]{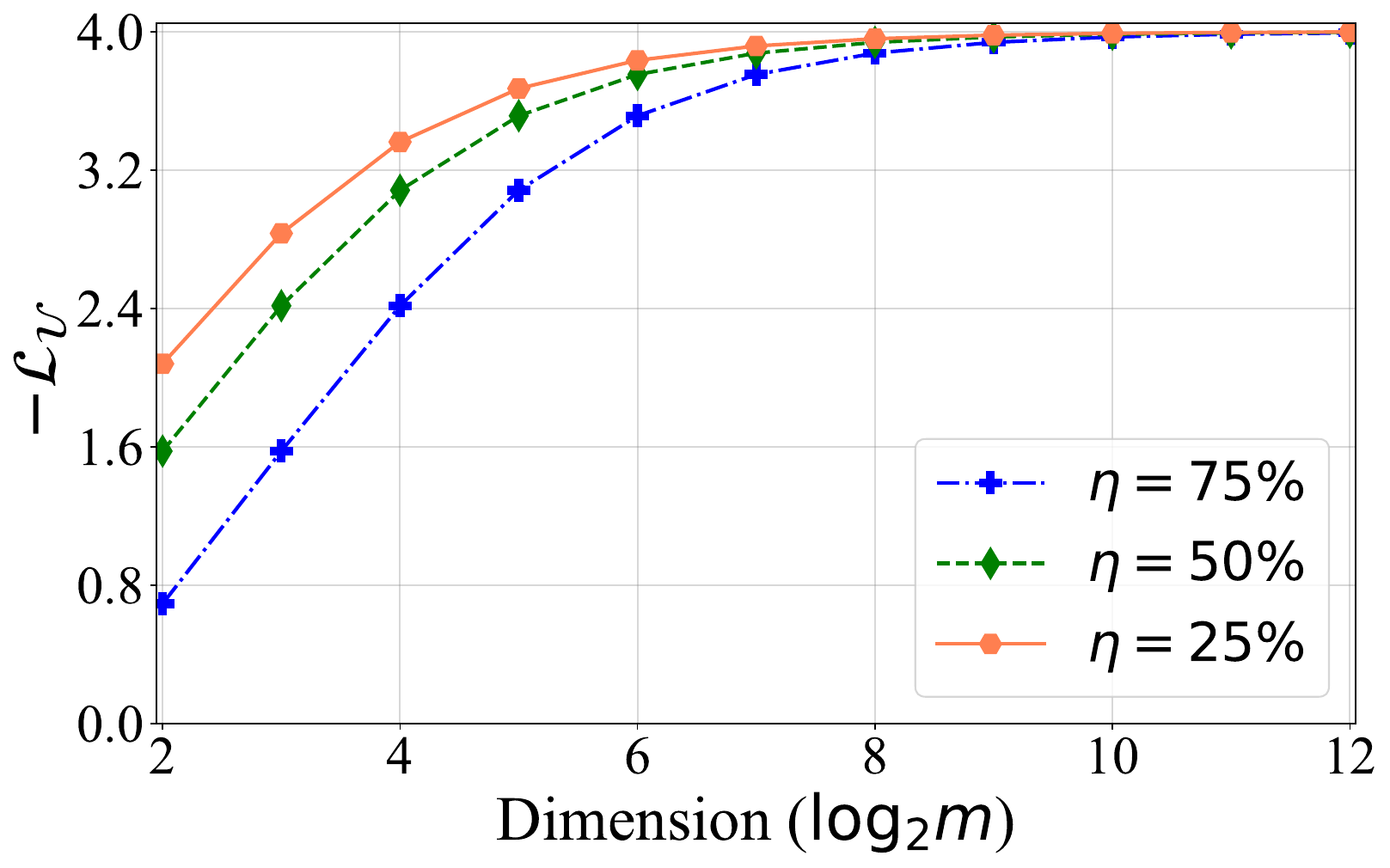}}
\hspace{0.6cm}
\subfigure[Effectiveness of $-\mathcal{W}_{2}$ of when increasing $m$]{
\label{fig:dimension wp}
\includegraphics[width=0.42\textwidth]{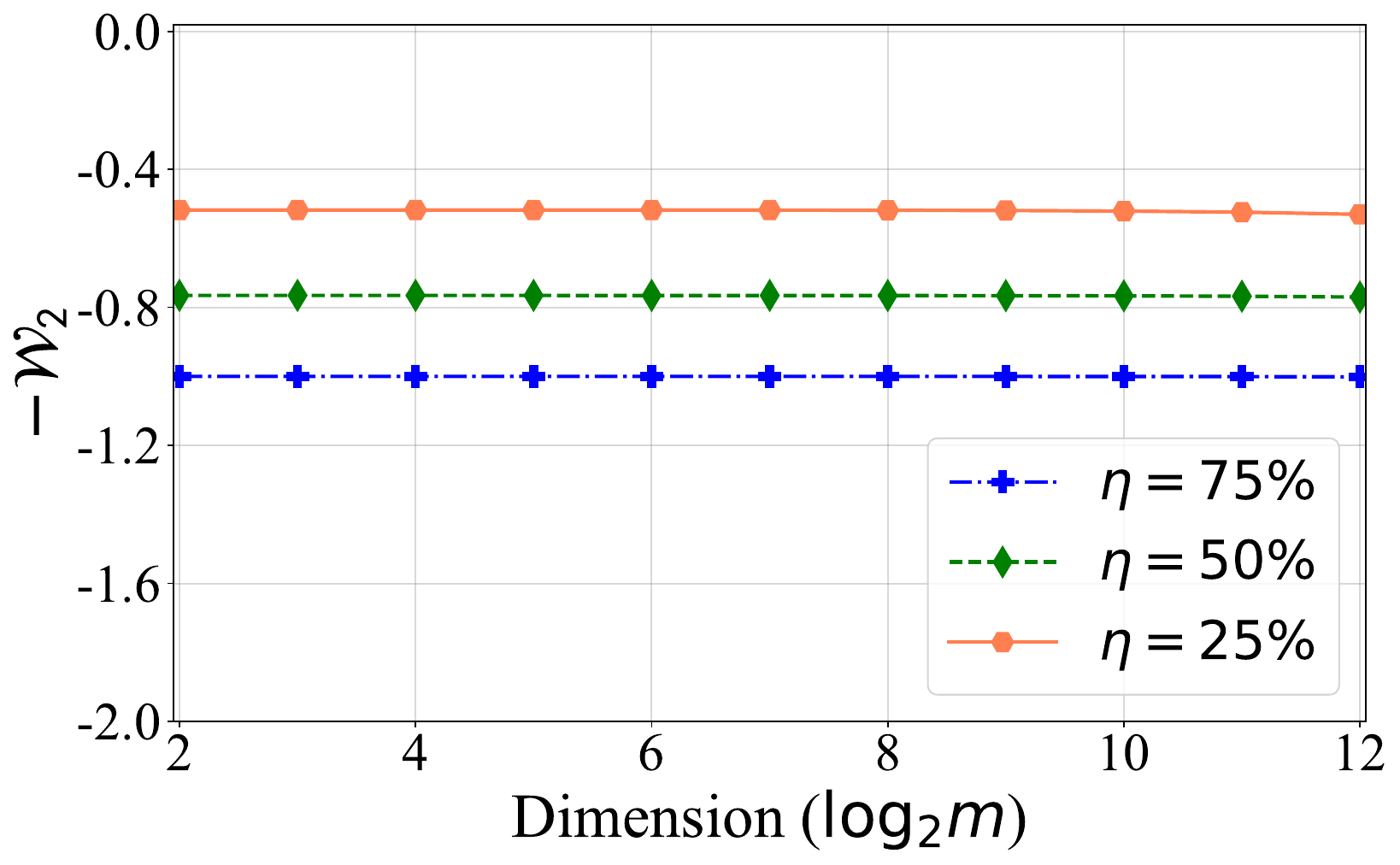}}
\caption{Effectiveness of the metrics when increasing dimension $m$: $-\mathcal{L_U}$ fails to distinguish  different dimensional collapse degrees  for large $m$, while $-\mathcal{W}_{2}$ is always able to.}
\label{fig:collapse analysis on dimension}
\vspace{-8pt}
\end{figure*}

\paragraph{On Sensitiveness of Dimensions}
Figure~\ref{fig:collapse analysis on dimension} demonstrates  that $-\mathcal{L_U}$  can not distinguish between different degrees of dimensional collapse ($\eta = 25\%, 50\%,$ and $75\%$) as the dimension $m$ increases (e.g., $m \geq 2^8 = 256$). 
In contrast, $-\mathcal{W}_{2}$ only depends on the degree of dimensional collapse and is independent of the dimensions $m$.

To complement the theoretical comparisons between the two metrics discussed in Section~\ref{sec:theoretical comparison}, we also conduct empirical comparisons in terms of FCC and FBC. ICC comparisons are collected in Appendix~\ref{appendix: supplementary empirical study}.

\begin{figure*}[t]
\vspace{-6pt}
\centering
\subfigure[$-\mathcal{L_U}$ does NOT satisfy FCC] { 
\label{fig:fcc lu}
\includegraphics[width=0.42\textwidth]{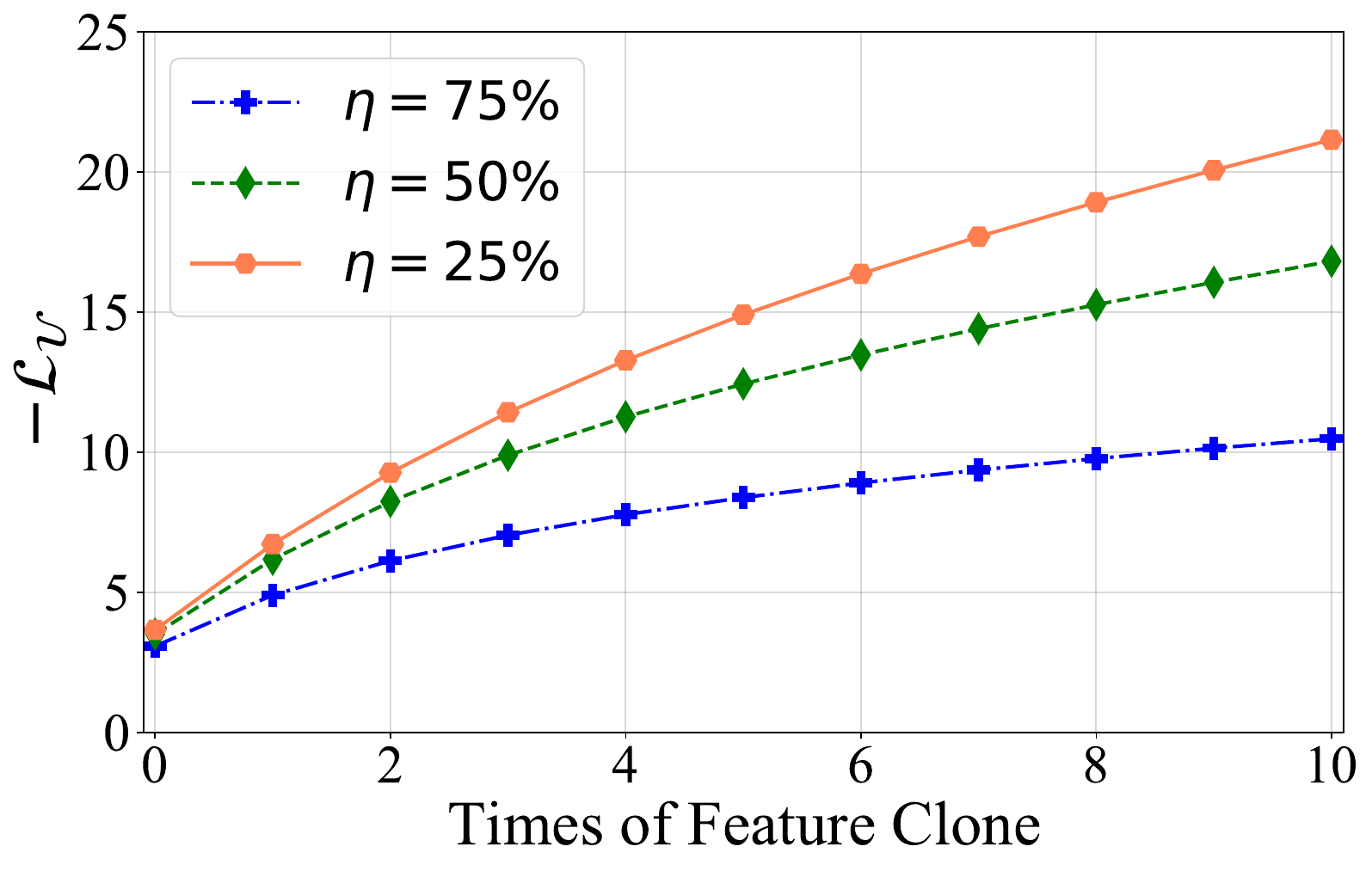}}
\hspace{0.6cm}
\subfigure[$-\mathcal{W}_{2}$ does satisfy FCC]{
\label{fig:fcc wp}
\includegraphics[width=0.42\textwidth]{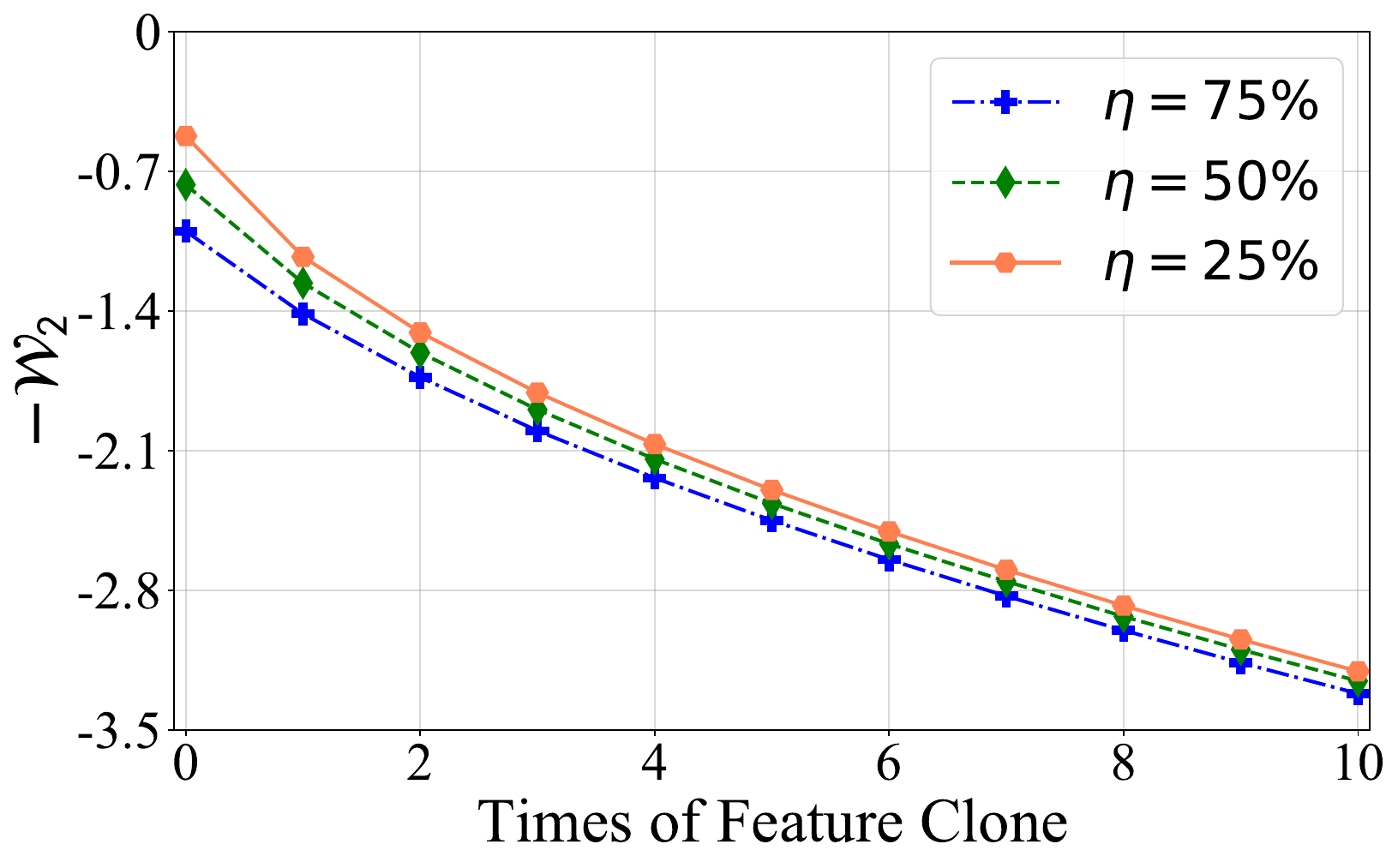}}
\vspace{-8pt}
\caption{{FCC analysis.}}
\label{fig:FCC analysis}
\vspace{-8pt}
\end{figure*}

\paragraph{On Feature Cloning Constraint}
We investigate the impact of feature cloning by creating multiple feature clones of the dataset, such as  $\mathcal{D}\oplus \mathcal{D}$ and $\mathcal{D} \oplus \mathcal{D} \oplus \mathcal{D}$, corresponding to one and two times cloning, respectively. Figure~\ref{fig:fcc lu} demonstrates that the value of $-\mathcal{L_U}$ increases as the number of clones increases, which violates the strict decline in Eqn.~(\ref{eq:fcc}). In contrast, in Figure~\ref{fig:fcc wp}, our proposed metric $-\mathcal{W}_{2}$ decreases, satisfying the property.

\begin{figure*}[h]
\vspace{-6pt}
\centering
\subfigure[$-\mathcal{L_U}$ does NOT satisfy FBC] { 
\label{fig:fbc lu}
\includegraphics[width=0.42\textwidth]{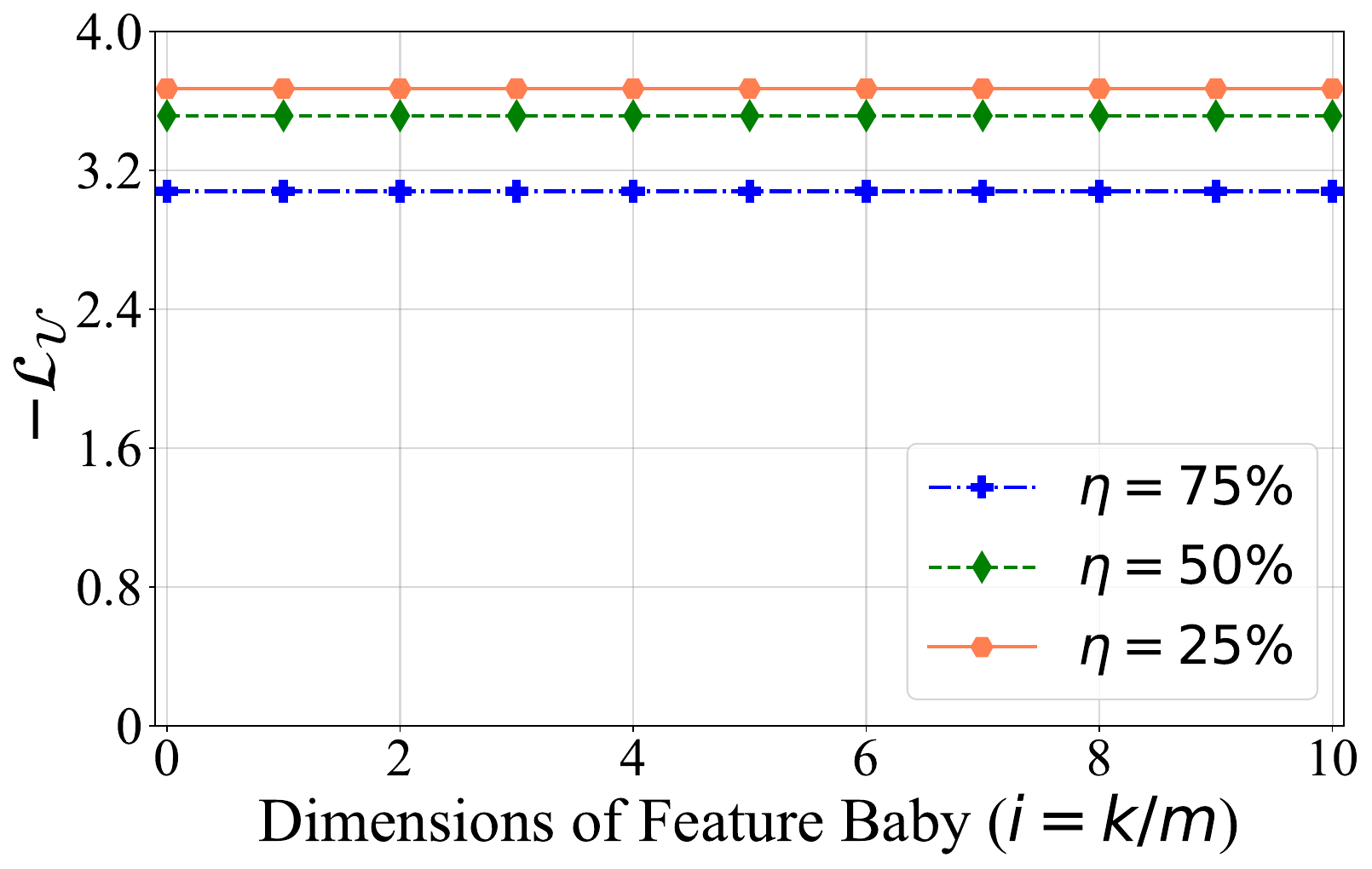}}
\hspace{0.6cm}
\subfigure[$-\mathcal{W}_{2}$ does satisfy FBC]{
\label{fig:fbc wp}
\includegraphics[width=0.42\textwidth]{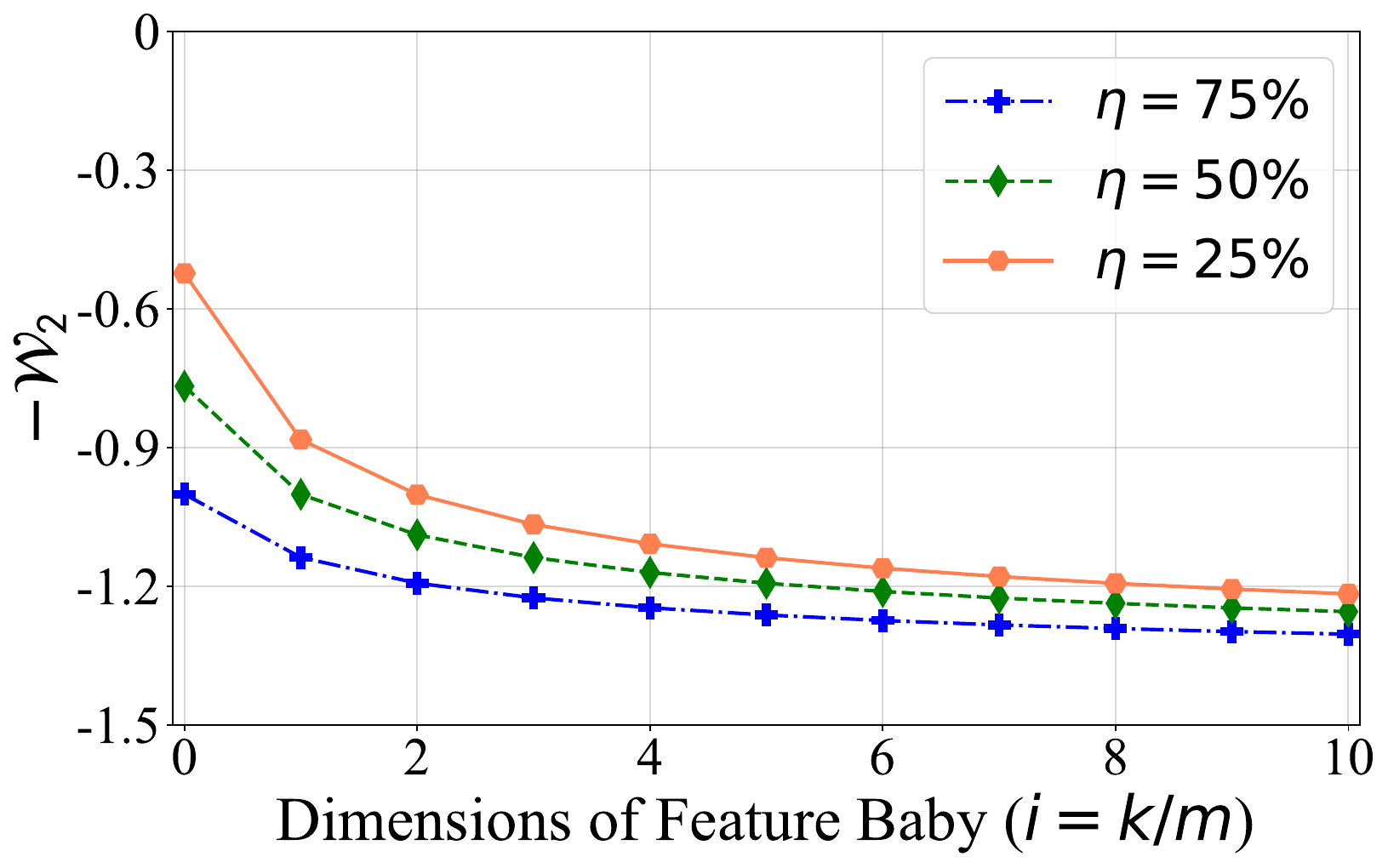}}
\vspace{-8pt}
\caption{{FBC analysis.}}
\label{fig:FBC analysis}
\vspace{-8pt}
\end{figure*}

\paragraph{On Feature Baby Constraint}
We proceed to analyze the effect of feature baby, where we insert $k$ dimensional zero vectors into each instance of  $\mathcal{D}$. This modified dataset is denoted as  $\mathcal{D} \oplus \m 0^{k}$, and we examine the impact of $k$ on both metrics. Figure~\ref{fig:fbc lu} shows that the value of $-\mathcal{L_U}$ remains constant as $k$ increases, violating the strict inequality constraint in Eqn.~(\ref{eq:fbc}). In contrast, Figure~\ref{fig:fbc wp} shows that our proposed metric $-\mathcal{W}_{2}$ decreases, satisfying the constraint.

\paragraph{Summary of Synthetic Studies} 
In summary, our empirical results corroborate our theoretical analysis, confirming that our proposed metric $-\mathcal{W}_{2}$ outperforms the existing metric $-\mathcal{L_U}$ in capturing feature redundancy and dimensional collapse.

%% file: sections/experiment.tex
\section{Experiments}\label{sec:experiment}

In this section, we integrate the proposed uniformity loss as an auxiliary term into various existing self-supervised methods. We then conduct experiments on CIFAR-10 and CIFAR-100 datasets to demonstrate its effectiveness.

\paragraph{Models}
We conduct experiments on a series of self-supervised representation learning models: (i) AlignUniform~\citep{Wang2020UnderstandingCR}, which incorporates both alignment and uniformity losses in its objective function; (ii) three contrastive learning methods, namely SimCLR~\citep{Chen2020ASF}, MoCo~\citep{He2020MomentumCF}, and NNCLR~\citep{Dwibedi2021WithAL}; (iii) two asymmetric models, BYOL~\citep{Grill2020BootstrapYO} and SimSiam~\citep{Chen2021ExploringSS}; (iv) two methods based on redundancy reduction, BarlowTwins~\citep{Zbontar2021BarlowTS} and Zero-CL~\citep{zhang2022zerocl}. To investigate the behavior of the proposed Wasserstein uniformity loss in self-supervised learning, we integrate it as an auxiliary loss into the following models: MoCo v2, BYOL, BarlowTwins, and Zero-CL. Additionally, we propose using linear decay to weight the Wasserstein uniformity loss during training. This is achieved by setting $\alpha_t = \alpha_{\max} - t,(\alpha_{\max} - \alpha_{\min})/T$, where $t$, $T$, $\alpha_{\max}$, $\alpha_{\min}$, and $\alpha_t$ represent the current epoch, maximum epochs, maximum weight, minimum weight, and current weight, respectively. Further details on the experimental settings can be found in Appendix~\ref{sec:parameter setting}.

\paragraph{Accuracy and representation capacity}
We assess the aforementioned methods using two distinct criteria: accuracy and representation quality/capacity. Accuracy is gauged through linear evaluation accuracy, quantified by Top-1 accuracy (Acc@1) and Top-5 accuracy (Acc@5). On the other hand, representation quality/capacity is evaluated using the uniformity losses $\mathcal{L_U}$ and $\mathcal{W}_{2}$, along with the alignment loss $\mathcal{L_A}$.
.

\paragraph{Main Results}
As depicted in Table~\ref{table:main results table}, incorporating $\mathcal{W}_{2}$ as an additional loss consistently yields superior performance compared to models without this loss or those with $\mathcal{L_U}$ as the additional term. Intriguingly, although it marginally compromises alignment, it enhances uniformity and accuracy in downstream tasks. This underscores the effectiveness of $\mathcal{W}_{2}$ as a uniformity loss. Notably, integrating the Wasserstein uniformity loss does not impede training or inference efficiency.

\begin{table*}[h]
\vspace{-2pt}
\centering
\caption{Main results on CIFAR-10 and CIFAR-100. Proj. and Pred. are the hidden dimensions in  projector and predictor. {\color{red} $\uparrow$} and {\color{blue} $\downarrow$} indicates  gains and losses, respectively.}
\vspace{6pt}
\resizebox{0.98\textwidth}{!}{
\begin{tabular}{lcccccccccccc}\hline
\multirow{2}{*}{Methods} & \multirow{2}{*}{Proj.} & \multirow{2}{*}{Pred.} &\multicolumn{5}{|c}{CIFAR-10} & \multicolumn{5}{|c}{CIFAR-100} \\\cline{4-13}
& & & \multicolumn{1}{|c}{Acc@1$\uparrow$} & Acc@5$\uparrow$ & $\mathcal{W}_{2}\downarrow$ & $\mathcal{L_U}\downarrow$ & $\mathcal{L_A}\downarrow$ & \multicolumn{1}{|c}{Acc@1$\uparrow$} & Acc@5$\uparrow$ & $\mathcal{W}_{2}\downarrow$ & $\mathcal{L_U}\downarrow$ & $\mathcal{L_A}\downarrow$ \\\hline
SimCLR & 256 & \XSolidBrush  & 89.85\quad\quad\enspace & 99.78 & 1.04\quad\quad\enspace & -3.75 & 0.47\quad\quad\enspace & 63.43\quad\quad\enspace & 88.97 & 1.05\quad\quad\enspace & -3.75 &  0.50\quad\quad\enspace\\
NNCLR & 256 & 256  & 87.46\quad\quad\enspace & 99.63 & 1.23\quad\quad\enspace & -3.12 & 0.38\quad\quad\enspace  & 54.90\quad\quad\enspace & 83.81 & 1.23\quad\quad\enspace & -3.18 & 0.43\quad\quad\enspace\\
SimSiam & 256 & 256  & 86.71\quad\quad\enspace & 99.67 & 1.19\quad\quad\enspace & -3.33 & 0.39\quad\quad\enspace & 56.10\quad\quad\enspace & 84.34 & 1.21\quad\quad\enspace & -3.29 & 0.42\quad\quad\enspace\\
AlignUniform & 256 &\XSolidBrush & 90.37\quad\quad\enspace & 99.76 & 0.94\quad\quad\enspace & -3.82 & 0.51\quad\quad\enspace & 65.08\quad\quad\enspace  & 90.15 & 0.95\quad\quad\enspace & -3.82 & 0.53\quad\quad\enspace\\\hline
MoCo v2 & 256 & \XSolidBrush & 90.65\quad\quad\enspace & 99.81 & 1.06\quad\quad\enspace & -3.75  & 0.51\quad\quad\enspace & 60.27\quad\quad\enspace & 86.29 & 1.07\quad\quad\enspace & -3.60 & 0.46\quad\quad\enspace \\
MoCo v2 + $\mathcal{L_U}$ & 256 & \XSolidBrush & 90.98 { \color{red} $\uparrow _{0.33}$} & 99.67 & 0.98 { \color{red} $\uparrow _{0.08}$} & -3.82 & 0.53 { \color{blue} $\downarrow _{0.02}$} & 61.21 { \color{red} $\uparrow _{0.94}$} & 87.32 & 0.98 { \color{red} $\uparrow _{0.09}$} & -3.81 & 0.52 { \color{blue} $\downarrow _{0.06}$}\\
MoCo v2 + $\mathcal{W}_{2}$ & 256 & \XSolidBrush  &  91.41 { \color{red} $\uparrow _{0.76}$} & 99.68 & 0.33 { \color{red} $\uparrow _{0.73}$} & -3.84 & 0.63 { \color{blue} $\downarrow _{0.12}$} & 63.68 { \color{red} $\uparrow _{3.41}$} & 88.48 & 0.28 { \color{red} $\uparrow _{0.79}$} & -3.86 & 0.66 { \color{blue} $\downarrow _{0.20}$}\\\hline
BYOL & 256 & 256 & 89.53\quad\quad\enspace & 99.71 & 1.21\quad\quad\enspace& -2.99 & \textbf{0.31}\quad\quad\enspace& 63.66\quad\quad\enspace & 88.81 & 1.20\quad\quad\enspace & -2.87 & \textbf{0.33}\quad\quad\enspace\\
BYOL + $\mathcal{L_U}$ & 256 & \XSolidBrush & 90.09 { \color{red} $\uparrow _{0.56}$} & 99.75 & 1.09 { \color{red} $\uparrow _{0.12}$} & -3.66 & 0.40 { \color{blue} $\downarrow _{0.09}$} & 62.68 { \color{blue} $\downarrow _{0.98}$} & 88.44 & 1.08 { \color{red} $\uparrow _{0.12}$} & -3.70 & 0.51 { \color{blue} $\downarrow _{0.18}$}\\
BYOL + $\mathcal{W}_{2}$ & 256 & 256 & 90.31 { \color{red} $\uparrow _{0.78}$} & 99.77 & 0.38 { \color{red} $\uparrow _{0.83}$} & -3.90 & 0.65 { \color{blue} $\downarrow _{0.34}$} & 65.16 { \color{red} $\uparrow _{1.50}$} & 89.25 & 0.36 { \color{red} $\uparrow _{0.84}$} & -3.91 & 0.69 { \color{blue} $\downarrow _{0.36}$}\\\hline
BarlowTwins & 256 & \XSolidBrush & 91.16\quad\quad\enspace & 99.80 & 0.22\quad\quad\enspace & -3.91  & 0.75\quad\quad\enspace & 68.19\quad\quad\enspace & 90.64 & 0.23\quad\quad\enspace & -3.91 & 0.75\quad\quad\enspace\\
BarlowTwins + $\mathcal{L_U}$ & 256 & \XSolidBrush & 91.38 { \color{red} $\uparrow _{0.22}$} & 99.77 & 0.21 { \color{red} $\uparrow _{0.01}$} & -3.92 & 0.76 { \color{blue} $\downarrow _{0.01}$} & 68.41 { \color{red} $\uparrow _{0.22}$} & 90.99 & 0.22 { \color{red} $\uparrow _{0.01}$} & -3.91 & 0.76 { \color{blue} $\downarrow _{0.01}$}\\
BarlowTwins + $\mathcal{W}_{2}$ & 256 & \XSolidBrush & \textbf{91.43} { \color{red} $\uparrow _{0.27}$} & 99.78 & 0.19 { \color{red} $\uparrow _{0.03}$} & -3.92 & 0.76 { \color{blue} $\downarrow _{0.01}$} & 68.47 { \color{red} $\uparrow _{0.28}$} & 90.64 & 0.19 { \color{red} $\uparrow _{0.04}$} & -3.91 &  0.79 { \color{blue} $\downarrow _{0.04}$} \\\hline
Zero-CL & 256 & \XSolidBrush & 91.35\quad\quad\enspace & 99.74 & 0.15\quad\quad\enspace & \textbf{-3.94} & 0.70\quad\quad\enspace & 68.50\quad\quad\enspace & 90.97 & 0.15\quad\quad\enspace & -3.93 & 0.75\quad\quad\enspace\\
Zero-CL + $\mathcal{L_U}$ & 256 & \XSolidBrush & 91.28 { \color{blue} $\downarrow _{0.07}$} & 99.74 & 0.15\quad\quad\enspace & \textbf{-3.94} & 0.72 { \color{blue} $\downarrow _{0.02}$} & 68.44 { \color{blue} $\downarrow _{0.06}$} & 90.91 & 0.15\quad\quad\enspace & -3.93 & 0.74 { \color{red} $\uparrow _{0.01}$}\\
Zero-CL + $\mathcal{W}_{2}$  & 256 & \XSolidBrush & 91.42 { \color{red} $\uparrow _{0.07}$} & \textbf{99.82} & \textbf{0.14} { \color{red} $\uparrow _{0.01}$} & \textbf{-3.94} & 0.71 { \color{blue} $\downarrow _{0.01}$} & \textbf{68.55} { \color{red} $\uparrow _{0.05}$} &  \textbf{91.02} & \textbf{0.14} { \color{red} $\uparrow _{0.01}$} & \textbf{-3.94} & 0.76 { \color{blue} $\downarrow _{0.01}$}\\\hline
\end{tabular}}
\label{table:main results table}
\end{table*}

\paragraph{Convergence Analysis} 
We evaluate the Top-1 accuracy of these models on CIFAR-10 and CIFAR-100 using the linear evaluation protocol, as described in Appendix~\ref{appendix:convergence analysis}, across different training epochs. Figure~\ref{fig:convergence} illustrates the results. By incorporating $\mathcal{W}_{2}$ as an additional loss for these models, we observe faster convergence compared to the raw models, particularly for MoCo v2 and BYOL, which exhibit significant collapse issues. Our experiments demonstrate that imposing the proposed Wasserstein uniformity metric as an auxiliary penalty loss greatly enhances uniformity but may compromise alignment. We further analyze uniformity and alignment throughout all training epochs in Appendix~\ref{appendix:representation analysis}.

\begin{figure*}[h]
	\centering
        \vspace{-6pt}
        \subfigure[Singular Value Spectra]{ 
        \label{fig:singular value spectrum}  	\includegraphics[width=0.31\textwidth]{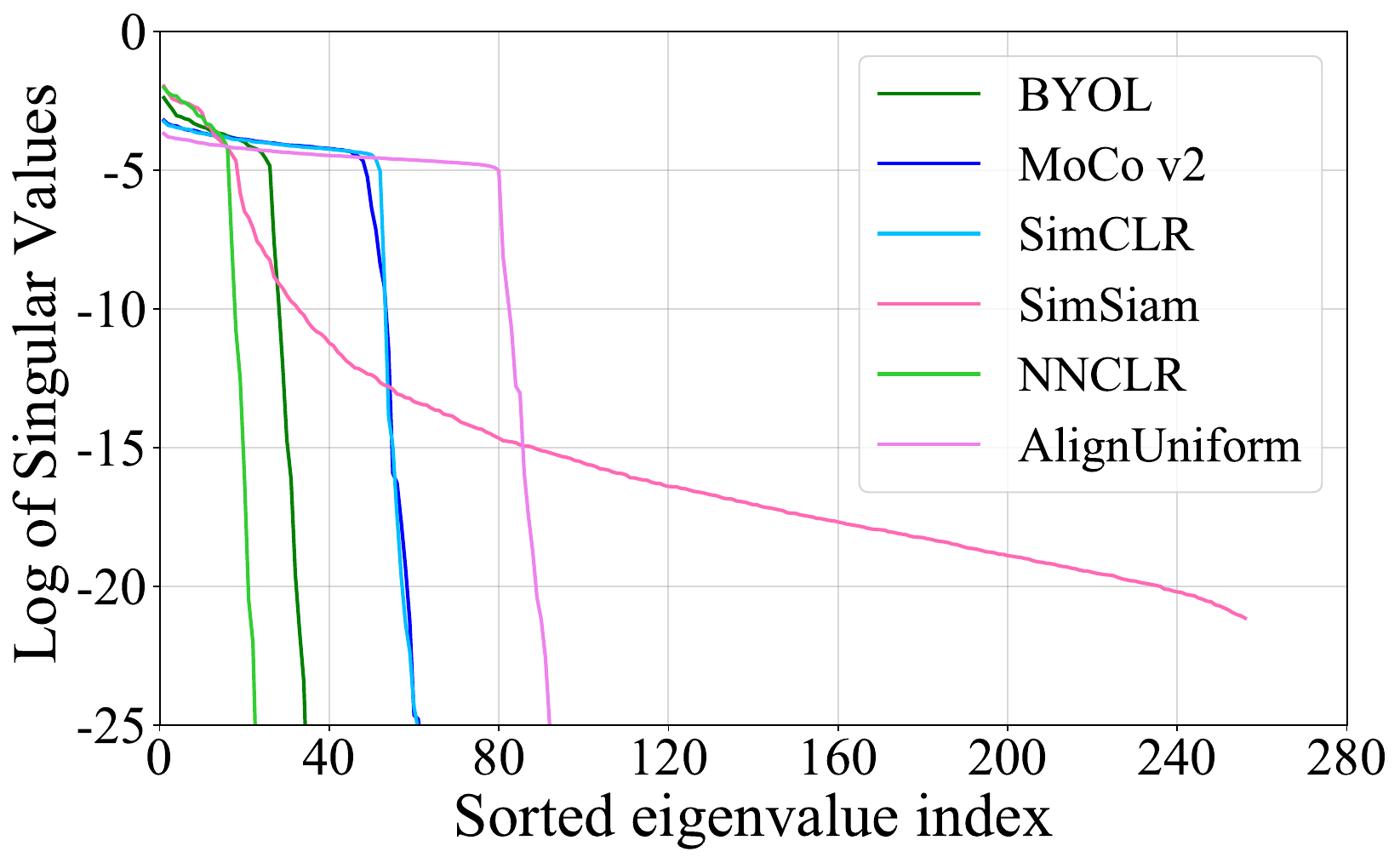}}
	\subfigure[MoCo v2]{ 
		\label{fig:moco v2 dimensional collapse}  	\includegraphics[width=0.31\textwidth]{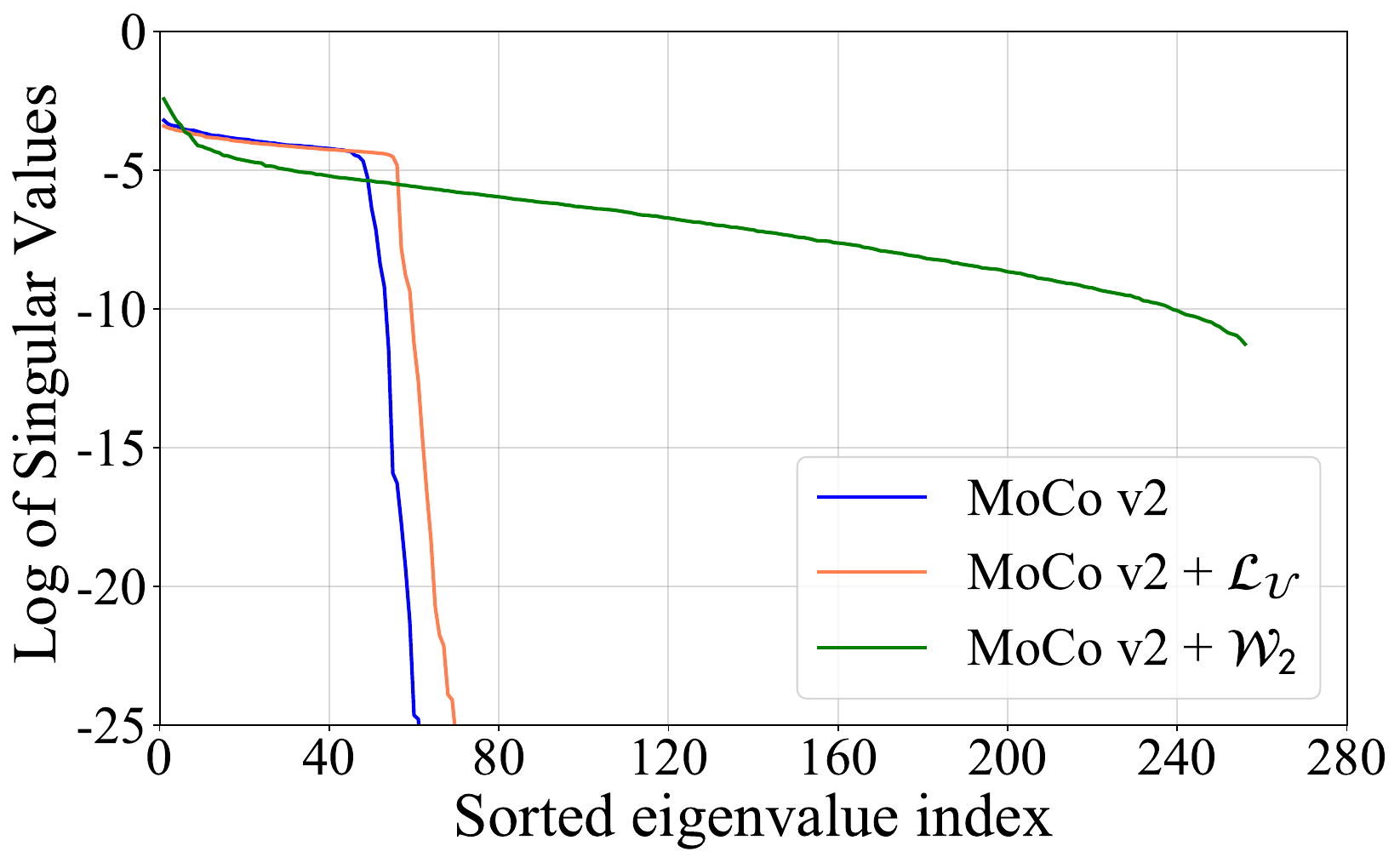}
	}
	\subfigure[BYOL]{
		\label{fig:byol dimensional collapse}	\includegraphics[width=0.31\textwidth]{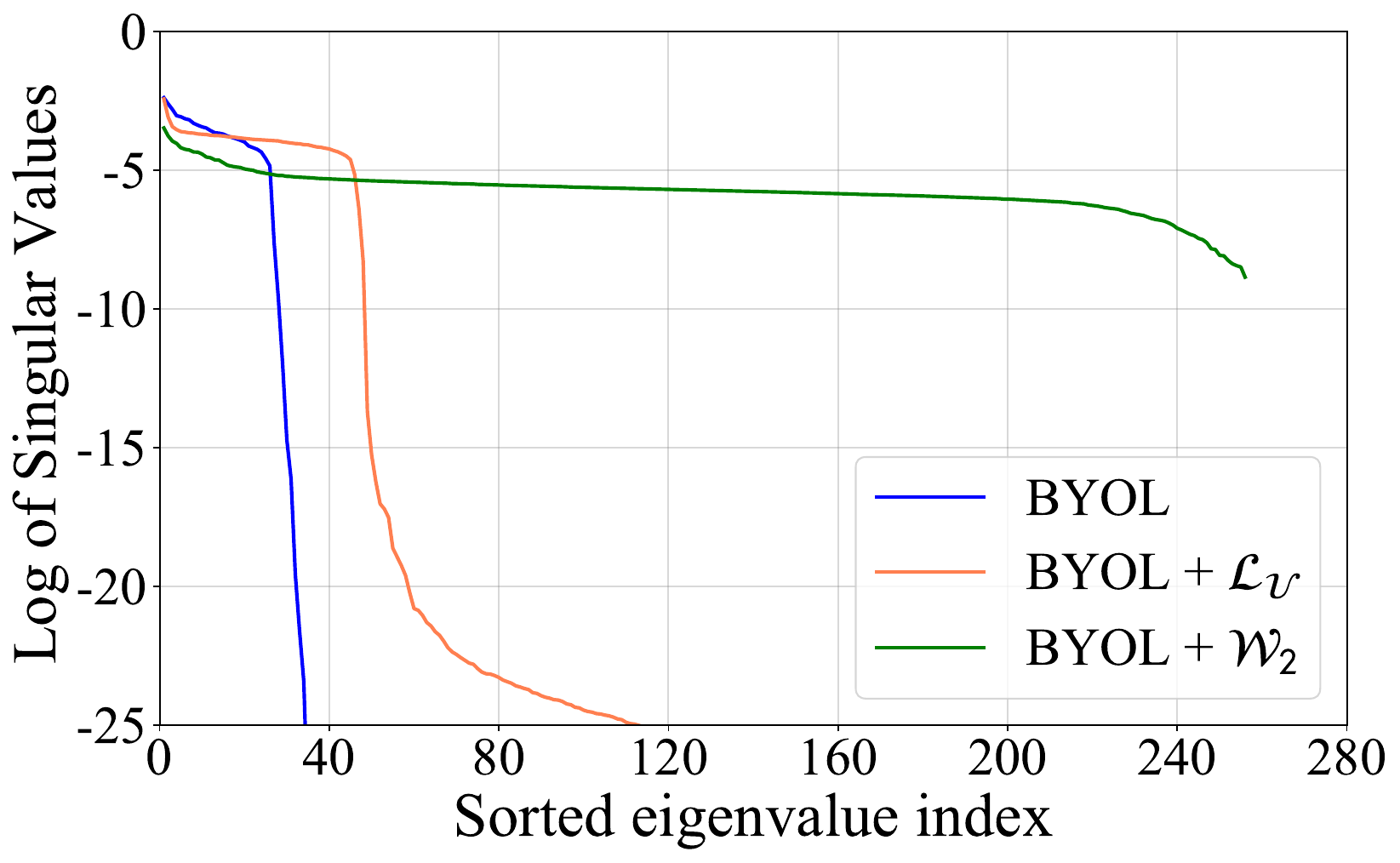}
	}
        \vspace{-8pt}
	\caption{ Dimensional collapse analysis on CIFAR-100 dataset.}
	\label{fig:dimensional collapse analysis}
        \vspace{-8pt}
\end{figure*}

\paragraph{Dimensional Collapse Analysis} 
We visualize the singular value spectra of the learned representations~\citep{Jing2021UnderstandingDC} for various models. These spectra  contain the singular values of the covariance matrix of representations from CIFAR-100 dataset, sorted in logarithmic scale order. As shown in Figure~\ref{fig:singular value spectrum}, most singular values collapse to zeros in most models, 
indicating a large number of collapsed coordinates in most models. To further understand how the additional loss $\mathcal{W}_{2}$  helps prevent dimensional collapse, we add  $\mathcal{W}_{2}$ as an additional loss for Moco v2 and BYOL, the numbers of collapsed coordinates decrease to zeros in both cases; see Figure~\ref{fig:moco v2 dimensional collapse} and Figure~\ref{fig:byol dimensional collapse}. This verifies that our proposed uniformity loss $\mathcal{W}_{2}$ can effectively address the dimensional collapse issue for Moco v2 and BYOL. In contrast,  $\mathcal{L}_{\mathcal{U}}$ can not effectively prevent  dimensional collapse.

%% file: sections/conclusion.tex
\section{Conclusion}\label{sec:conclusion}

In this paper, we have identified four principled properties that an effective uniformity metric should possess. Namely, an effective uniformity metric should remain invariant to instance permutations and sample replications while accurately capturing feature redundancy and dimensional collapse. Surprisingly, the popular uniformity metric proposed by \cite{Wang2020UnderstandingCR} fails to meet the majority of these properties, unveiling its limitations.   Empirical investigations corroborate our theoretical findings. To overcome these limitations, we introduce a new uniformity metric that satisfies all four properties. Particularly, this new metric demonstrates remarkable abilities to capture feature redundancy and dimensional collapse.  Integrating it as an auxiliary loss in various self-supervised learning methods effectively mitigates dimensional collapse and consistently improves their performance on downstream tasks.  Nonetheless, it is worth noting that the four identified properties may not encompass a comprehensive characterization of an ideal uniformity metric, warranting further exploration.

%% file: sections/appendix_new.tex
\appendix

\vspace{-40pt}
\addcontentsline{toc}{section}{Appendix} 
\part{Appendix}
\parttoc 


\section{Statistical distances over Gaussian distributions}
\label{sec:other distribution distance}

We first introduce the Wasserstein distance or the earth mover distance. 
\begin{definition}
The Wasserstein distance or earth-mover distance with $p$ norm is defined as below:
\begin{equation}
W_{p}(\mathbb{P}_r, \mathbb{P}_g) = (\inf_{\gamma \in \Pi(\mathbb{P}_r ,\mathbb{P}_g)} \mathbb{E}_{(x, y) \sim \gamma}\big[\|x - y\|^{p}\big])^{1/p}~.
\label{eq:def wasserstein distance}
\end{equation}
where $\Pi(\mathbb{P}_r,\mathbb{P}_g)$ denotes the set of all joint distributions $\gamma(x,y)$ whose marginals are respectively $\mathbb{P}_r$ and $\mathbb{P}_g$. Intuitively, when viewing each distribution as a unit amount of earth/soil, the Wasserstein distance or earth-mover distance takes the minimum cost of transporting ``mass'' from $x$ to $y$ to transform the distribution
$\mathbb{P}_r$ into the distribution $\mathbb{P}_g$.
This distance is also called the quadratic Wasserstein distance when $p=2$. 
\end{definition}

In this paper, we mainly exploit the quadratic Wasserstein distance over Gaussian distributions. Besides this distance,  we also discuss other distribution distances as uniformity metrics and make comparisons with the Wasserstein distance. Specifically, the Kullback-Leibler divergence and the Bhattacharyya distance over Gaussian distributions are provided in Lemma~\ref{theorem:kl distance} and Lemma~\ref{theorem:bd distance} respectively. Both distances require full-rank covariance matrices, making them impropriate to conduct dimensional collapse analysis. In contrast, our quadratic Wasserstein distance-based uniformity metric is free of such a requirement.
\begin{lemma}[Kullback-Leibler divergence~\citep{Lindley1959InformationTA}] \label{theorem:kl distance} 
Suppose two random variables $\m Z_1 \sim \mathcal{N}(\bm \mu_1, \m \Sigma_1)$ and $\m Z_2 \sim \mathcal{N}(\bm \mu_2, \m \Sigma_2)$ obey multivariate normal distributions, then Kullback-Leibler divergence between $\m Z1$ and $\m Z_2$ is:
\begin{align}
\KL(\m Z_1, \m Z_2) = \frac{1}{2}((\bm \mu_1-\bm \mu_2)^T \m \Sigma_2^{-1}(\bm \mu_1-\bm \mu_2) + \tr(\m \Sigma_2^{-1}\m \Sigma_1-\m I) + \ln{\frac{\det{\m \Sigma_2}}{\det \m \Sigma_1}}). \nonumber
\end{align}
\end{lemma}

\begin{lemma}
[Bhattacharyya Distance~\citep{Bhattacharyya1943OnAM}]
\label{theorem:bd distance}
Suppose two random variables $\m Z_1 \sim \mathcal{N}(\bm \mu_1, \m \Sigma_1)$ and $\m Z_2 \sim \mathcal{N}(\bm \mu_2, \m \Sigma_2)$ obey multivariate normal distributions, $\m \Sigma = \frac{1}{2}(\m \Sigma_1 + \m \Sigma_2)$, then bhattacharyya distance between $\m Z1$ and $\m Z_2$ is:
\begin{align}
\mathcal{D}_{\textrm{B}}(\m Z_1, \m Z_2) = \frac{1}{8}(\bm \mu_1-\bm \mu_2)^T \m \Sigma^{-1}(\bm \mu_1-\bm \mu_2) + \frac{1}{2}\ln \frac{\det \m \Sigma}{\sqrt{\det \m \Sigma_1 \det \m \Sigma_2}}. \nonumber
\end{align}
\end{lemma}
  
\section{Proof of Theorem~\ref{proof:the kl divergence}}
\label{app:KL_distance}

We first need the following lemma, whose proof is collected in the end of this section.  

\begin{lemma}\label{proof:the pdf of y_i}
Let  $\m Z \sim \mathcal{N}(\m 0, \sigma^2 \m I_m)$ and $\m Y=\m Z/\Vert \m Z \Vert_2$. Then the probability density function of $Y_i$, the $i$-th coordinate of $\m Y$ is: 
\begin{align}
\mathit{f}_{\scaleto{Y_i}{5pt}}(y_i) = \frac{\Gamma(m/2)}{\sqrt{\pi}\Gamma((m-1)/2)}(1-y_i^2)^{(m-3)/2}, ~~\forall~y_i\in [-1, 1]. \nonumber
\end{align}
\end{lemma}

We are ready to prove Theorem \ref{proof:the kl divergence}. 

\begin{proof}[Proof of Theorem \ref{proof:the kl divergence}]
According to the Lemma~\ref{proof:the pdf of y_i}, the pdf of $Y_i$ and $\hat{Y}_i$ are:
\begin{align}
\mathit{f}_{\scaleto{Y_i}{5pt}}(y) = \frac{\Gamma(m/2)}{\sqrt{\pi}\Gamma((m-1)/2)}(1-y^2)^{(m-3)/2}, \quad \mathit{f}_{\scaleto{\hat{Y}_i}{5pt}}(y) = \sqrt{\frac{m}{2\pi}} \exp\{-\frac{my^2}{2}\}. \nonumber
\end{align}
Then the Kullback-Leibler divergence between $Y_i$ and $\hat{Y}_i$ is
\begin{align}
\nonumber \KL(Y_i \Arrowvert \hat{Y}_i) 
& = \int_{-1}^{1}  \mathit{f}_{\scaleto{Y_i}{5pt}}(y) [\log \mathit{f}_{\scaleto{Y_i}{5pt}}(y)- \log \mathit{f}_{\scaleto{\hat{Y}_i}{5pt}}(y)] dy \\\nonumber
& = \int_{-1}^{1} \mathit{f}_{\scaleto{Y_i}{5pt}}(y) [\log \frac{\Gamma(m/2)}{\sqrt{\pi}\Gamma((m-1)/2)} +\frac{m-3}{2} \log (1-y^2) - \log\sqrt{\frac{m}{2\pi}}+\frac{my^2}{2}] dy \\\nonumber
& = \log \sqrt{\frac{2}{m}}\frac{\Gamma(m/2)}{\Gamma((m-1)/2)} + \int_{-1}^{1} \mathit{f}_{\scaleto{Y_i}{5pt}}(y) [\frac{m-3}{2} \log (1-y^2) + \frac{my^2}{2}] dy.
\end{align}
Letting $\mu = y^2$, we have $y=\sqrt{\mu}$ and $dy =\frac{1}{2} \mu^{-\frac{1}{2}}du$. Thus, 
\begin{align}
\nonumber \mathcal{A} & \triangleq \int_{-1}^{1} \mathit{f}_{\scaleto{Y_i}{5pt}}(y) [\frac{m-3}{2} \log (1-y^2) + \frac{my^2}{2}] dy \\\nonumber
& = 2 \int_{0}^{1} \frac{\Gamma(m/2)}{\sqrt{\pi}\Gamma((m-1)/2)}(1-y^2)^{\frac{m-3}{2}} [\frac{m-3}{2} \log (1-y^2) + \frac{my^2}{2}] dy \\\nonumber
& = \frac{\Gamma(m/2)}{\sqrt{\pi}\Gamma((m-1)/2)} \int_{0}^{1} (1-\mu)^{\frac{m-3}{2}} [\frac{m-3}{2} \log(1-\mu) + \frac{m}{2} \mu] \mu^{-\frac{1}{2}} d\mu \\\nonumber
& = \frac{\Gamma(m/2)}{\sqrt{\pi}\Gamma((m-1)/2)} \frac{m-3}{2} \int_{0}^{1} (1-\mu)^{\frac{m-3}{2}} \mu^{-\frac{1}{2}} \log(1-\mu) d\mu \\\nonumber 
& + \frac{\Gamma(m/2)}{\sqrt{\pi}\Gamma((m-1)/2)} \frac{m}{2} \int_{0}^{1} (1-\mu)^{\frac{m-3}{2}} \mu^{\frac{1}{2}} d\mu. \\\nonumber
\end{align}
By using the property of Beta distribution, and the inequality that $\frac{-\mu}{1-\mu} \leq \log (1-\mu) \leq - \mu$, we have
\begin{align}
\nonumber \mathcal{A}_1 
& \triangleq \frac{\Gamma(m/2)}{\sqrt{\pi}\Gamma((m-1)/2)} \frac{m-3}{2} \int_{0}^{1} (1-\mu)^{\frac{m-3}{2}} \mu^{-\frac{1}{2}} \log(1-\mu) d\mu \nonumber\\
& \leq - \frac{\Gamma(m/2)}{\sqrt{\pi}\Gamma((m-1)/2)} \frac{m-3}{2} \int_{0}^{1} (1-\mu)^{\frac{m-3}{2}} \mu^{\frac{1}{2}} d\mu \nonumber\\
& = - \frac{\Gamma(m/2)}{\sqrt{\pi}\Gamma((m-1)/2)} \frac{m-3}{2} \mathit{B}(\frac{3}{2}, \frac{m-1}{2}) ~\text{and} \nonumber\\
\mathcal{A}_2, 
& \triangleq \frac{\Gamma(m/2)}{\sqrt{\pi}\Gamma((m-1)/2)} \frac{m}{2} \int_{0}^{1} (1-\mu)^{\frac{m-3}{2}} \mu^{\frac{1}{2}} d\mu \nonumber\\
& = \frac{\Gamma(m/2)}{\sqrt{\pi}\Gamma((m-1)/2)} \frac{m}{2}  \mathit{B}(\frac{3}{2}, \frac{m-1}{2}). \nonumber
\end{align}
Then, for $\mathcal{A}$, we have
\begin{align}
\nonumber \mathcal{A} = \mathcal{A}_1 + \mathcal{A}_2 & \leq  - \frac{\Gamma(m/2)}{\sqrt{\pi}\Gamma((m-1)/2)} \frac{m-3}{2} \mathit{B}(\frac{3}{2}, \frac{m-1}{2}) + \frac{\Gamma(m/2)}{\sqrt{\pi}\Gamma((m-1)/2)} \frac{m}{2}  \mathit{B}(\frac{3}{2}, \frac{m-1}{2}) \\\nonumber
& = \frac{3}{2} \frac{\Gamma(m/2)}{\sqrt{\pi}\Gamma((m-1)/2)} \mathit{B}(\frac{3}{2}, \frac{m-1}{2}) = \frac{3}{2} \frac{\Gamma(m/2)}{\sqrt{\pi}\Gamma((m-1)/2)} \frac{\Gamma(3/2)\Gamma((m-1)/2)}{\Gamma((m+2)/2)}\\\nonumber
& = \frac{3}{2}  \frac{\Gamma(3/2)\Gamma(m/2)}{\sqrt{\pi}\Gamma((m+2)/2)} = \frac{3}{2}  \frac{(\sqrt{\pi}/2) \Gamma(m/2)}{\sqrt{\pi}\Gamma((m+2)/2)} = \frac{3}{4} \frac{\Gamma(m/2)}{\Gamma((m+2)/2)}.
\end{align}
Using the Stirling formula, we have $\Gamma(x+\alpha) \to \Gamma(x)x^{\alpha}$ as $x \to \infty $ and thus
\begin{align}
\nonumber \lim\limits_{m \to \infty} \KL(Y_i \Arrowvert \hat{Y}_i) &=  \lim\limits_{m \to \infty} \log \sqrt{\frac{2}{m}}\frac{\Gamma(m/2)}{\Gamma((m-1)/2)} +  \lim\limits_{m \to \infty} \mathcal{A} \\\nonumber
& \leq \lim\limits_{m \to \infty} \log \sqrt{\frac{2}{m}} \frac{\Gamma((m-1)/2)(\frac{m-1}{2})^{1/2}}{\Gamma((m-1)/2)} + \lim\limits_{m \to \infty} \frac{3}{4} \frac{\Gamma(m/2)}{\Gamma((m+2)/2)} \\\nonumber
& = \lim\limits_{m \to \infty} \log \sqrt{\frac{2}{m}} \sqrt{\frac{m-1}{2}} + \frac{3}{4} \frac{\Gamma(m/2)}{\Gamma(m/2) m} = \lim\limits_{m \to \infty} \log \sqrt{\frac{m-1}{m}} + \frac{3}{4m} = 0.
\end{align}
We further use $T_2$ inequality \citep[Theorem 4.31]{van2016probability}  to derive the quadratic Wasserstein
metric \citep[Definition 4.29]{van2016probability} as:
\$
\lim_{m \to \infty} \mathcal{W}_2(Y_i, \hat Y_i) 
&\leq \lim_{m \to \infty} \sqrt{\frac{2}{m} \KL(Y_i \Arrowvert \hat{Y}_i) } = 0.
\$
\end{proof}

\subsection{Proofs for  supporting lemmas}
\begin{proof}[Proof of Lemma \ref{proof:the pdf of y_i}]
Let $\m Z=[Z_1, Z_2, \cdots, Z_m] \sim \mathcal{N}(\m 0, \sigma^2 \m I_m)$, then $Z_i \sim \mathcal{N}(0, \sigma^2) , \forall i \in [1, m]$. Let  $U = Z_i/\sigma \sim \mathcal{N}(0, 1)$, $V = \sum_{j \neq i}^{m} (Z_j/\sigma)^2 \sim \mathcal{X}^2(m-1)$, then $U$ and $V$ are independent with each other. The random variable $T = \frac{U}{\sqrt{V/(m-1)}}$ follows  the Student's t-distribution with $m-1$ degrees of freedom, and its probability density function (pdf) is:
\begin{align}
\nonumber \mathit{f}_{\scaleto{T}{5pt}}(t) = \frac{\Gamma(m/2)}{\sqrt{(m-1)\pi}\Gamma((m-1)/2)} (1+\frac{t^2}{m-1})^{-m/2}.
\end{align}

For random variable $Y_i$, we have 
\$
Y_i = \frac{Z_i}{\sqrt{\sum_{i=1}^{m}Z_i^2}} = \frac{Z_i}{\sqrt{Z_i^2 + \sum_{j \neq i}^{m} Z_j^2}} = \frac{Z_i/\sigma}{\sqrt{(Z_i/\sigma)^2 + \sum_{j \neq i}^{m} (Z_j/\sigma)^2}} = \frac{U}{\sqrt{U^2 + V}},
\$
and then $T = \frac{U}{\sqrt{V/(m-1)}} = \frac{\sqrt{m-1}Y_i}{\sqrt{1-Y_i^2}}$, $Y_i = \frac{T}{\sqrt{T^2+m-1}}$. Therefore,  the cumulative distribution function (cdf) of $T$ is: 
\begin{align}
\nonumber F_{Y_i}(y_i) = P(\{Y_i \leq y_i\}) &= 
\begin{cases}
P(\{Y_i \leq y_i\}) & y_i \leq 0 \\
P(\{Y_i \leq 0\}) +P(\{0 < Y_i \leq y_i\}) & y_i > 0
\end{cases} \\\nonumber
& = \begin{cases}
P(\{ \frac{T}{\sqrt{T^2+m-1}} \leq y_i\}) & y_i \leq 0 \\
P(\{ \frac{T}{\sqrt{T^2+m-1}} \leq 0\}) +P(\{0 < \frac{T}{\sqrt{T^2+m-1}} \leq y_i\}) & y_i > 0
\end{cases} \\\nonumber
& = \begin{cases}
P(\{ \frac{T^2}{T^2+m-1} \geq y_i^2, T \leq 0\}) & y_i \leq 0 \\
P(\{ T \leq 0\} +P(\{ \frac{T^2}{T^2+m-1} \leq y_i^2, T>0 \}) & y_i > 0
\end{cases} \\\nonumber
& = \begin{cases}
P(\{ T \leq \frac{\sqrt{m-1}y_i}{\sqrt{1-y_i^2}} \}) & y_i \leq 0 \\
P(\{ T \leq 0\} +P(\{ 0< T \leq \frac{\sqrt{m-1}y_i}{\sqrt{1-y_i^2}}\}) & y_i > 0
\end{cases} \\\nonumber
& = P(\{ T \leq \frac{\sqrt{m-1}y_i}{\sqrt{1-y_i^2}} \}) = F_{T}(\frac{\sqrt{m-1}y_i}{\sqrt{1-y_i^2}}).
\end{align}
The probability density function of $Y_i$ can then be derived as:
\begin{align}
\nonumber \mathit{f}_{\scaleto{Y_i}{5pt}}(y_i) 
& = \frac{d }{dy_i}F_{Y_i}(y_i) = \frac{d }{dy_i}F_{T}(\frac{\sqrt{m-1}y_i}{\sqrt{1-y_i^2}}) \\\nonumber
& = \mathit{f}_{\scaleto{T}{5pt}}(\frac{\sqrt{m-1}y_i}{\sqrt{1-y_i^2}}) \frac{d}{dy_i} (\frac{\sqrt{m-1}y_i}{\sqrt{1-y_i^2}}) \\\nonumber
& = [\frac{\Gamma(m/2)}{\sqrt{(m-1)\pi}\Gamma((m-1)/2)} (1-y_i^2)^{m/2}] [\sqrt{m-1} (1-y_i^2)^{-3/2}] \\\nonumber
& = \frac{\Gamma(m/2)}{\sqrt{\pi}\Gamma((m-1)/2)}(1-y_i^2)^{(m-3)/2}.
\end{align}
\end{proof}

\section{Examining the four properties for two uniformity metrics
}
\label{appendix: claim}

\subsection{Proof of Theorem \ref{theorem:lu}: Examining the four properties for $-\mathcal{L_U}$}
\label{appendix: claim lu}

{Property~\ref{pro:ipc}} can be easily verified for $-\mathcal{L_U}$ and thus we skip the verification. We only examine the other  three properties for the uniformity metric $-\mathcal{L_U}$.

\underline{First, we prove that $-\mathcal{L_U}$ does not satisfy {Property~\ref{pro:icc}}.} Due to the definition of $\mathcal{L_U}$ in Eqn.~(\ref{eq:uniform loss}), we have
\begin{align} 
{\mathcal{L_U}(\mathcal{D} \uplus \mathcal{D})} 
&\triangleq \log \frac{1}{2n(2n-1)/2} 
\left(4 \sum_{i=2}^{n} \sum_{j=1}^{i-1} e^{-t \Vert \m z_i  - \m z_j \Vert_2^{2}} + \sum_{i=1}^{n} e^{-t \Vert \m z_i - \m z_i \Vert_2^{2}}\right) \nonumber\\
& = \log \frac{1}{2n(2n-1)/2} \left(4\sum_{i=2}^{n} \sum_{j=1}^{i-1} e^{-t \Vert \m z_i  - \m z_j \Vert_2^{2}} + n\right). \nonumber\\ 
\end{align}
Letting $G = \sum_{i=2}^{n} \sum_{j=1}^{i-1} e^{-t \Vert \m z_i  - \m z_j \Vert_2^{2}}$,  we have
\begin{align}
G =  \sum_{i=2}^{n} \sum_{j=1}^{i-1} e^{-t \Vert \m z_i  - \m z_j \Vert_2^{2}} \leq  \sum_{i=2}^{n} \sum_{j=1}^{i-1} e^{-t \Vert \m z_i  - \m z_i\Vert_2^{2}} = n(n-1)/2, \nonumber
\end{align}
and $G = n(n-1)/2$ if and only if $\m z_1 = \m z_2 = \ldots = \m z_n$. Thus 
\begin{align}
\mathcal{L_U}(\mathcal{D} \uplus \mathcal{D}) - \mathcal{L_U}({\mathcal{D}})  & = \log \frac{4G + n}{2n(2n-1)/2} - \log \frac{G}{n(n-1)/2} \nonumber\\
& = \log \frac{(4G + n )n(n-1)/2}{2nG(2n-1)/2} 
= \log \frac{(4G + n)(n-1)}{4nG-2G} \nonumber\\
& = \log \frac{4nG-4G+ n^2-n}{4nG -2G} \geq \log 1 = 0. \nonumber
\end{align}
The above equality holds if and only if $G = n(n-1)/2$, which requires $\m z_1 = \m z_2 = ... = \m z_n$, a trivial case when all representations collapse to one constant point. We have excluded this trivial case, and thus $-\mathcal{L_U}(\mathcal{D} \uplus \mathcal{D}) < -\mathcal{L_U}(\mathcal{D})$.  Therefore, the uniformity metric $-\mathcal{L_U}$ does not satisfy {Property~\ref{pro:icc}}.

\underline{Second, we prove that  $-\mathcal{L_U}$ does not  satisfy {Property~\ref{pro:fcc}}}. 
Letting $\hat{\m z}_i = \m z_i \oplus \m z_i$ and $\hat{\m z}_j = \m z_j \oplus \m z_j$, we have
\begin{align}
\mathcal{L_U}(\mathcal{D} \oplus \mathcal{D}) \triangleq \log \frac{1}{n(n-1)/2} \sum_{i=2}^{n} \sum_{j=1}^{i-1} e^{-t\Vert \hat{\m z}_i - \hat{\m z}_j \Vert_2^{2}}. \nonumber
\end{align}

By the definitions of $\hat{\m z}_i$ and $\hat{\m z}_j$, we have $\Vert \hat{\m z}_i \Vert_2 = \sqrt{2} \Vert \m z_i \Vert_2$, $\Vert \hat{\m z}_j \Vert_2 = \sqrt{2} \Vert \m z_j \Vert_2$, and $\langle \hat{\m z}_i, \hat{\m z}_j \rangle= 2\langle \m z_i, \m z_j \rangle$. Thus 
\begin{align}
\Vert \hat{\m z}_i - \hat{\m z}_j\Vert_2^{2} = 2 \Vert \m z_i \Vert_2^2 + 2 \Vert \m z_j \Vert_2^2 - 4 \langle \m z_i, \m z_j \rangle = 2 \Vert \m z_i - \m z_j\Vert_2^{2} \geq \Vert \m z_i - \m z_j\Vert_2^{2}.\nonumber
\end{align}
Therefore, $-\mathcal{L_U}(\mathcal{D} \oplus \mathcal{D}) \geq -\mathcal{L_U}(\mathcal{D})$, indicating that the uniformity  metric $-\mathcal{L_U}$ does not satisfy the {Property~\ref{pro:fcc}}.

Third, we prove that the existing metric $-\mathcal{L_U}$ does not satisfy the {Property~\ref{pro:fbc}}.
Letting $\hat{\m z}_i = \m z_i \oplus \m 0^{k}$ and $\hat{\m z}_j = \m z_j \oplus \m 0^{k}$, we have
\begin{align}
\mathcal{L_U}(\mathcal{D} \oplus \m 0^{k}) \triangleq \log \frac{1}{n(n-1)/2} \sum_{i=2}^{n} \sum_{j=1}^{i-1} e^{-t\Vert \hat{\m z}_i - \hat{\m z}_j\Vert_2^{2}}. \nonumber
\end{align}
By the definitions of $\hat{\m z}_i$ and $\hat{\m z}_j$, we have $\Vert \hat{\m z}_i \Vert_2 = \Vert \m z_i \Vert_2$, $\Vert \hat{\m z}_j \Vert_2 = \Vert \m z_j \Vert_2$, $\langle \hat{\m z}_i, \hat{\m z}_j \rangle=\langle \m z_i, \m z_j \rangle$, and thus 
\begin{align}
\Vert \hat{\m z}_i - \hat{\m z}_j \Vert_2^{2} = \Vert \hat{\m z}_i \Vert_2^2 + \Vert \hat{\m z}_j \Vert_2^2 - 2 \langle \hat{\m z}_i, \hat{\m z}_j \rangle =  \Vert \m z_i \Vert_2^2 + \Vert \m z_j \Vert_2^2 - 2 \langle \m z_i, \m z_j \rangle = \Vert \m z_i - \m z_j \Vert_2^{2}.\nonumber
\end{align}
Therefore, $-\mathcal{L_U}(\mathcal{D} \oplus \m 0^{k}) = -\mathcal{L_U}(\mathcal{D})$, indicating that the uniformity metric $-\mathcal{L_U}$ does not satisfy {Property~\ref{pro:fbc}}.

\subsection{Proof of Theorem \ref{theorem:w2}: Examining  the four properties for $-\mathcal{W}_2$}
\label{appendix: claim w2}
{Property~\ref{pro:ipc} can be easily verified for  $-\mathcal{W}_{2}$, and thus the proof is skipped.  We only examine the rest three properties  for the proposed uniformity metric $-\mathcal{W}_{2}$.

\underline{First, we prove that our proposed metric $-\mathcal{W}_{2}$ satisfies {Property~\ref{pro:icc}}}. 
Let $\hat {\bm\mu}$ and $\hat{\bm \Sigma}$ be defined as above, for $\mathcal{D} \uplus \mathcal{D} =  \{\m z_1, \m z_2, ..., \m z_n, \m z_1, \m z_2, ..., \m z_n\}$, the mean  and covariance estimators are
\begin{align}
\widetilde{\bm \mu} = \frac{1}{2n}\sum_{i=1}^{n} 2 \m z_i = \hat{\bm \mu}, \quad \widetilde{\m \Sigma} = \frac{1}{2n}\sum_{i=1}^{n}2(\m z_i - \widetilde{\bm \mu})^T(\m z_i -\widetilde{\bm \mu}) = \hat{\m \Sigma}, \nonumber
\end{align}
which agree with those for $\mathcal{D}$. 
Then we have
\begin{align}
\mathcal{W}_{2}(\mathcal{D} \uplus \mathcal{D}) \triangleq \sqrt{\Vert \hat{\bm \mu}\Vert^2_{2} + 1 + \tr(\hat{\m \Sigma}) -\frac{2}{\sqrt{m}} \tr(\hat{\m \Sigma}^{1/2}}) = \mathcal{W}_{2}(\mathcal{D}). \nonumber
\end{align}
Therefore,  our proposed metric $-\mathcal{W}_{2}$ satisfies {Property~\ref{pro:icc}}.

\underline{Second, we prove that $-\mathcal{W}_{2}$ satisfies {Property~\ref{pro:fcc}}}. 
Let  $\widetilde{\m z}_i = \m z_i \oplus \m z_i \in \mathbb{R}^{2m}$. For $\mathcal{D} \oplus \mathcal{D}$, the  mean  and covariance estimators are: 
\begin{align}
\widetilde{\bm\mu} 
=\begin{pmatrix}
\hat{\bm\mu}\\
\hat{\bm\mu}
\end{pmatrix}, \quad
\widetilde{\m \Sigma} = \begin{pmatrix}
\hat{\m \Sigma} & \hat{\m \Sigma} \\
\hat{\m \Sigma} & \hat{\m \Sigma}
\end{pmatrix}. \nonumber
\end{align}
We easily have 
\$
\widetilde{\m \Sigma}^{1/2}=\begin{pmatrix}
\hat{\m \Sigma}^{1/2}/\sqrt{2} & \hat{\m \Sigma}^{1/2}/\sqrt{2}  \\
\hat{\m \Sigma}^{1/2}/\sqrt{2}  & \hat{\m \Sigma}^{1/2}/\sqrt{2} 
\end{pmatrix}, ~\tr(\widetilde{\m \Sigma}) = 2\tr(\hat{\m \Sigma}), ~\text{and}~\tr(\widetilde{\m \Sigma}^{1/2})=\sqrt{2}\tr(\hat{\m \Sigma}^{1/2}). 
\$ 
Thus 
\begin{align}
\mathcal{W}_{2}(\mathcal{D} \oplus \mathcal{D}) & \triangleq \sqrt{\Vert \widetilde{\bm \mu}\Vert^2_{2} + 1 + \tr(\widetilde{\m \Sigma}) -\frac{2}{\sqrt{2m}}\tr(\widetilde{\m \Sigma}^{1/2})} \nonumber \\
& = \sqrt{2\Vert \hat{\bm \mu} \Vert^2_{2} + 1 + 2\tr(\hat{\m \Sigma}) -\frac{2\sqrt{2}}{\sqrt{2m}} \tr(\hat{\m \Sigma}^{1/2})} \nonumber \\
& > \sqrt{\Vert \hat{\bm \mu} \Vert^2_{2} + 1 + \tr(\hat{\m \Sigma}) -\frac{2}{\sqrt{m}} \tr(\hat{\m \Sigma}^{1/2})} = \mathcal{W}_{2}(\mathcal{D}). \nonumber
\end{align}
Therefore, $-\mathcal{W}_{2}(\mathcal{D} \oplus \mathcal{D}) < -\mathcal{W}_{2}(\mathcal{D})$, indicating that our proposed metric $-\mathcal{W}_{2}$ could satisfy the {Property~\ref{pro:fcc}}.

\underline{Third, we prove that our proposed metric $-\mathcal{W}_{2}$ satisfies {Property~\ref{pro:fbc}}}.
Let $\widetilde{\m z}_i = \m z_i \oplus \m 0^{k} \in \mathbb{R}^{m+k}$ with an overload of notation. For  $\mathcal{D} \oplus \m 0^{k}$, the sample  mean  and covariance estimators are 
\begin{align}
\widetilde{\bm \mu} = \begin{pmatrix}
\hat{\bm \mu} \\
\m 0^{k}
\end{pmatrix}, \quad
\widetilde{\m \Sigma} 
= \begin{pmatrix}
\hat{\m \Sigma} & \m 0^{m \times k} \\
\m 0^{k \times m} & \m 0^{k \times k}
\end{pmatrix},\nonumber
\end{align}
where $\hat{\bm\mu}$ and $\hat{\m\Sigma}$ are defined previously.  Therefore, we have $\tr(\widetilde{\m \Sigma}) = \tr(\hat{\m \Sigma})$, $\tr(\widetilde{\m \Sigma}^{1/2}) = \tr(\hat{\m \Sigma}^{1/2})$, and thus
\begin{align}
\mathcal{W}_{2}(\mathcal{D} \oplus \m 0^{k}) & \triangleq \sqrt{\Vert \widetilde{\bm \mu}\Vert^2_{2} + 1 + \tr(\widetilde{\m \Sigma}) -\frac{2}{\sqrt{m+k}}\tr(\widetilde{\m \Sigma}^{1/2})} \nonumber \\
& = \sqrt{\Vert \hat{\bm \mu} \Vert^2_{2} + 1 + \tr(\hat{\m \Sigma}) -\frac{2}{\sqrt{m+k}} \tr(\hat{\m \Sigma}^{1/2})} \nonumber \\
& > \sqrt{\Vert \hat{\bm \mu} \Vert^2_{2} + 1 + \tr(\hat{\m \Sigma}) -\frac{2}{\sqrt{m}} \tr(\hat{\m \Sigma}^{1/2})} = \mathcal{W}_{2}(\mathcal{D}). \nonumber
\end{align}
Therefore,  $-\mathcal{W}_{2}(\mathcal{D} \oplus \m 0^{k}) < -\mathcal{W}_{2}(\mathcal{D})$, indicating that our proposed metric $-\mathcal{W}_{2}$ satisfies the {Property~\ref{pro:fbc}}. 

\section{Further comparisons between $\m Y$ and $\hat{\m Y}$}\label{sec:further comparisons}

This section further compares the distributions of $\m Y$ and  $\hat{\m Y}$.

We visually compare the distributions of $Y_i$ and $\hat{Y}_i$. To estimate the distributions of $Y_i$ and $\hat{Y}_i$, we bin 200,000 sampled data points into 51 groups. Figure~\ref{fig:binning density over dimensions} compares the binning densities of $Y_i$ and  $\hat{Y}_i$ when  $m\in \{2, 4, 8, 16, 32, 64, 128, 256\}$. We can observe that two distributions are highly overlapped when $m$ is moderately large, e.g.,  $m\geq 8$ or $m\geq 16$.
\begin{figure*}[ht]
	\centering
        \vspace{-3mm}
	\subfigure[$m=2$]{ 
		\label{fig:binning density m=2}  
		\includegraphics[width=0.225\textwidth]{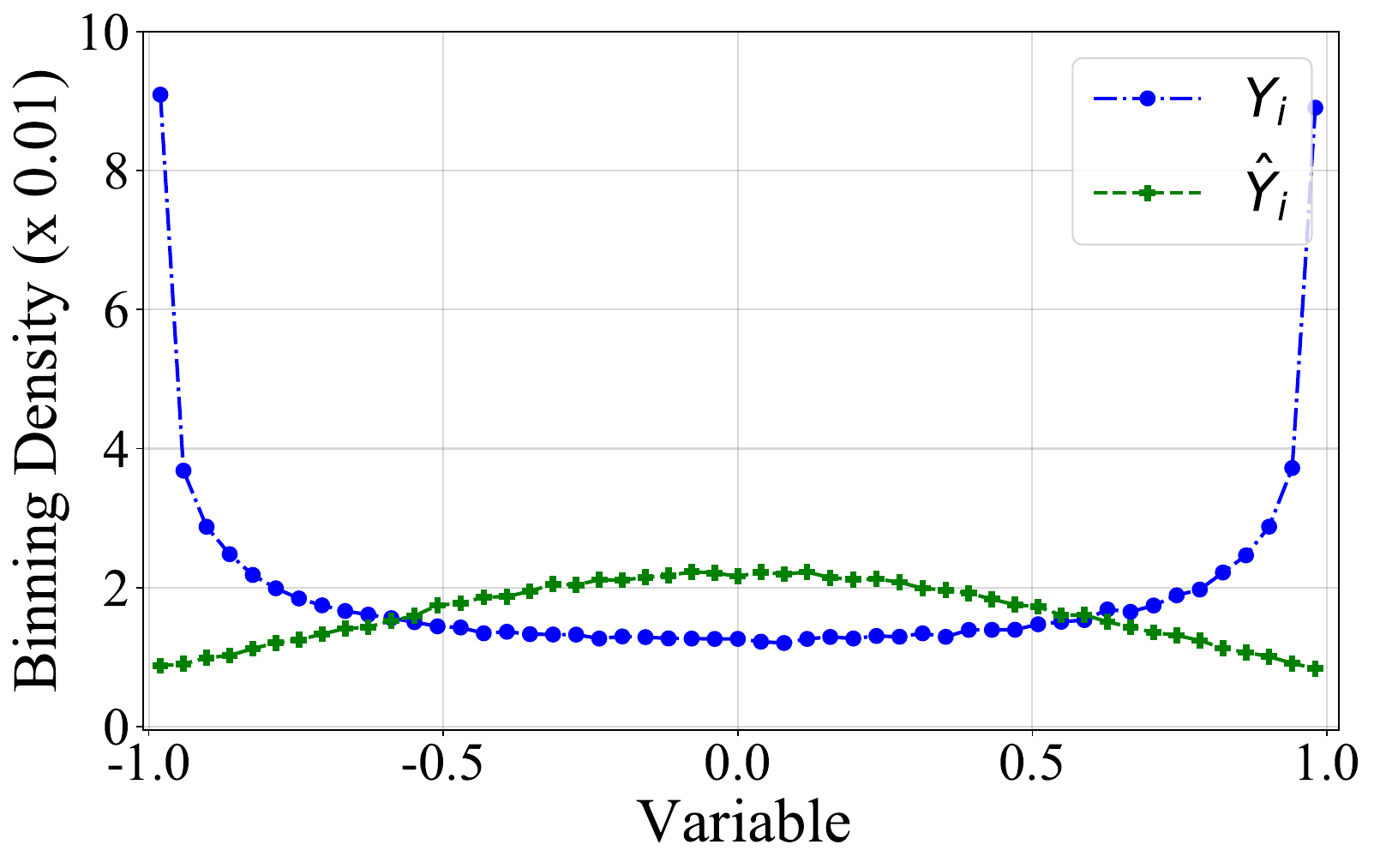}
	}
	\subfigure[$m=4$]{ 
		\label{fig:binning density m=4}  
		\includegraphics[width=0.225\textwidth]{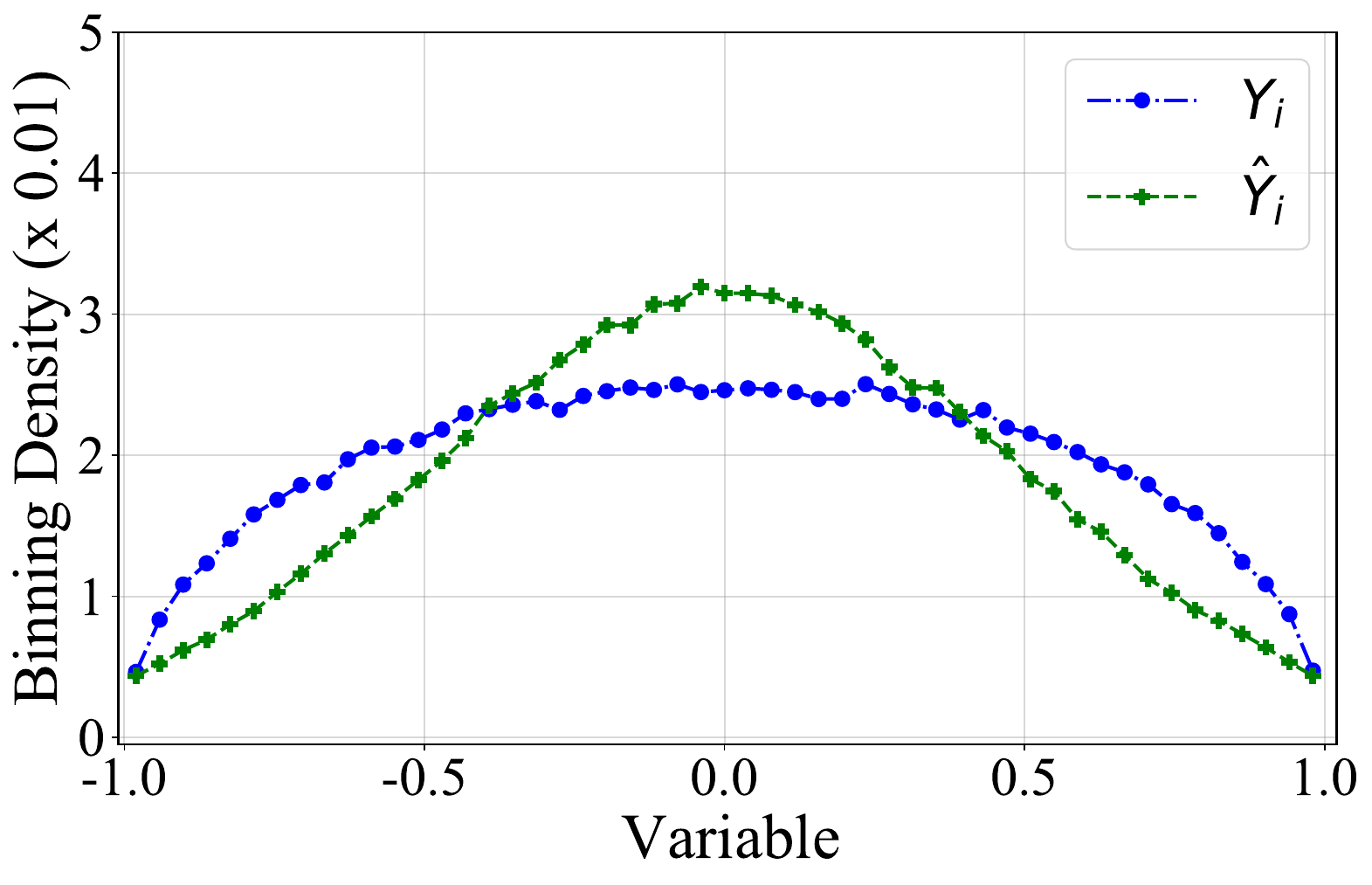}
	}
	\subfigure[$m=8$]{ 
		\label{fig:binning density m=8}  
		\includegraphics[width=0.225\textwidth]{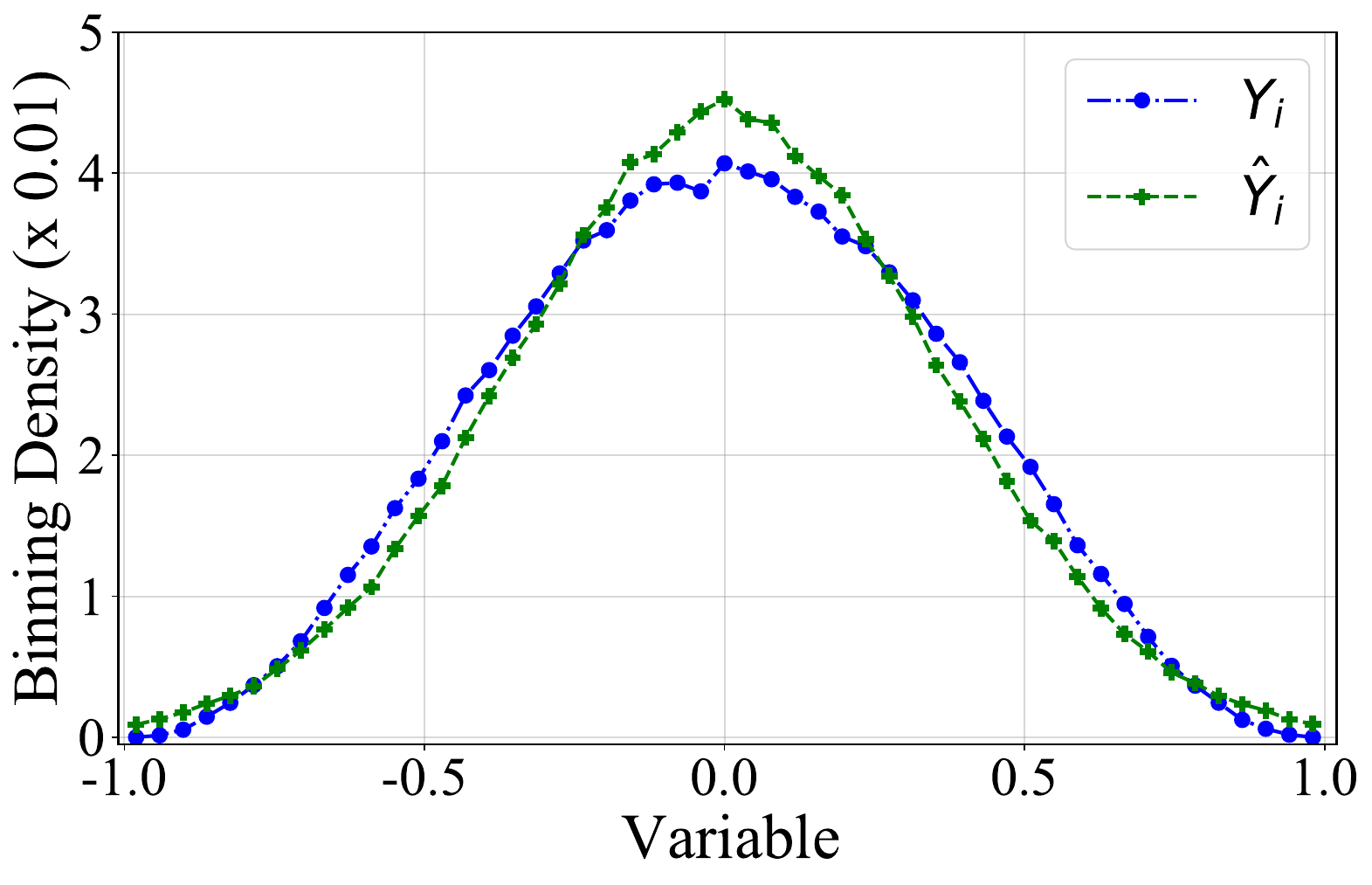}
	}
	\subfigure[$m=16$]{ 
		\label{fig:binning density m=16}  
		\includegraphics[width=0.225\textwidth]{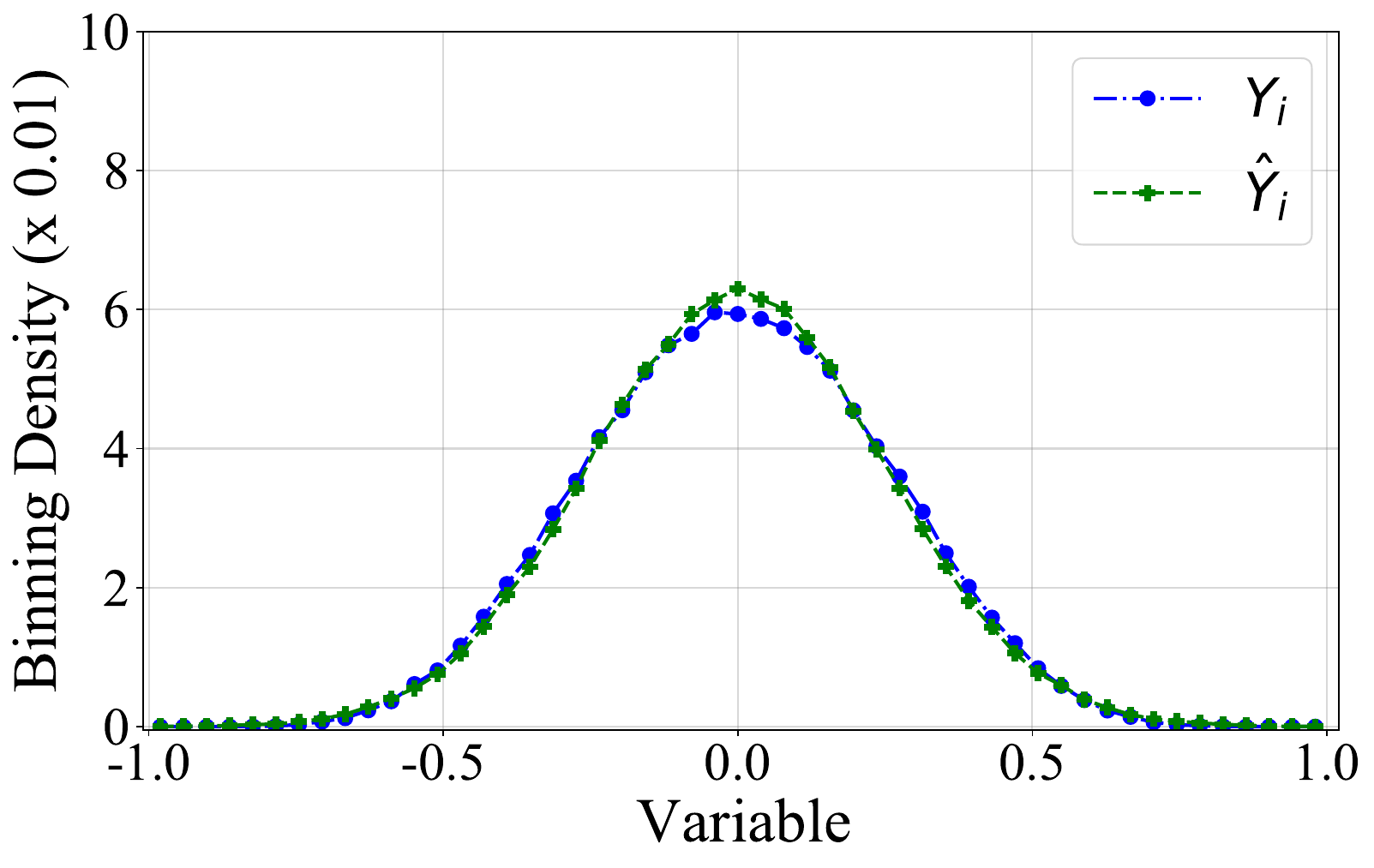}
	}
        \vspace{-4mm}
	\subfigure[$m=32$]{ 
		\label{fig:binning density m=32}  
		\includegraphics[width=0.225\textwidth]{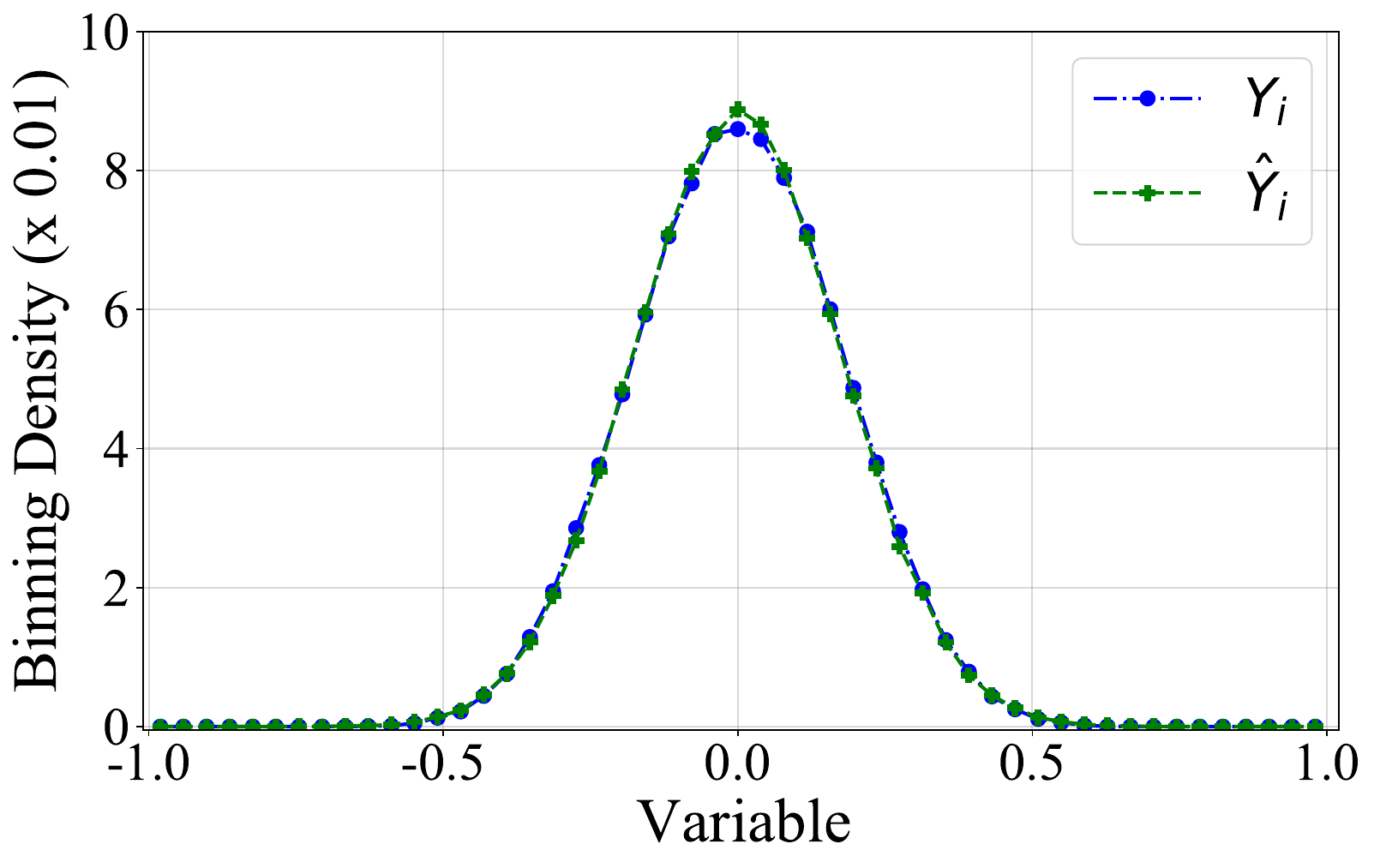}
	}
	\subfigure[$m=64$]{ 
		\label{fig:binning density m=64}  
		\includegraphics[width=0.225\textwidth]{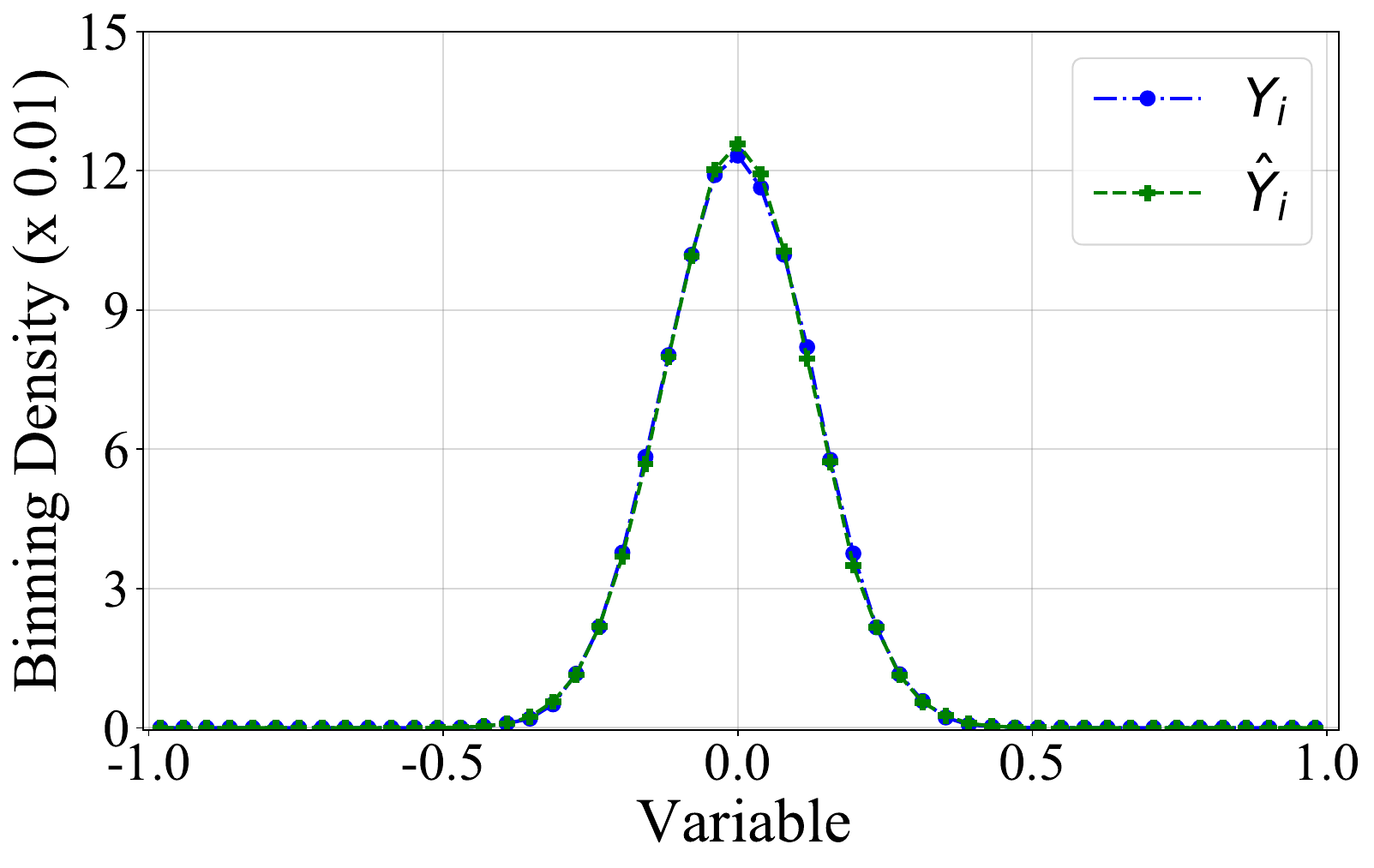}
	}
	\subfigure[$m=128$]{ 
		\label{fig:binning density m=128}  
		\includegraphics[width=0.225\textwidth]{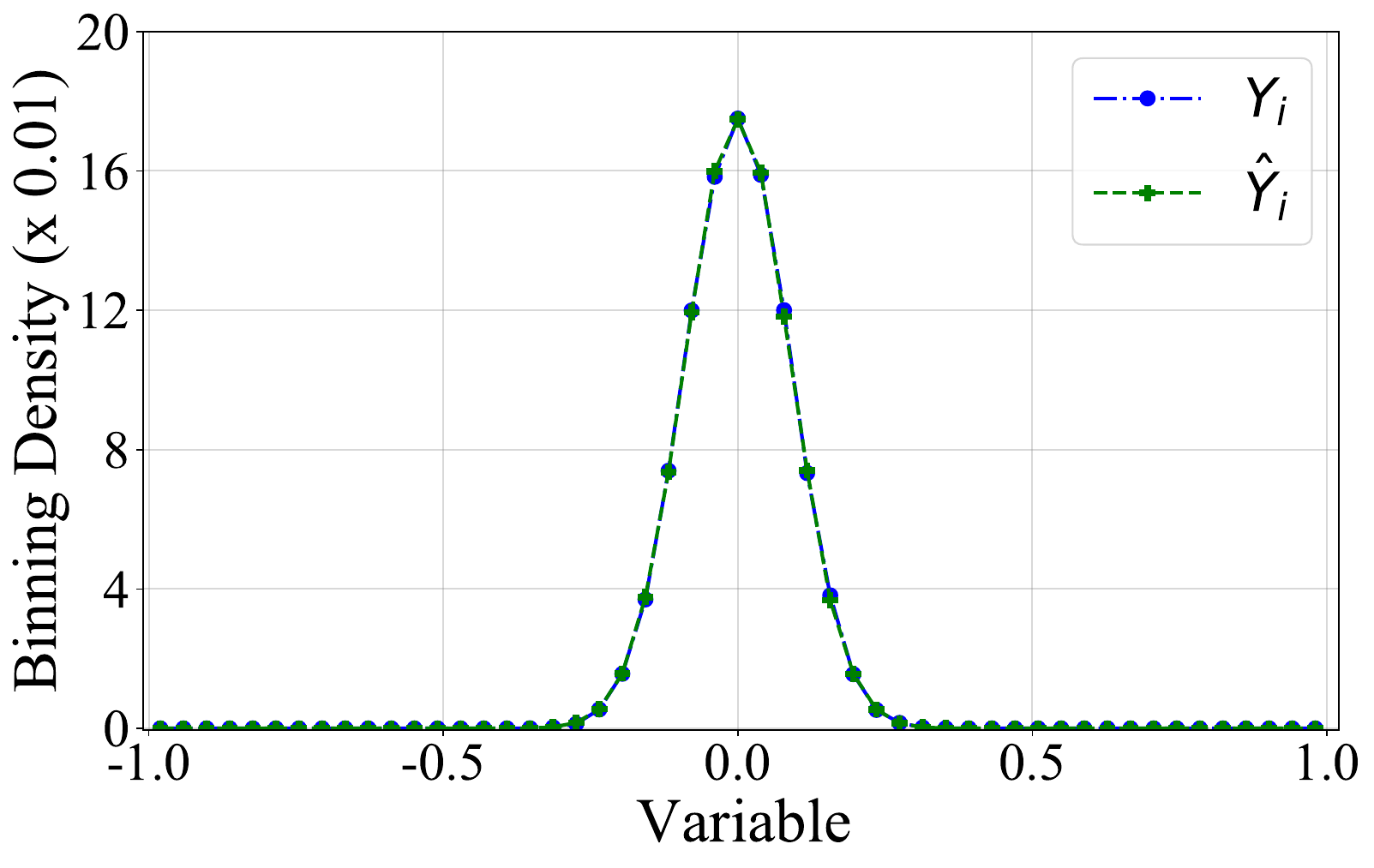}
	}
	\subfigure[$m=256$]{ 
		\label{fig:binning density m=256}  
		\includegraphics[width=0.225\textwidth]{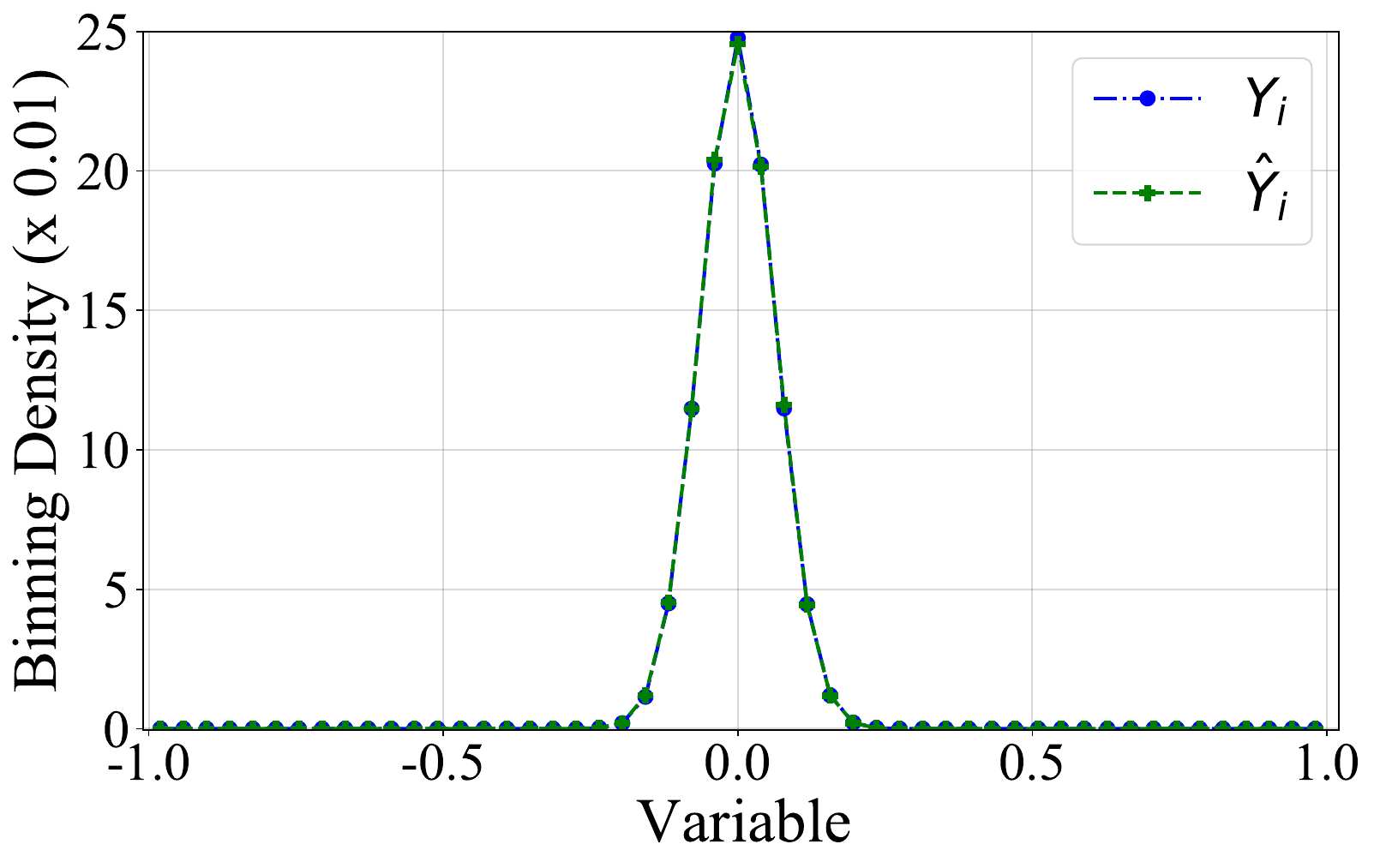}
	}
	\caption{Comparing the binning densities of $Y_i$ and $\hat{Y}_i$ with various dimensions.}
    \label{fig:binning density over dimensions}
    \vspace{-5mm}
\end{figure*}

By binning 2,000,000 data points into $51 \times 51$ groups in two-axis, we also analyze the joint binning densities and present 2D joint binning densities of  $(Y_i, Y_j)$ ($i \neq j$) in Figure~\ref{fig:2d density of y_i} and  $(\hat{Y}_i, \hat{Y}_j)$ ($i \neq j$) in Figure~\ref{fig:2d density of hat y_i}. Even if $m$ is relatively small (i.e., 32), the densities of the two distributions are close.

\begin{figure*}[h]
\vspace{-6pt}
	\centering
	\subfigure[Density for $Y_i$ and $Y_j$ ]{
		\label{fig:2d density of y_i}
		\includegraphics[width=0.38\textwidth]{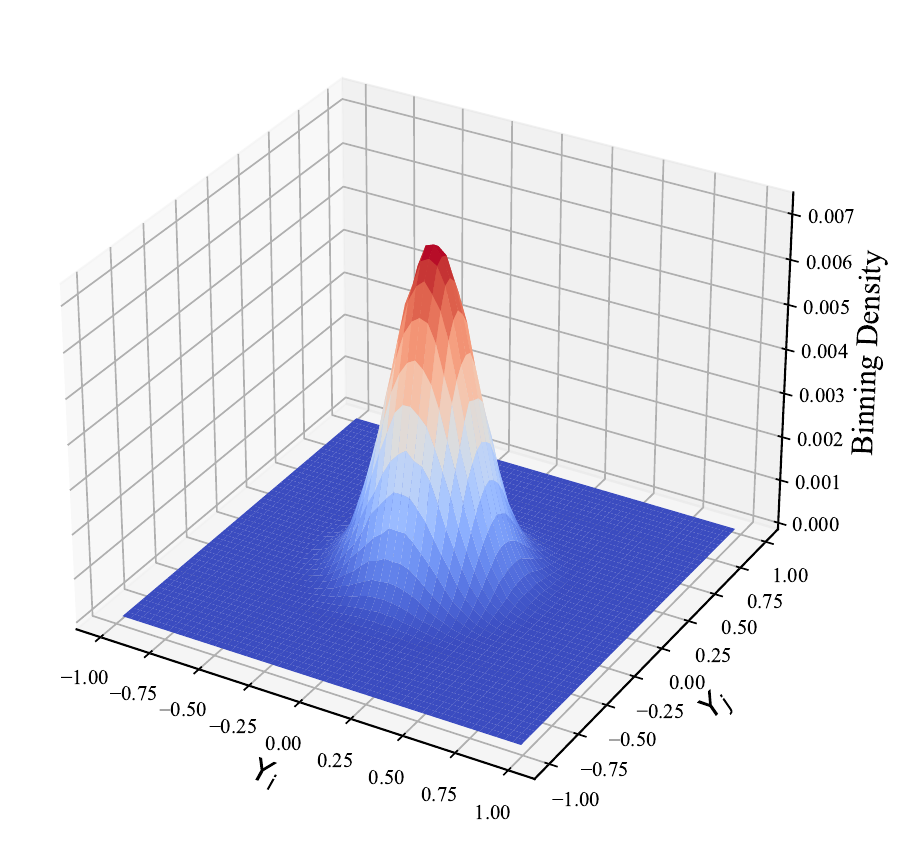}
	}
	\subfigure[Density for $\hat{Y}_i$ and $\hat{Y}_j$]{
		\label{fig:2d density of hat y_i}
		\includegraphics[width=0.38\textwidth]{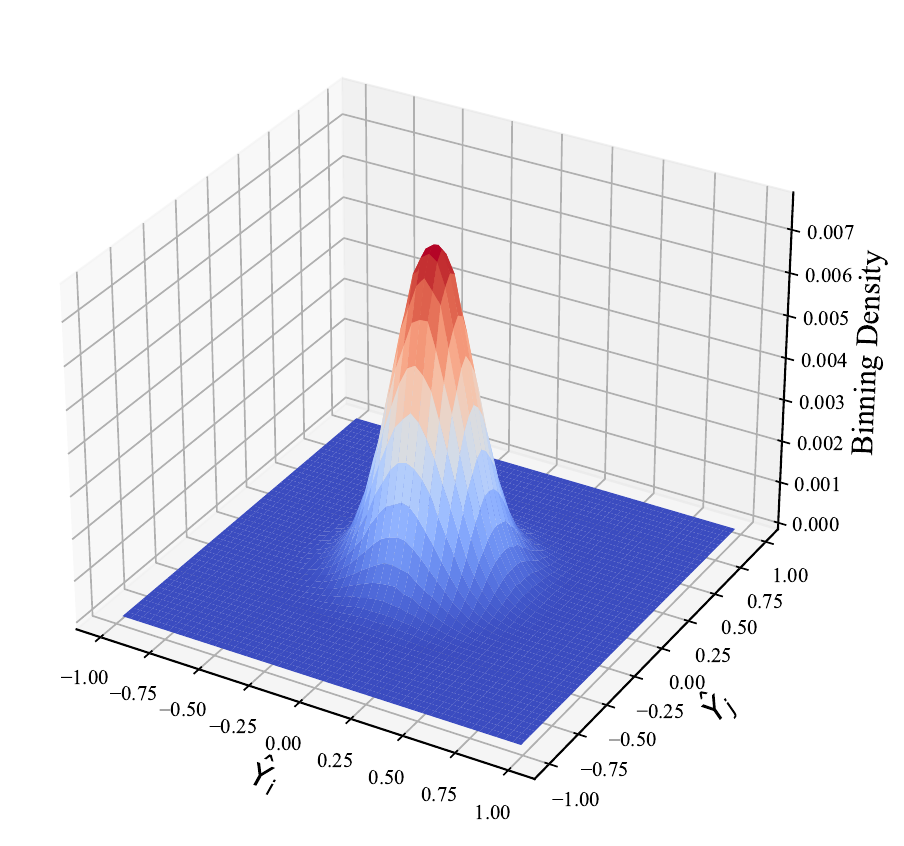}
	}
	\vspace{-4mm}
	\caption{Visualization of two arbitrary dimensions for  $\m Y$ and $ \hat{\m Y} $ when $m=32$.}
	\vspace{-4mm}
	\label{fig:density estimate}
\end{figure*}

\section{Additional synthetic studies}
\label{appendix: supplementary empirical study}

\subsection{Correlation between $-\mathcal{L_U}$ and $-\mathcal{W}_{2}$}
\label{Appendix:correlation}

\begin{wrapfigure}[12]{r}{0.40\textwidth}
 \vspace{-4pt}
\begin{center}
\includegraphics[width=0.39\textwidth]{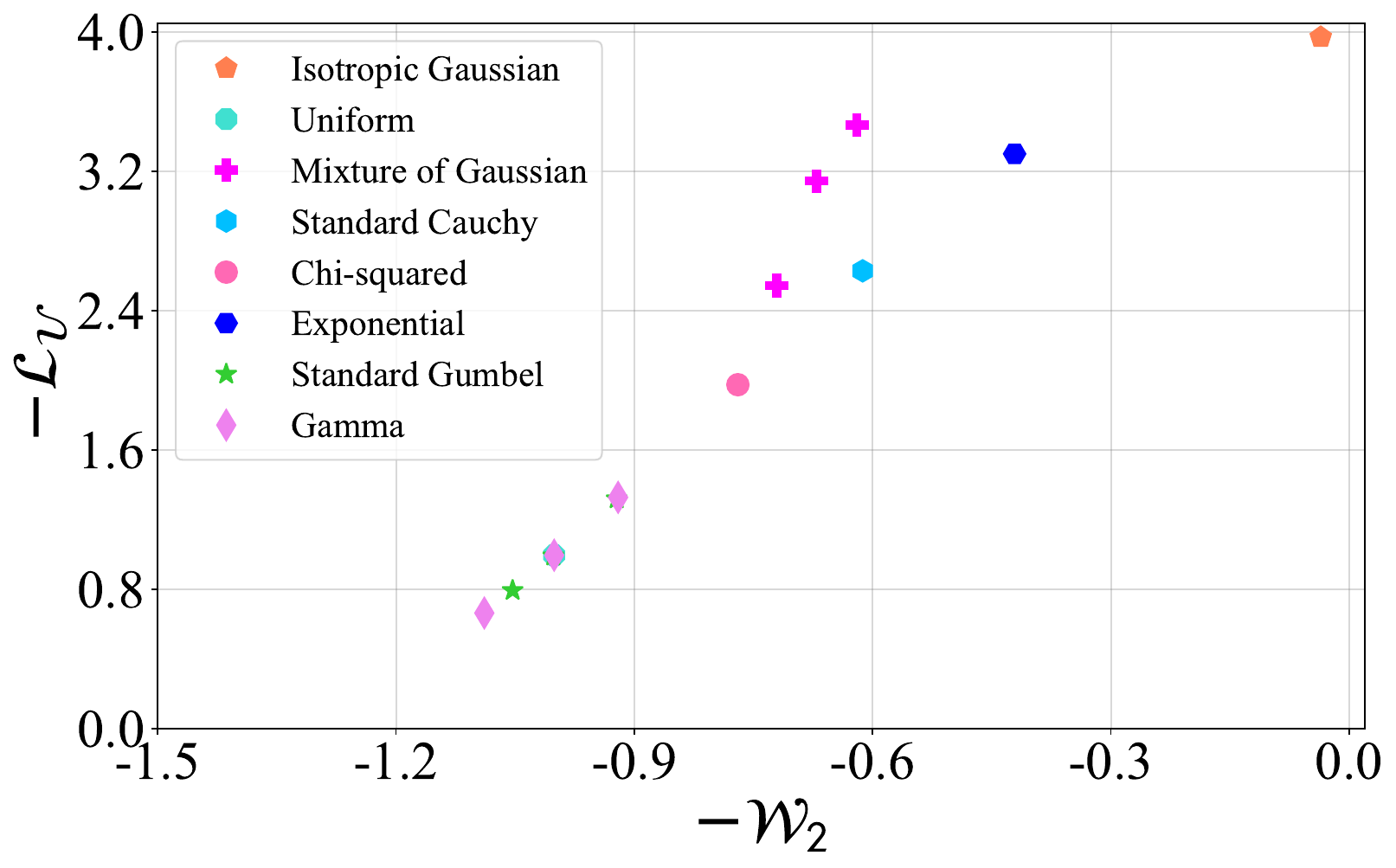}
\vspace{-4mm}
\caption{{Uniformity analysis for various distributions by two metrics.}}
\label{fig:distributions}
\end{center}
\end{wrapfigure} 

We employ synthetic experiments to study the uniformity metrics across different distributions. Specifically, we sample 50,000 data vectors ($m=256$) from different distributions, such as the isotropic Gaussian distribution $\mathcal{N}(\m 0, \m I)$, the uniform distribution on the hyperrectangle $[\m 0, \m 1]$, and the mixture of Gaussians, etc. Then we normalize these data vectors, and estimate the uniformity of different distributions by two metrics. As shown in Fig.~\ref{fig:distributions}, isotropic Gaussian distribution achieves the maximum values for both $-\mathcal{W}_{2}$ and $-\mathcal{L_U}$, which indicates that isotropic Gaussian distribution achieves larger uniformity than other distributions. This empirical result is consistent with  Fact~\ref{proof:maximum uniformity} that the isotropic Gaussian distribution (approximately) achieves the maximum uniformity.

\subsection{On Instance Cloning Constraint}
\label{sec:icc analysis}

\begin{wrapfigure}[11]{r}{0.42\textwidth}
\vspace{-12pt}
\begin{center}
\includegraphics[width=0.41\textwidth]{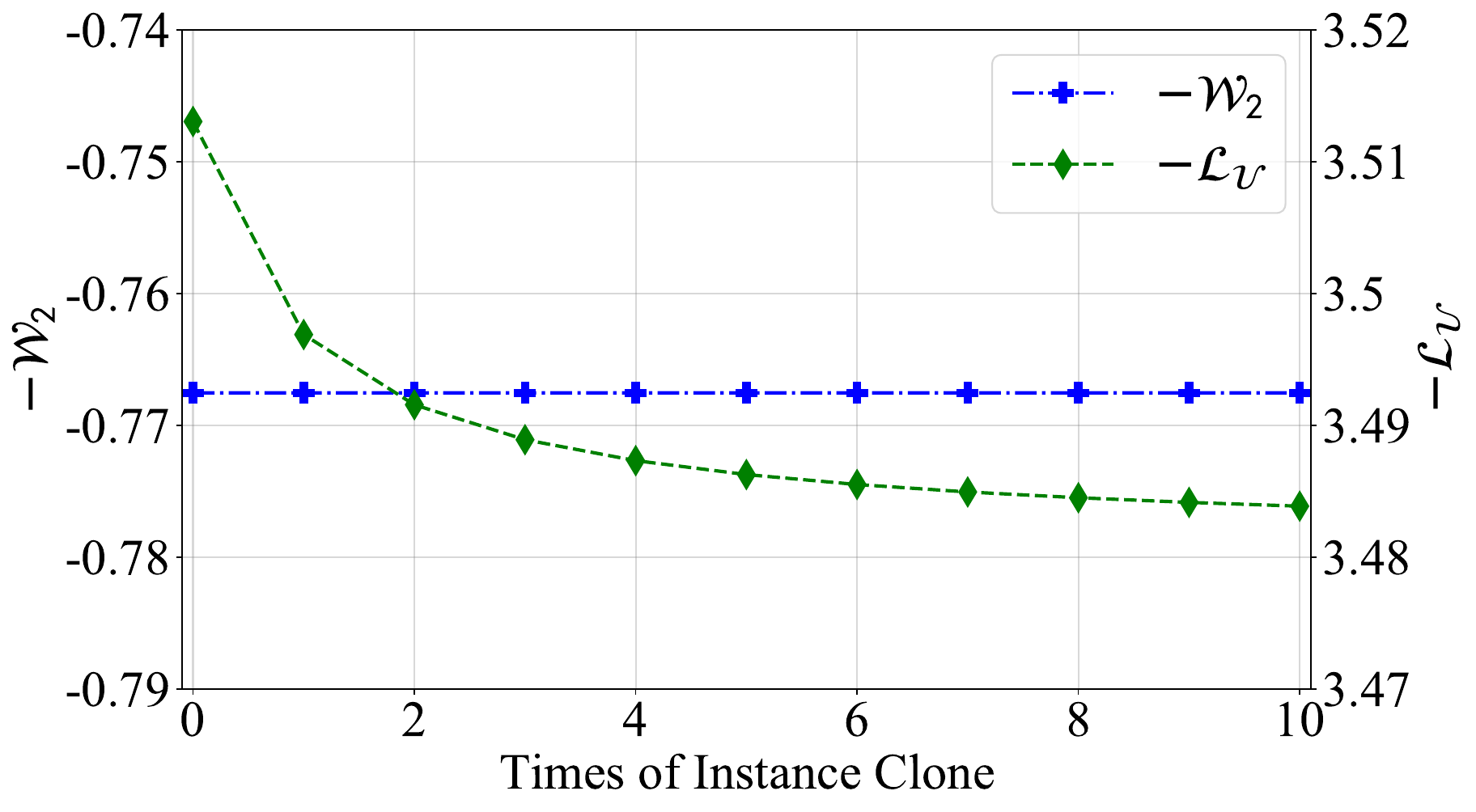}
\vspace{-4mm}
\caption{{ICC analysis.}}
\label{fig:analysisonICC}
\end{center}
\end{wrapfigure} 

In this section, we compare the two metrics in terms of Property \ref{pro:icc} (ICC). Specifically, we randomly sample 1,000 data vectors from the isotropic Gaussian distribution ($m = 32$) and then mask $50\%$ of their coordinates with zeros, forming a new dataset $\mathcal{D}$ with an overload of notation. To investigate the impact of instance cloning, we create multiple clones of the dataset, such as $\mathcal{D}\uplus \mathcal{D}$ and $\mathcal{D}\uplus \mathcal{D}\uplus \mathcal{D}$,  which correspond to one and two times cloning, respectively. We evaluate the two metrics on these datasets. Figure~\ref{fig:analysisonICC} shows that the value of $-\mathcal{L_U}$ slightly decreases as the number of clones increases, indicating that $-\mathcal{L_U}$ violates the equality in Equation~\ref{eq:icc}. In contrast, our proposed metric $-\mathcal{W}_{2}$ remains constant, satisfying the equality.

\subsection{Understanding Property~\ref{pro:fbc}: Why does it relate to dimensional collapse?}

This section delves into Property~\ref{pro:fbc} through case studies. Let us begin with a thought experiment. Consider a dataset $\mathcal{D}$ with instances uniformly distributed on the unit hypersphere, thereby possessing (almost) maximal uniformity. When additional coordinates with zeros are inserted  to each instance of $\mathcal{D}$, forming a new dataset $\mathcal{D} \oplus \m 0^{k}$, it can no longer maintain maximal uniformity. This is because, the new dataset only occupies a small area of the unit hypersphere. Consequently, as $k$ increases, the uniformity of the dataset would decrease significantly.

\begin{figure*}[t]
    \vspace{-3mm}
	\centering
	\subfigure[Two-dimensional visualization with no collapsed dimension]{ 
		\label{fig:2d dimensional collapse}  
		\includegraphics[width=0.22\textwidth]{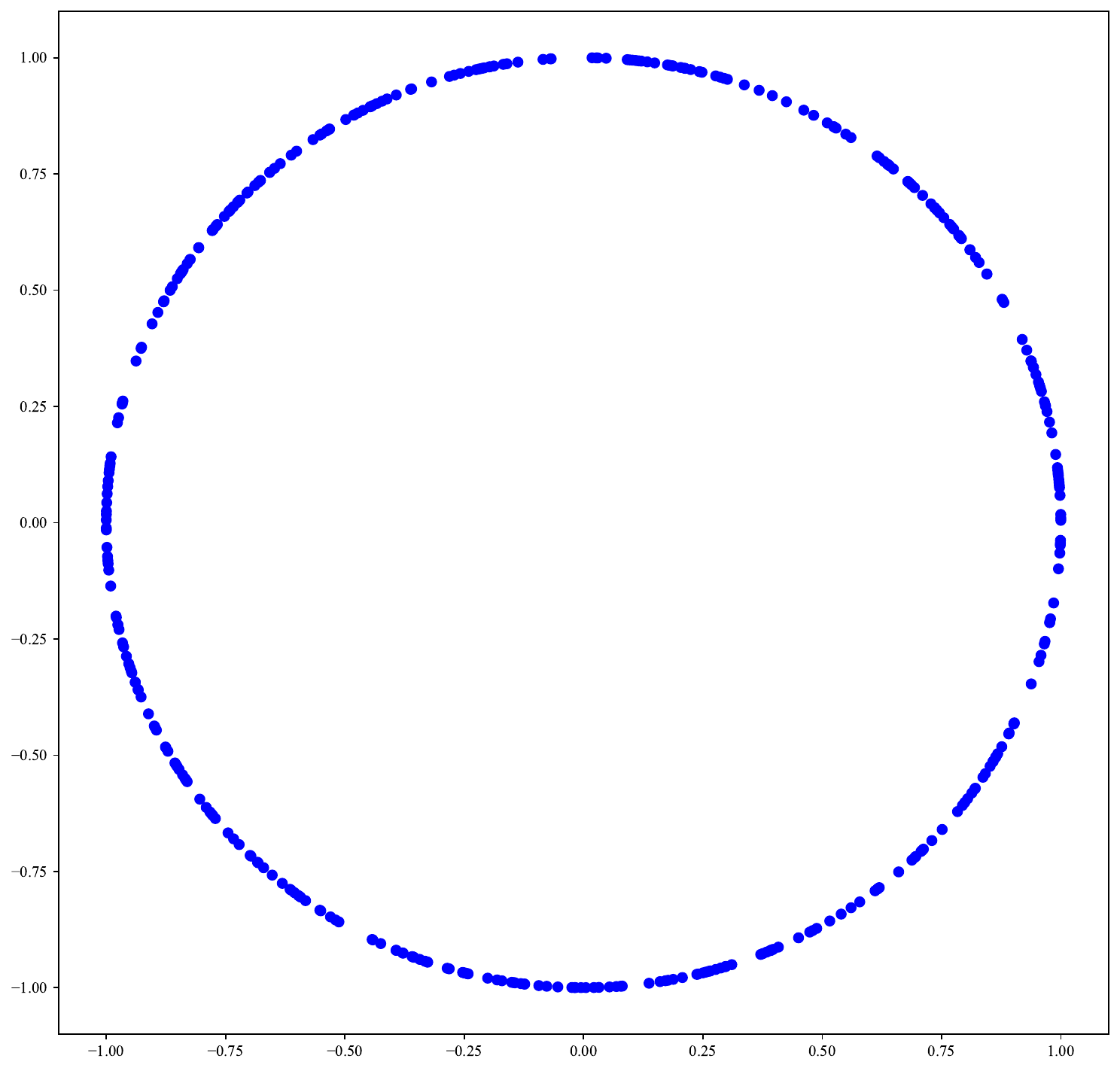}
	}
        \hspace{0.2cm}
	\subfigure[Three-dimensional visualization with one collapsed dimension]{
		\label{fig:3d yes dimensional collapse}
		\includegraphics[width=0.28\textwidth]{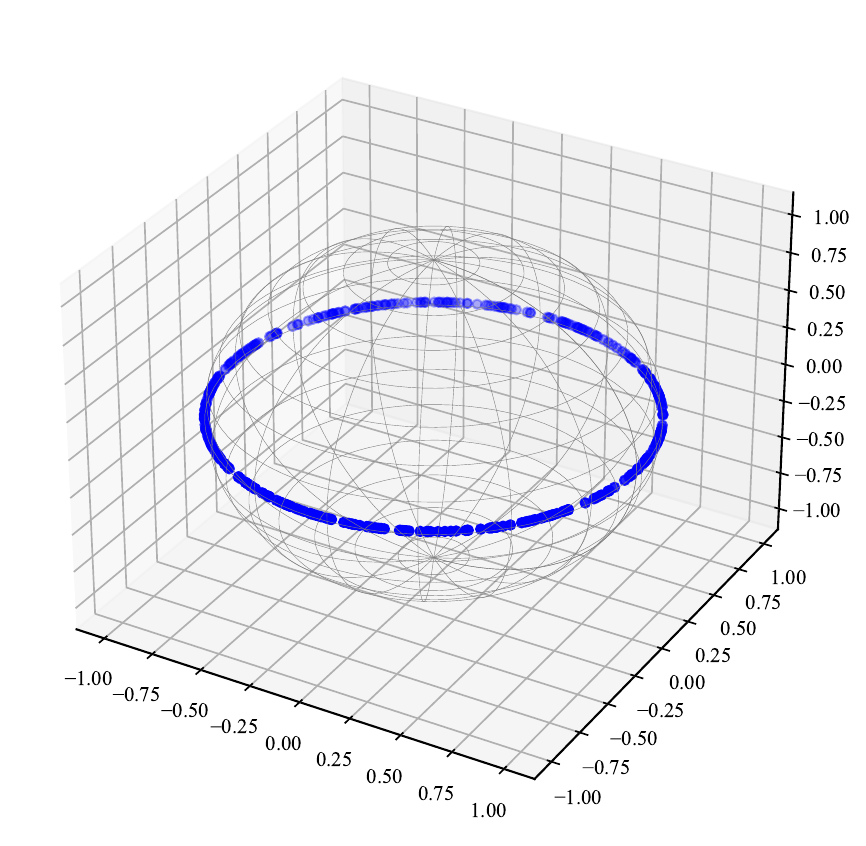}
	}
        \hspace{0.2cm}
	\subfigure[Three-dimensional visualization with no collapsed dimension]{
		\label{fig:3d no dimensional collapse}
		\includegraphics[width=0.28\textwidth]{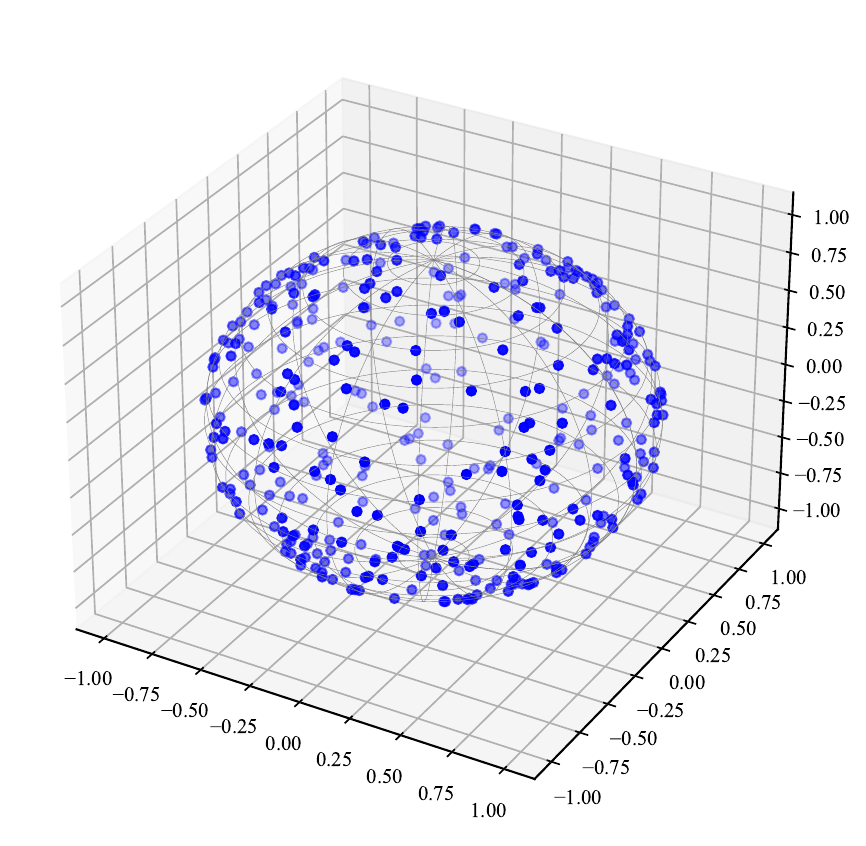}
	}
	\vspace{-4mm}
	\caption{{A case study for Property~\ref{pro:fbc} and blue points are data vectors.}}
	\label{fig:visualization for dimensional collapse analysis}
\end{figure*}

Let us visualize this thought experiment using synthetic studies. In Figure~\ref{fig:2d dimensional collapse}, we present 400 data vectors ($\mathcal{D}_1$) sampled from $\mathcal{N}(\m 0, \m I_2)$, which are also nearly uniformly distributed on $\mathcal{S}^{1}$. By inserting one zero-coordinate to each instance of $\mathcal{D}_1$, we obtain a new dataset $\mathcal{D}_1 \oplus \m 0^{1}$, as depicted in Figure~\ref{fig:3d yes dimensional collapse}. We also construct another dataset $\mathcal{D}_2$ consisting of 400 data vectors sampled from $\mathcal{N}(\m 0, \m I_3)$, visualized in Figure~\ref{fig:3d no dimensional collapse}. Notably, $\mathcal{D}_1 \oplus \m 0^{1}$ forms a ring on $\mathcal{S}^{2}$, while $\mathcal{D}_2$ is almost uniformly distributed over $\mathcal{S}^{2}$. Naturally, $\mathcal{U}(\mathcal{D}_2) > \mathcal{U}(\mathcal{D}_1 \oplus \m 0^{1})$. If $\mathcal{U}(\mathcal{D}_1) = \mathcal{U}(\mathcal{D}_2)$\footnote{Intuitively, maximal uniformity should stay constant regardless of dimensions; otherwise the corresponding uniformity metric exhibit a preference for larger or smaller dimensions. }, then $\mathcal{U}(\mathcal{D}_1) = \mathcal{U}(\mathcal{D}_2) > \mathcal{U}(\mathcal{D}_1 \oplus \m 0^{1})$. This partially confirms the validity of Property \ref{pro:fbc}.

\begin{wrapfigure}[11]{r}{0.42\textwidth}
\vspace{-15pt}
\begin{center}
\includegraphics[width=0.40\textwidth]{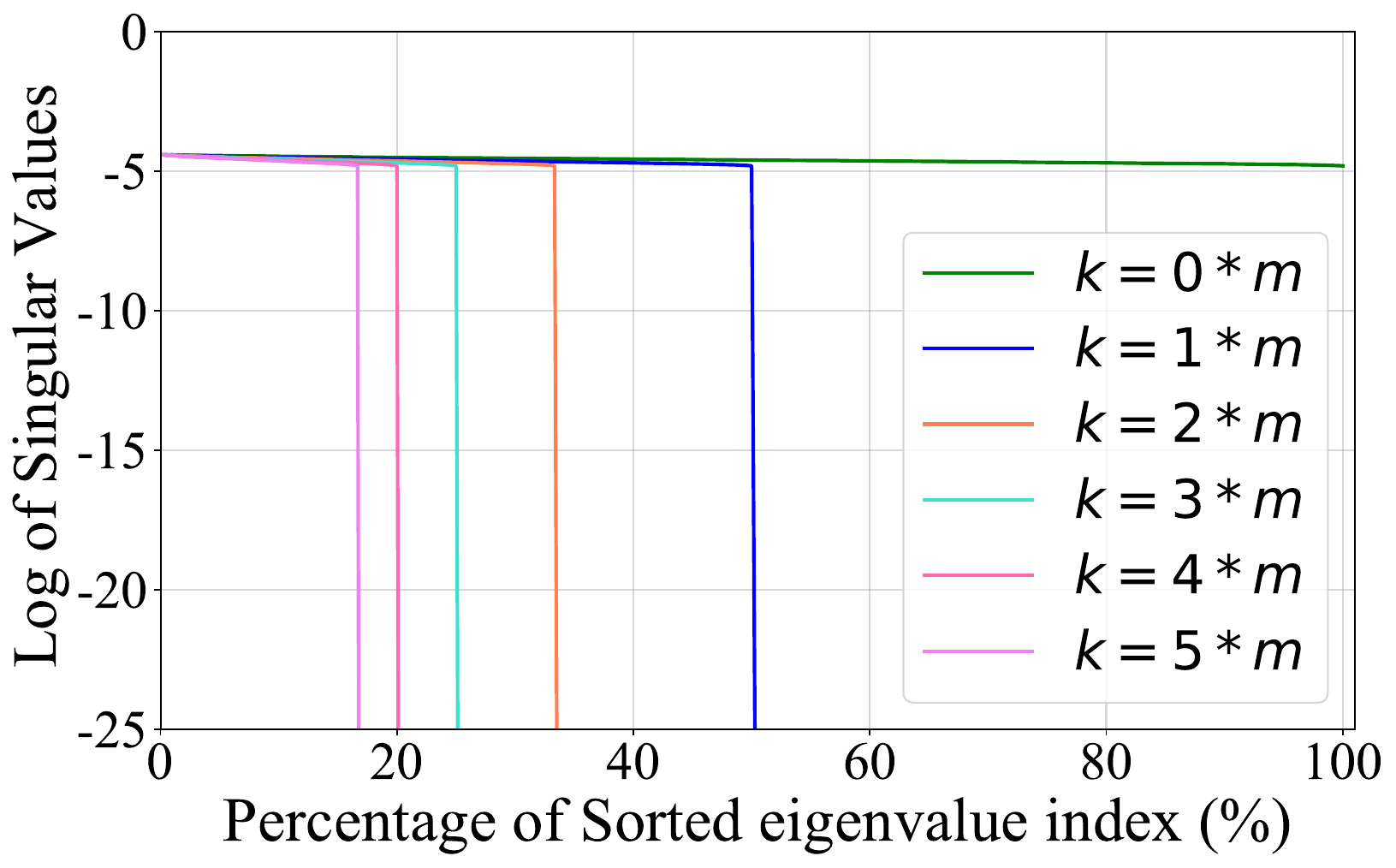}
\vspace{-4mm}
\caption{Singular value spectrum of $\mathcal{D} \oplus \m 0^{k}$.}
\label{fig:property 5 svd}
\end{center}
\end{wrapfigure} 

Additionally, increasing the value of $k$ in Property~\ref{pro:fbc} exacerbates the degree of dimensional collapse. To illustrate, consider a dataset $\mathcal{D}$ sampled from a multivariate Gaussian distribution $\mathcal{N}(\mathbf{0}, \mathbf{I}_m/m)$, exhibiting a collapse degree close to $0\%$. However, upon inserting $m$-dimensional zero-value vectors to each instance of  $\mathcal{D}$, denoted as $\mathcal{D} \oplus \mathbf{0}^m$, half of the dimensions collapse. Consequently, the collapse degree increases to $50\%$. Figure~\ref{fig:property 5 svd} visually represents the collapse of $\mathcal{D} \oplus \mathbf{0}^k$ using the singular value spectra of the representations. It is evident that a larger $k$ results in a more pronounced dimensional collapse. In summary, Property~\ref{pro:fbc} corresponds to dimensional collapse.

\subsection{Understanding $\mathcal{W}_{2}$: Large means may lead to collapse}
\label{sec:role of mean}

\begin{figure*}[!ht]
	\centering
	\subfigure[$\mathcal{N}(\m 0, \m I_{2})$]{ 
		\label{fig:mean 0)}  
		\includegraphics[width=0.22\textwidth]{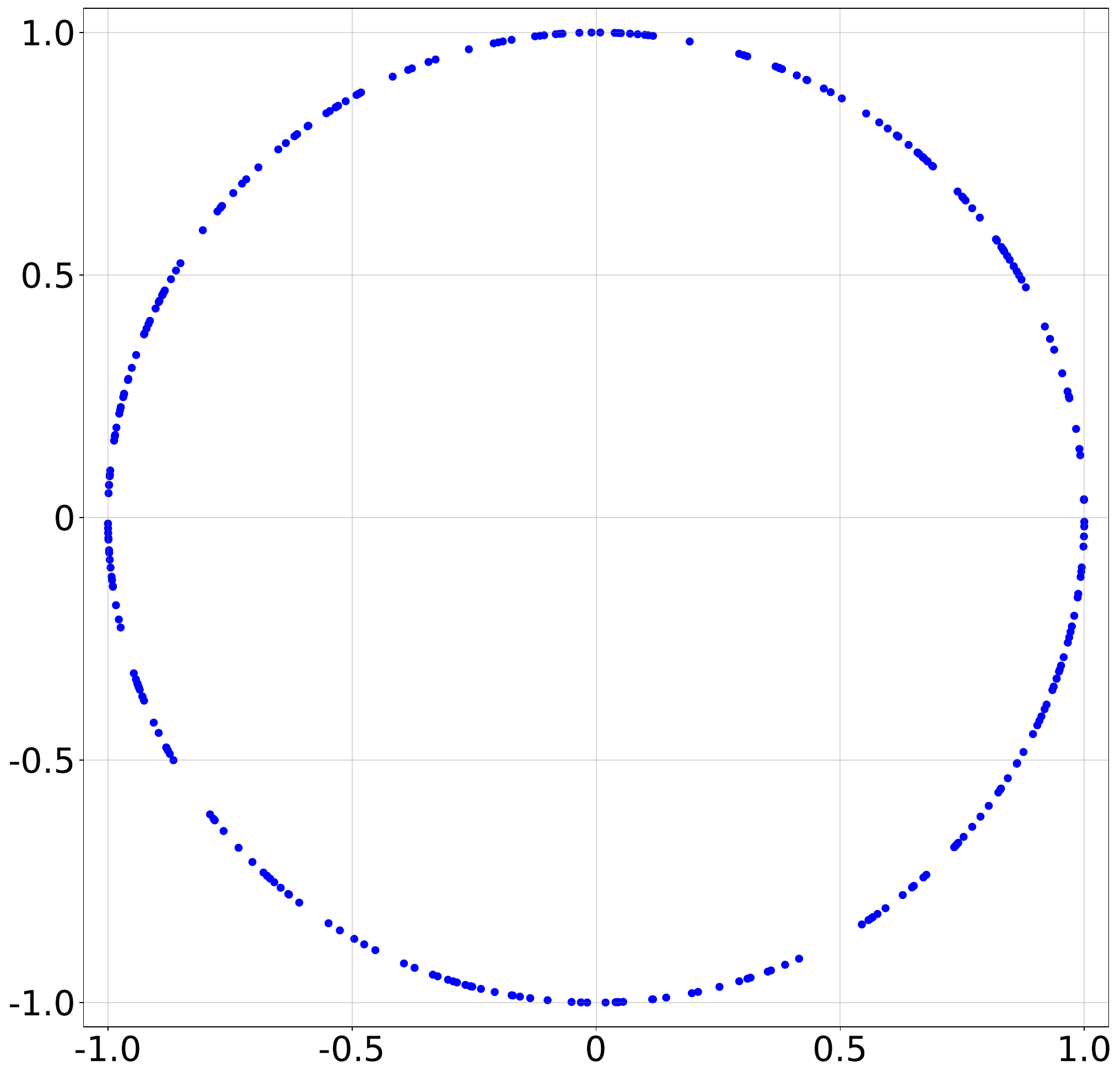}
	}
	\hspace{-0.2cm}
	\subfigure[$\mathcal{N}(0.5\cdot \m {1}, \m I_{2})$]{
		\label{fig:mean 0.5}
		\includegraphics[width=0.22\textwidth]{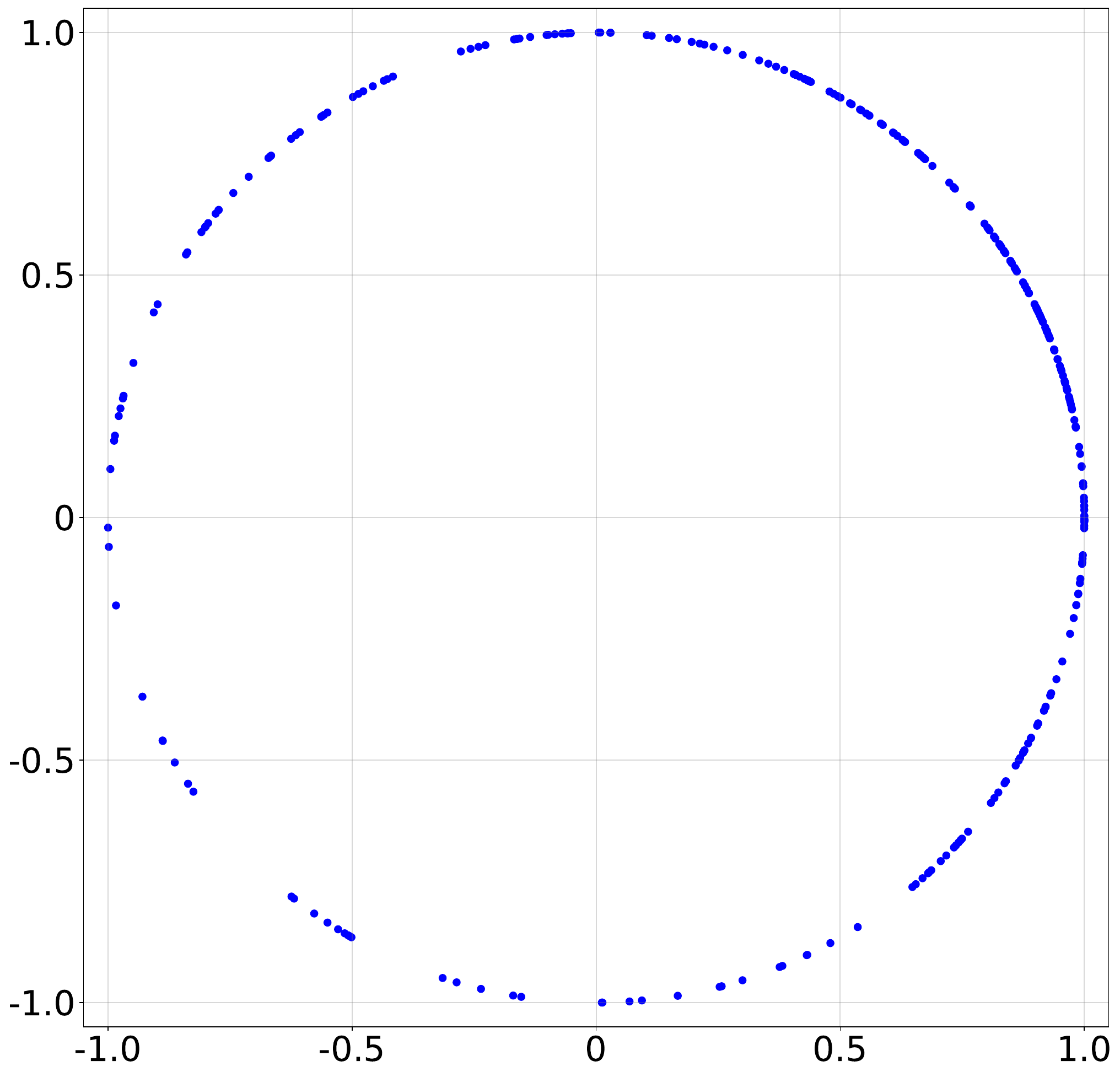}
	}
	\hspace{-0.2cm}
	\subfigure[$\mathcal{N}(1\cdot \m {1}, \m I_{2})$]{
		\label{fig:mean 1.0}
		\includegraphics[width=0.22\textwidth]{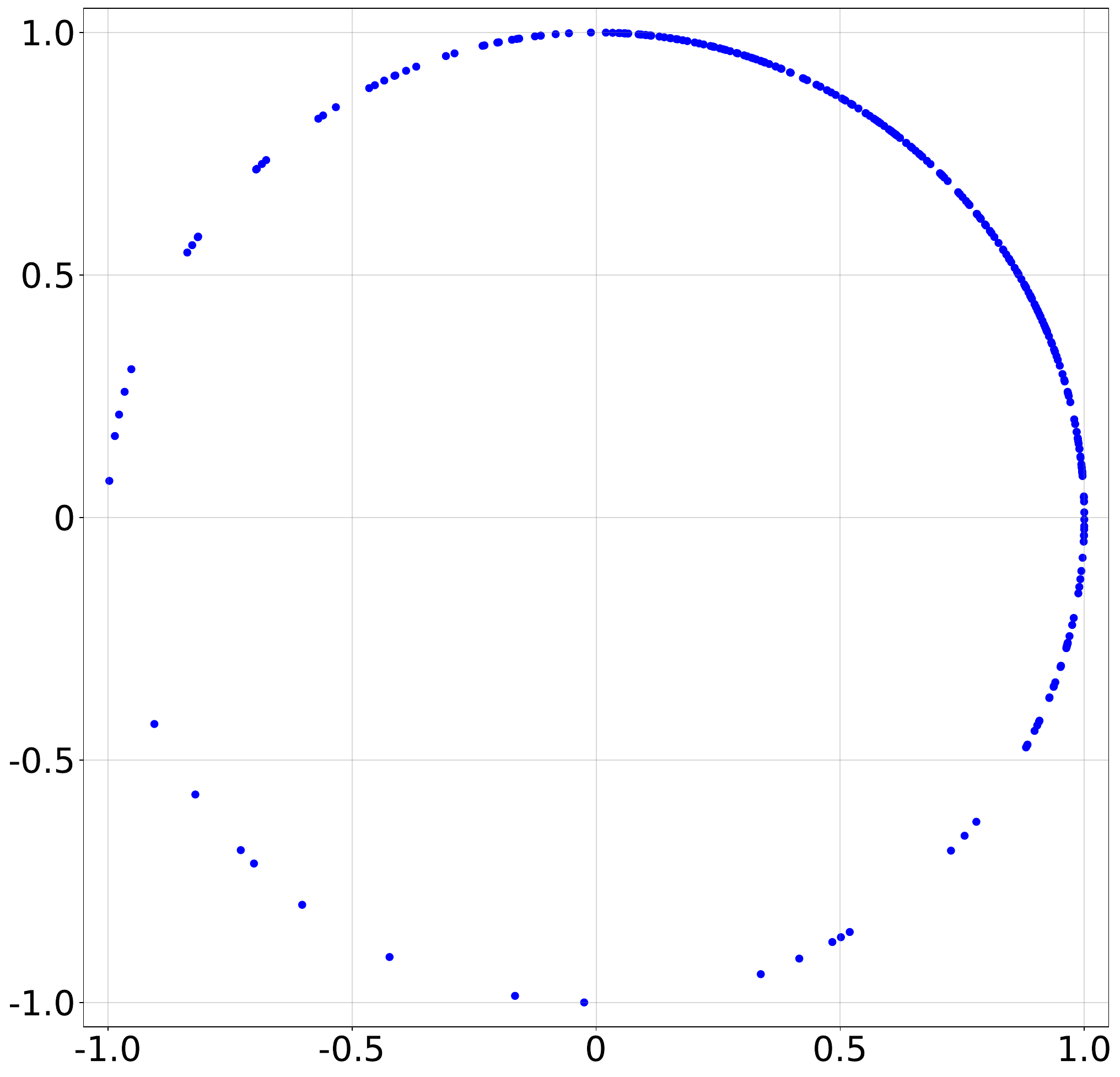}
	}
	\centering
	\subfigure[$\mathcal{N}( 2\cdot\m {1}, \m I_{2})$]{ 
		\label{fig:mean 2.0}  
		\includegraphics[width=0.22\textwidth]{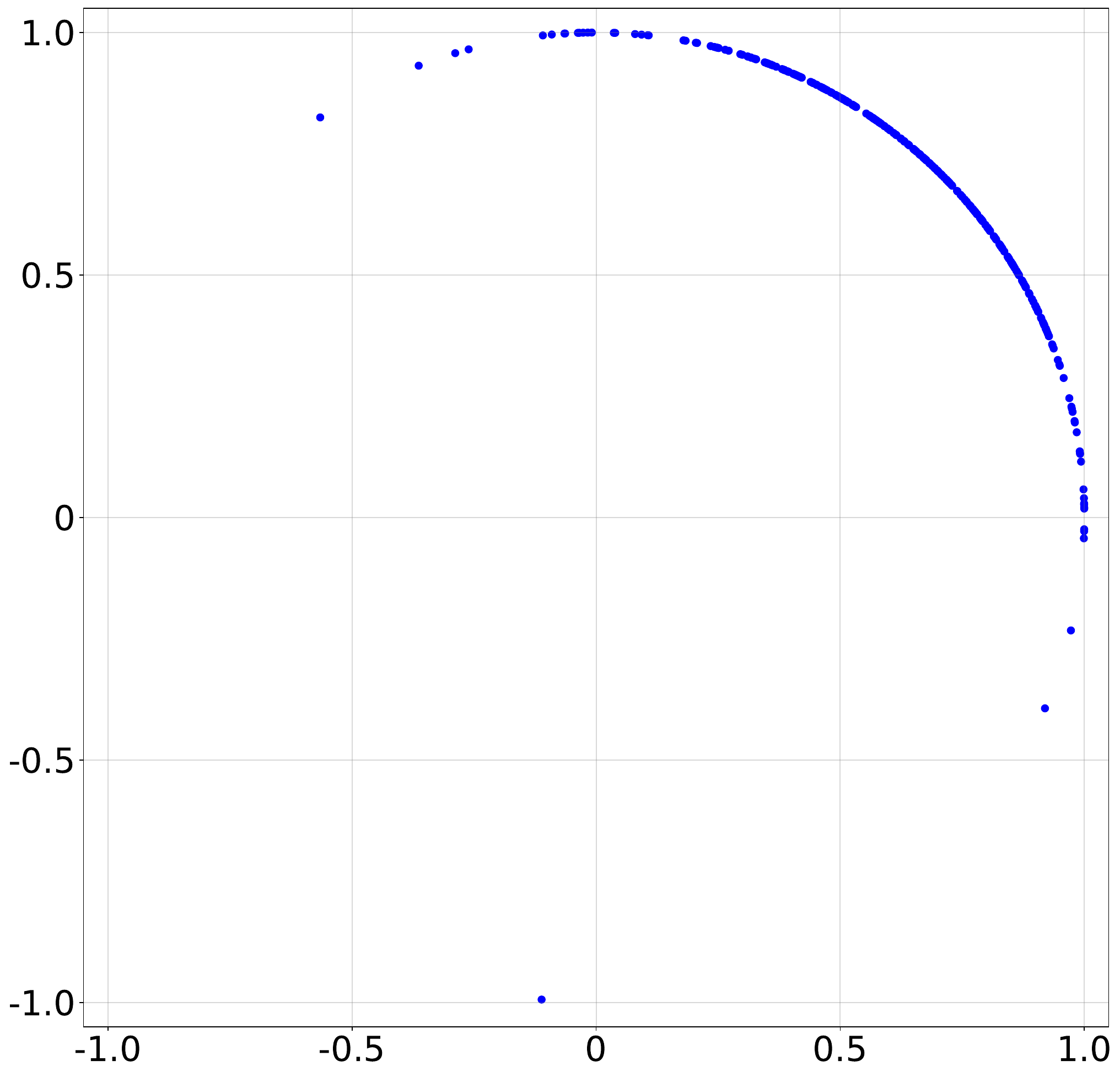}
	}
    \vspace{-4mm}
	\subfigure[$\mathcal{N}(4 \cdot\m {1}, \m I_{2})$]{
		\label{fig:mean 4.0}
		\includegraphics[width=0.22\textwidth]{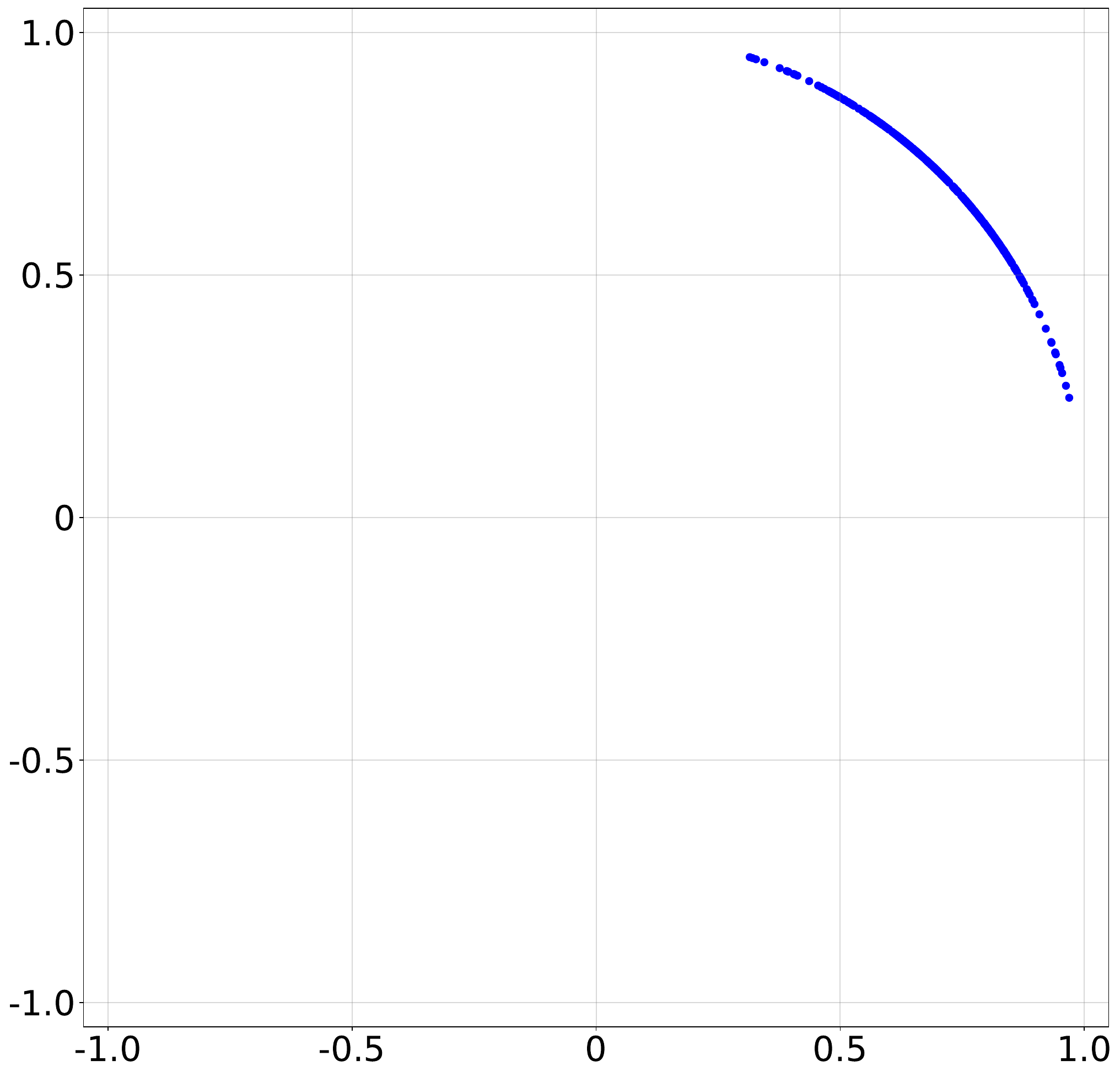}
	}
	\hspace{-0.2cm}
	\subfigure[$\mathcal{N}(8 \cdot\m {1}, \m I_{2})$]{
		\label{fig:mean 8.0}
		\includegraphics[width=0.22\textwidth]{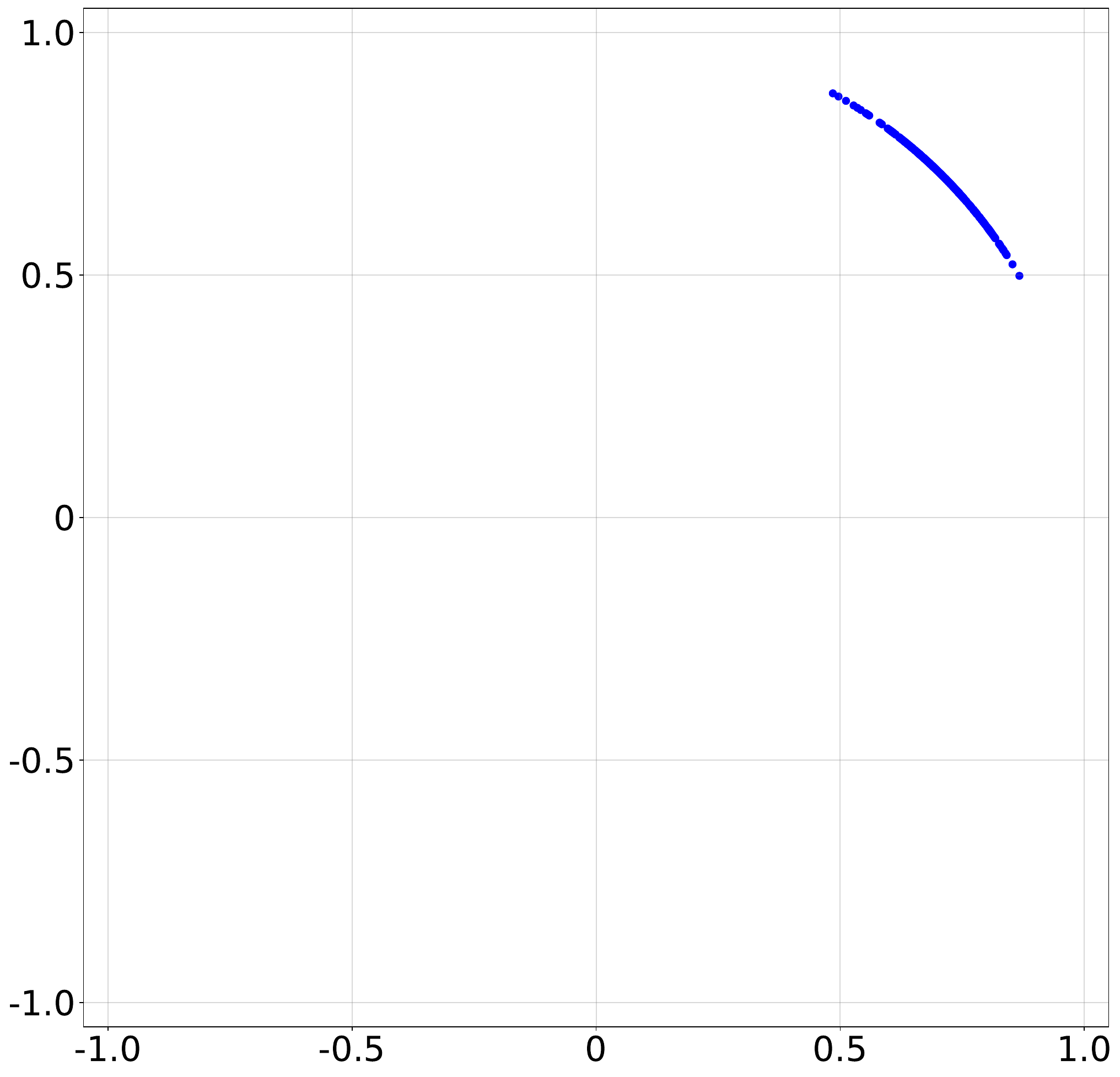}
	}
    \hspace{-0.2cm}
	\subfigure[$\mathcal{N}(16 \cdot\m {1}, \m I_{2})$]{
		\label{fig:mean 16.0}
		\includegraphics[width=0.22\textwidth]{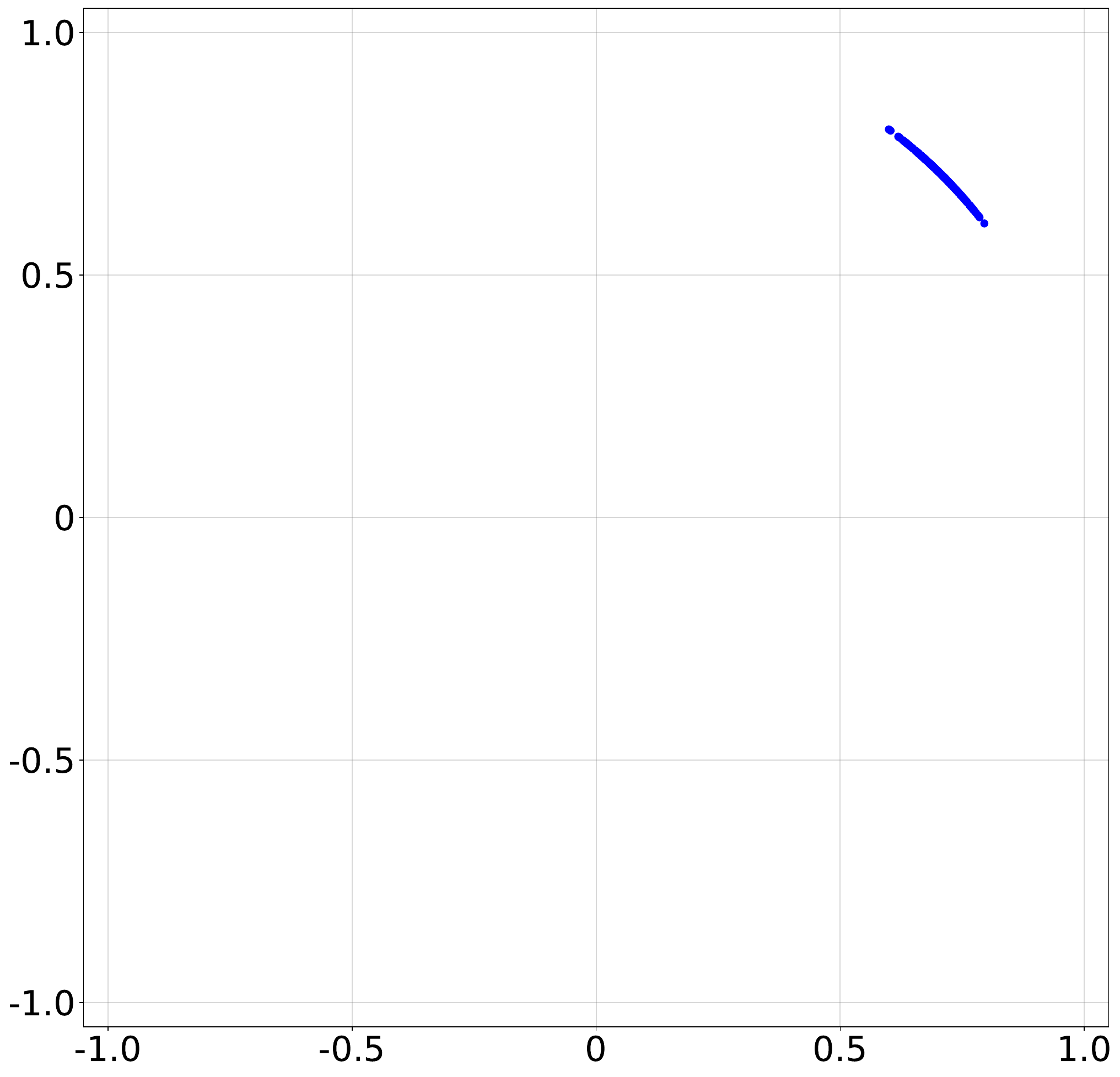}
	}
    \hspace{-0.2cm}
	\subfigure[$\mathcal{N}(32 \cdot\m {1}, \m I_{2})$]{
		\label{fig:mean 32.0}
		\includegraphics[width=0.22\textwidth]{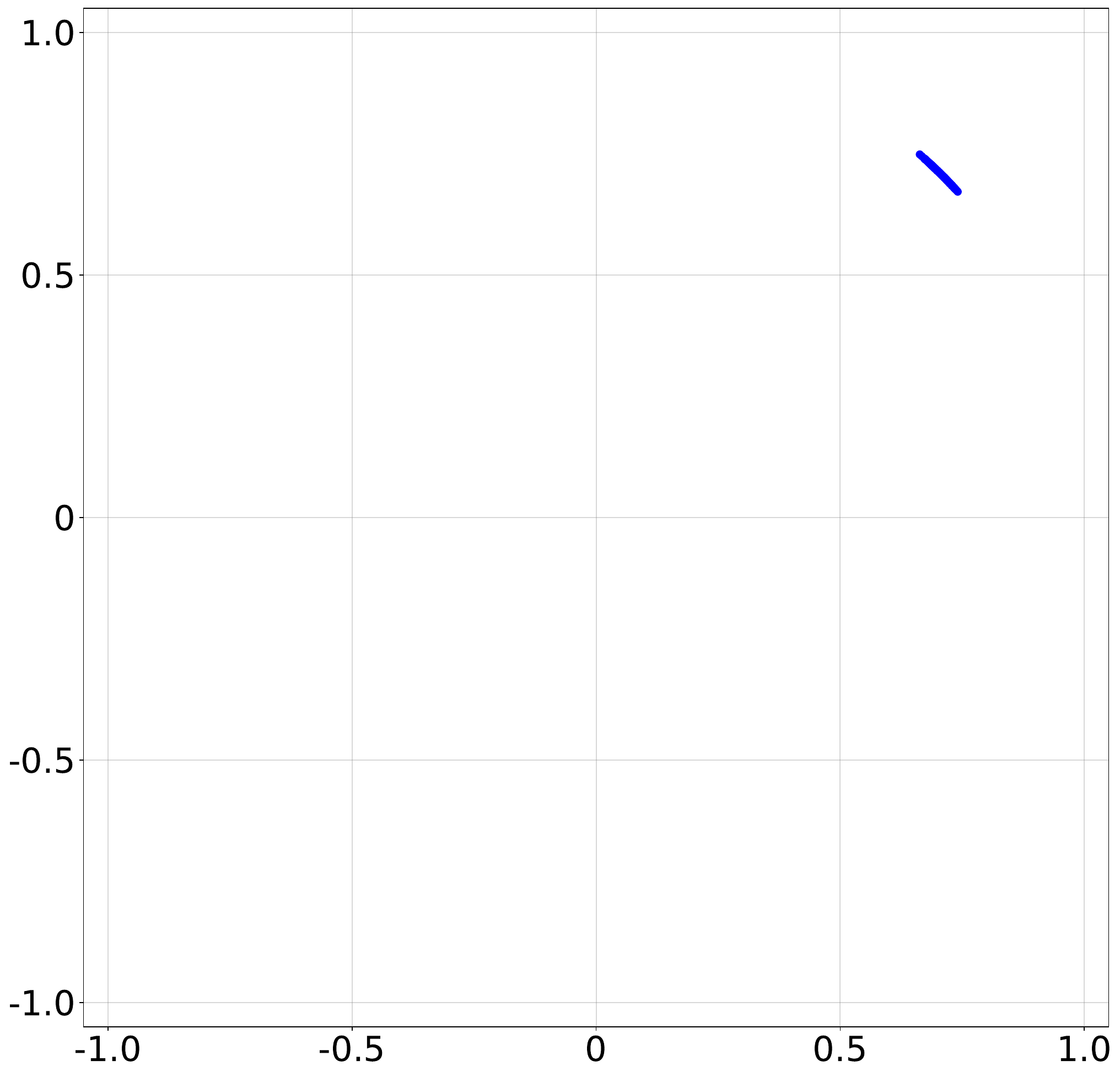}
	}
	\caption{Visualizing $\ell_2$ normalized Gaussian vectors with different means.}
	\label{fig:normalized gaussian distribution}
	\vspace{-2mm}
\end{figure*}

In this section, we explore our uniformity loss $\mathcal{W}_{2}$. This loss embodies two primary constraints. Firstly, it promotes the covariance matrix to be isotropic (specifically $\mathbf{I}_m/m$). Secondly, it enforces the mean to be zero. The latter constraint on the mean is crucial. To illustrate, we present a case study demonstrating that deviating the mean from zero compromises uniformity, even if the covariance matrix is precisely $\mathbf{I}_m/m$ and thus isotropic. Means deviating from zero may result in dimensional collapse and even constant collapse.


Assuming $\mathbf{X}\in \mathbb{R}^2$ follows a Gaussian distribution $\mathcal{N}(\mathbf{0}, \mathbf{I}_2)$, let $ \mathbf{Y} = \mathbf{X} + k\cdot \mathbf{1}$ such that $ \mathbf{Y} \sim \mathcal{N}(k \cdot\mathbf{1}, \mathbf{I}_2)$, where $\mathbf{1} \in \mathbb{R}^k$ represents a vector of all ones. We vary $k$ from $0$ to $32$ and visualize the $\ell_2$-normalized $\mathbf{Y}$'s in Figure~\ref{fig:normalized gaussian distribution} (by generating multiple independent copies).  It is clear that an excessively large means will cause representations to collapse to a single point, even if the covariance matrix is isotropic.
    
\section{Experiment settings and convergence analysis}\label{sec:setting and convergence}

\subsection{Experiment settings}
\label{sec:parameter setting}

To ensure fair comparisons, all experiments in Section~\ref{sec:experiment} are conducted on a single 1080 GPU. Additionally, we maintain consistency in network architecture across all models, utilizing ResNet-18~\citep{He2016DeepRL} as the backbone and a three-layer MLP as the projector. The LARS optimizer~\citep{You2017ScalingSB} is employed with a base learning rate of $0.2$, accompanied by a cosine decay learning rate schedule~\citep{Loshchilov2017SGDRSG} for all models. Evaluation follows a linear evaluation protocol, where models are pre-trained for 500 epochs. Evaluation involves adding a linear classifier and training the classifier for 100 epochs while preserving the learned representations. The same augmentation strategy is deployed across all models, encompassing various operations such as color distortion, rotation, and cutout. Following~\cite{Costa2022SololearnAL}, we set the temperature $t=0.2$ for all contrastive learning methods. For MoCo~\citep{He2020MomentumCF} and NNCLR~\citep{Dwibedi2021WithAL}, which require an additional queue to store negative samples, we set the queue size to $2^{12}$. Regarding the linear decay for weighting the quadratic Wasserstein distance, refer to Table~\ref{table:parameter setting} for the parameter settings.

\begin{table*}[t]
\centering
\caption{Parameter settings for various models in the experiments.}
\label{table:parameter setting}
\resizebox{0.60\textwidth}{!}{
\begin{tabular}{l c c c c} \hline
Models & MoCo v2 & BYOL & BarlowTwins & Zero-CL \\\hline
$\alpha_{\max}$ & 1.0 & 0.2 & 30.0 & 30.0 \\
$\alpha_{\min}$ & 1.0 & 0.2 & 0 & 30.0 \\\hline
\end{tabular}}
\end{table*}

\subsection{Convergence analysis for Top-1 accuracy}
\label{appendix:convergence analysis}

Here we illustrate the convergence of Top-1 accuracy across all training epochs in Fig~\ref{fig:convergence}. Throughout the training, we capture the model checkpoint at the end of each epoch to train a linear classifier. We subsequently evaluate the Top-1 accuracy on unseen images from the test set (either CIFAR-10 or CIFAR-100).

For both CIFAR-10 and CIFAR-100, we observe that integrating the proposed uniformity metric as an auxiliary loss significantly enhances the Top-1 accuracy, particularly in the initial stages of training.

\subsection{Convergence analysis for uniformity and alignment}\label{appendix:representation analysis}

This section presents the convergence of the uniformity metric  and alignment loss across all training epochs in Figure~\ref{fig:convergence on uniformity} and Figure~\ref{fig:convergence on alignment}, respectively. Throughout the training, we record the model checkpoint at the end of each epoch to evaluate the uniformity using the proposed metric $\mathcal{W}_{2}$ and alignment \citep{Wang2020UnderstandingCR} on unseen images from the test set (either CIFAR-10 or CIFAR-100).

For both CIFAR-10 and CIFAR-100, we observe that integrating the proposed uniformity metric as an auxiliary loss significantly improves uniformity. However, it also slightly compromises alignment (where a smaller alignment loss indicates better alignment). It should be noted that improved uniformity often leads to worse alignment.

\begin{figure*}[h]
    \vspace{-4mm}
	\centering
	\subfigure[MoCo v2 on CIFAR-10]{ 
		\label{fig:moco v2 on cifar-10}  
		\includegraphics[width=0.3\textwidth]{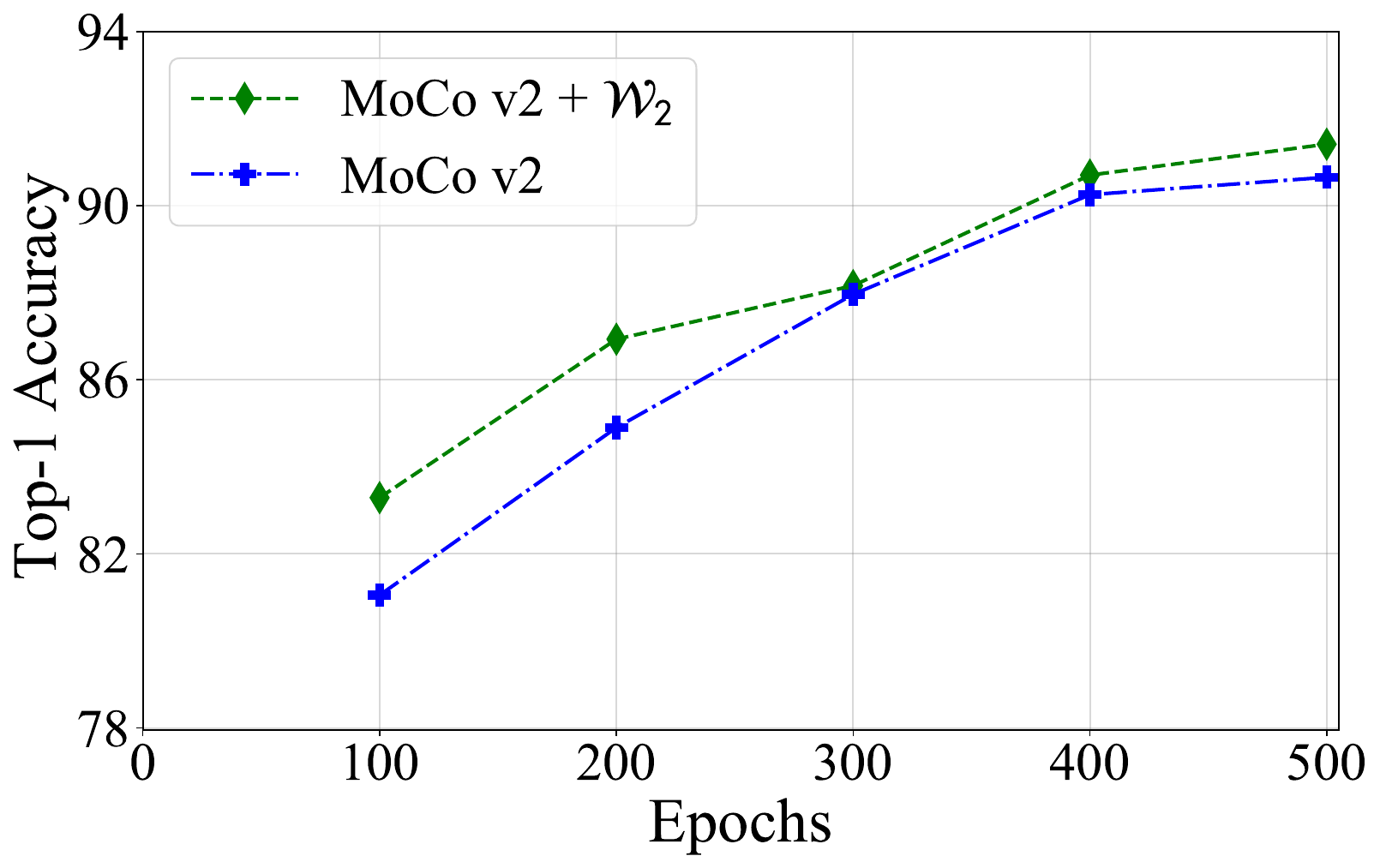}
	}
	\hspace{-0.2cm}
	\subfigure[BYOL on CIFAR-10]{
		\label{fig:byol on cifar-10}
		\includegraphics[width=0.3\textwidth]{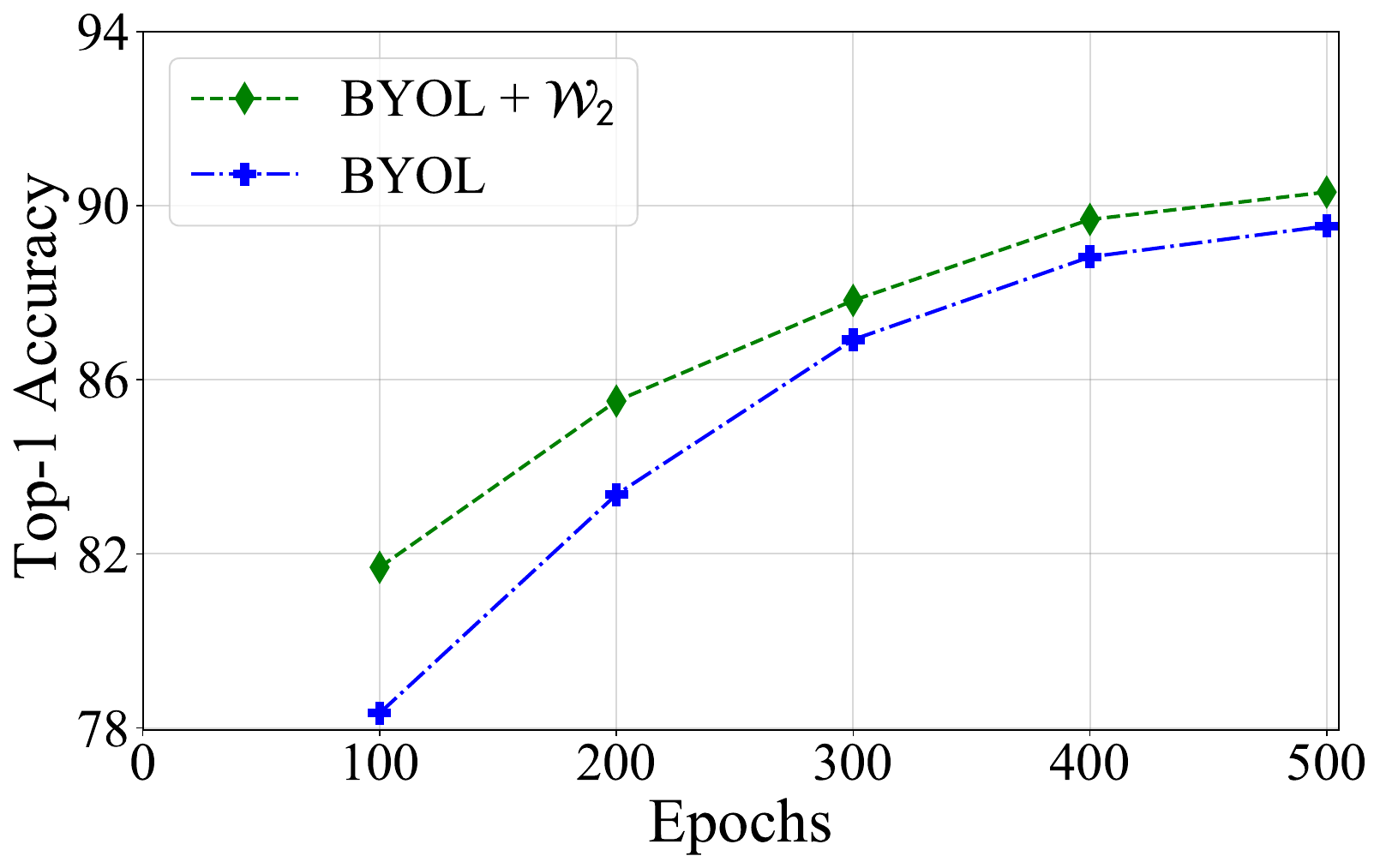}
	}
	\hspace{-0.2cm}
	\subfigure[BarlowTwins on CIFAR-10]{
		\label{fig:barlowtwins on cifar-10}
		\includegraphics[width=0.3\textwidth]{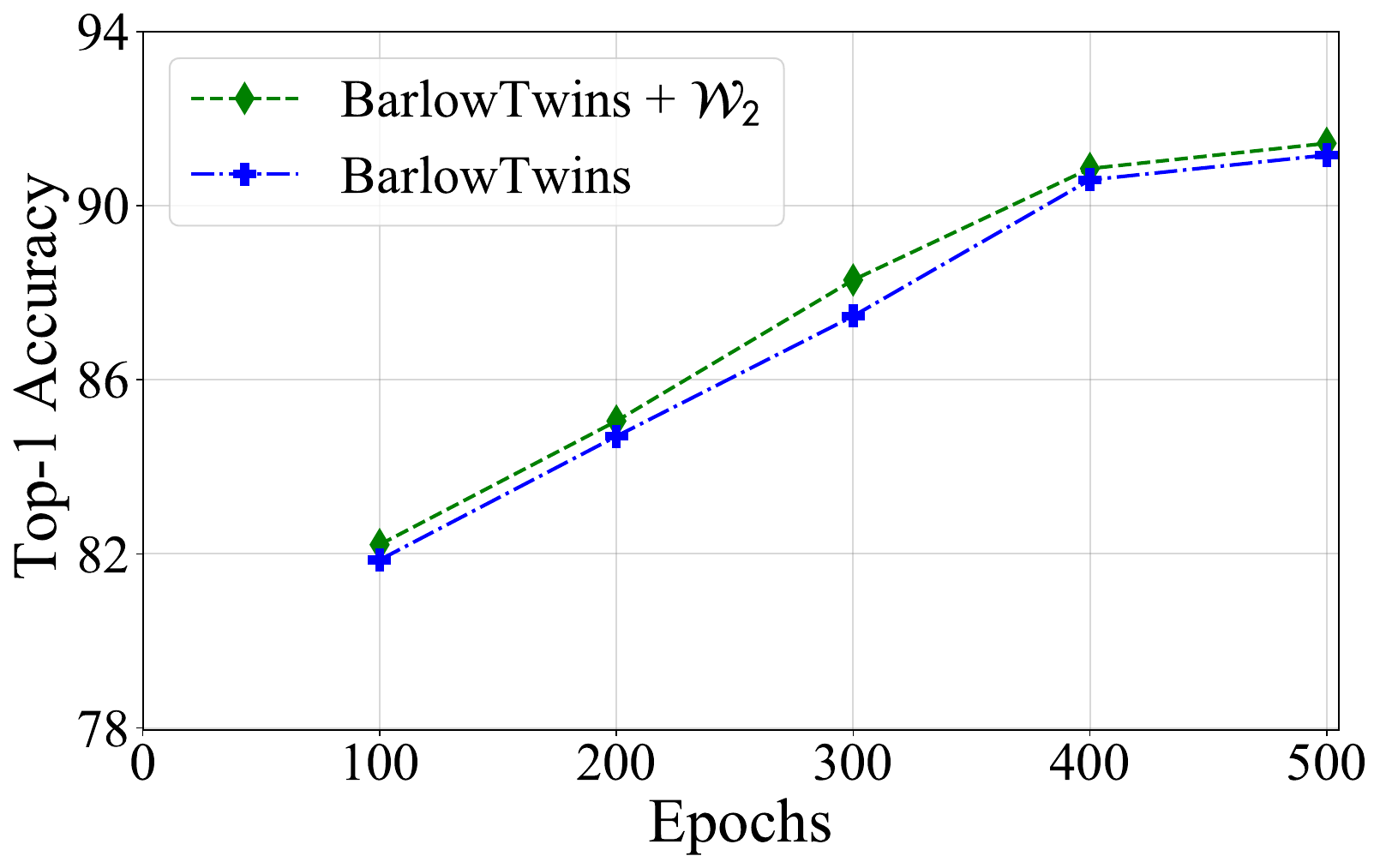}
	}
	\vspace{-2mm}
	\subfigure[MoCo v2 on CIFAR-100]{ 
		\label{fig:moco v2 on cifar-100}  
		\includegraphics[width=0.3\textwidth]{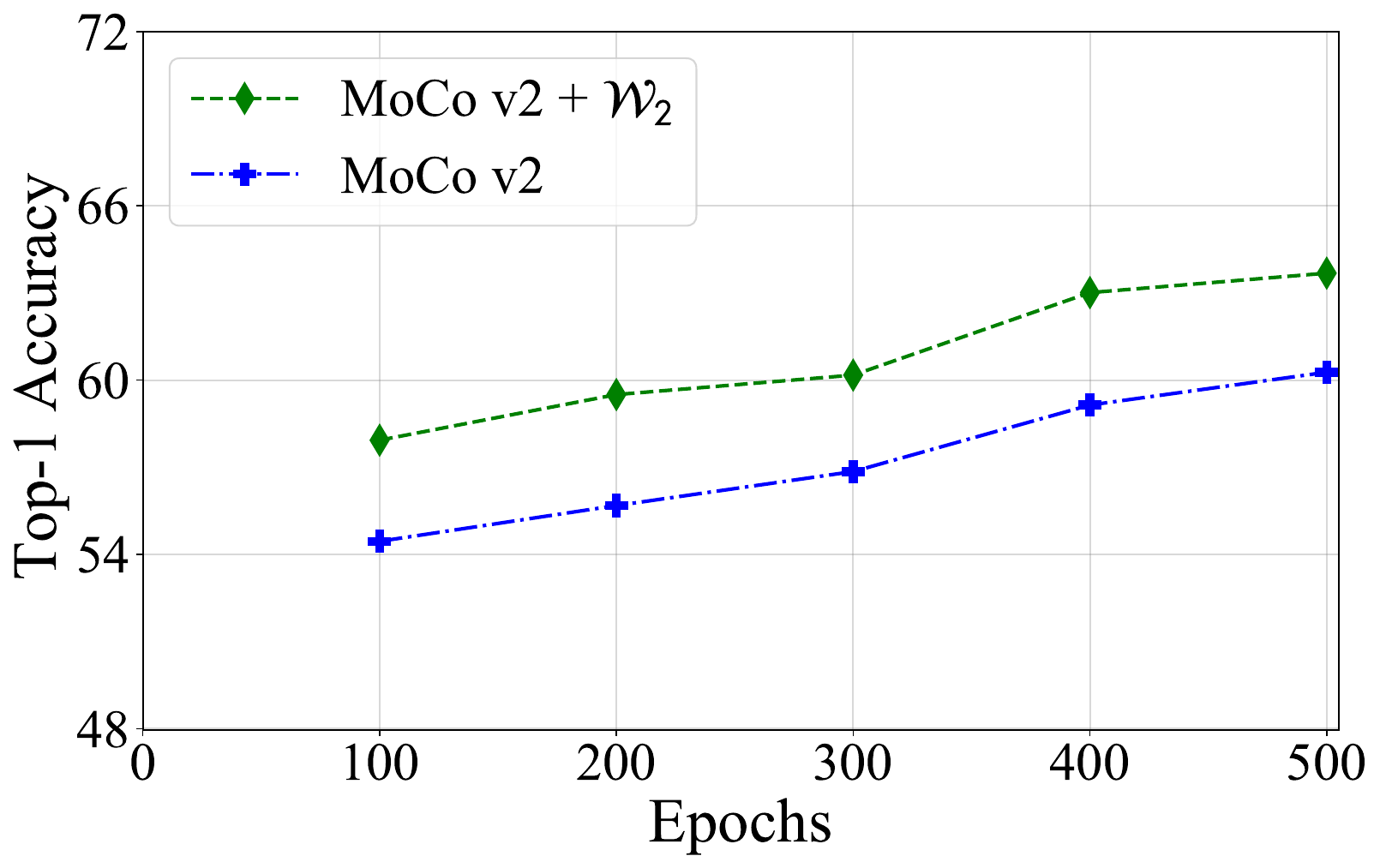}
	}
	\hspace{-0.2cm}
	\subfigure[BYOL on CIFAR-100]{
		\label{fig:byol on cifar-100}
		\includegraphics[width=0.3\textwidth]{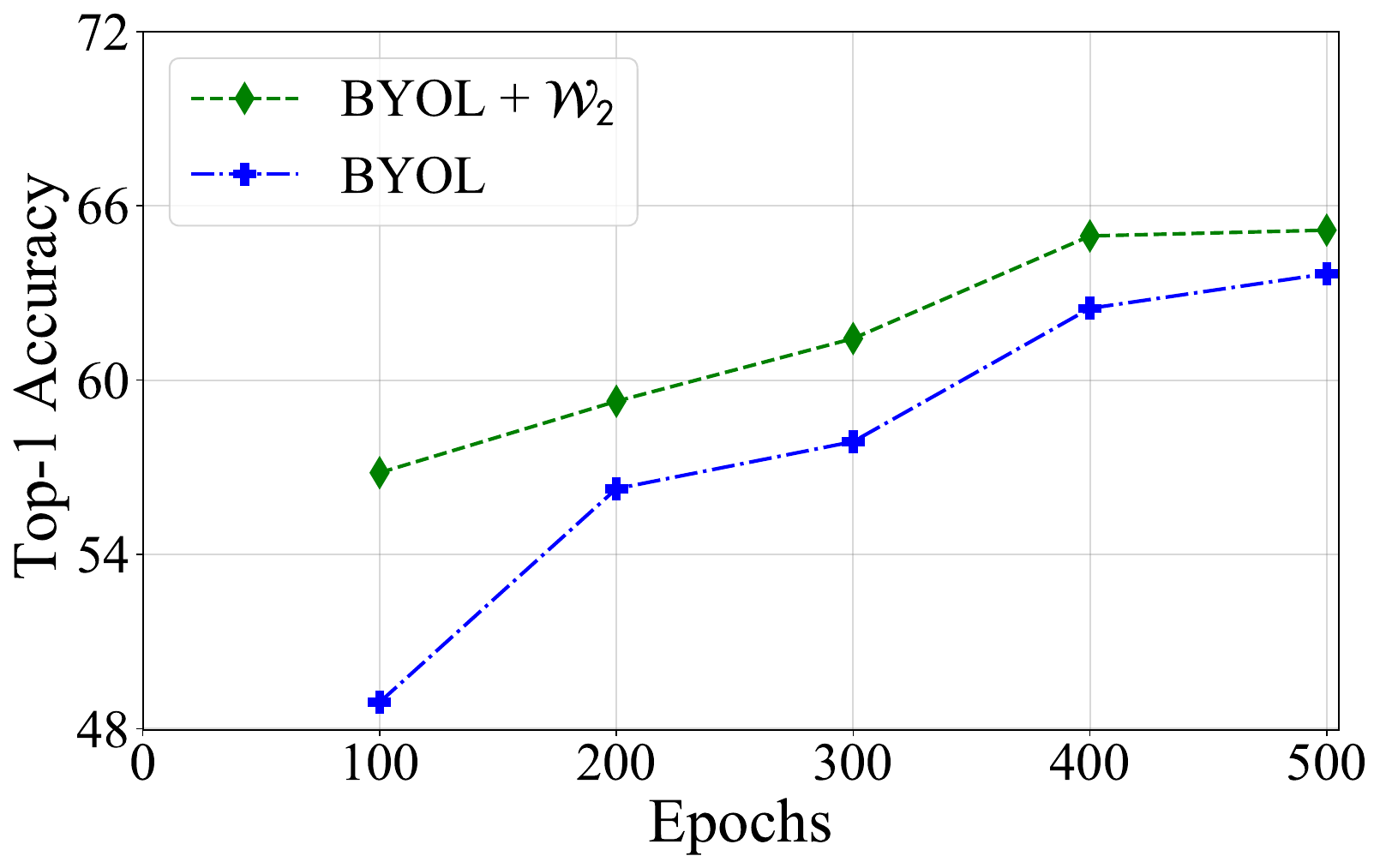}
	}
	\hspace{-0.2cm}
	\subfigure[BarlowTwins on CIFAR-100]{
		\label{fig:barlowtwins on cifar-100}
		\includegraphics[width=0.3\textwidth]{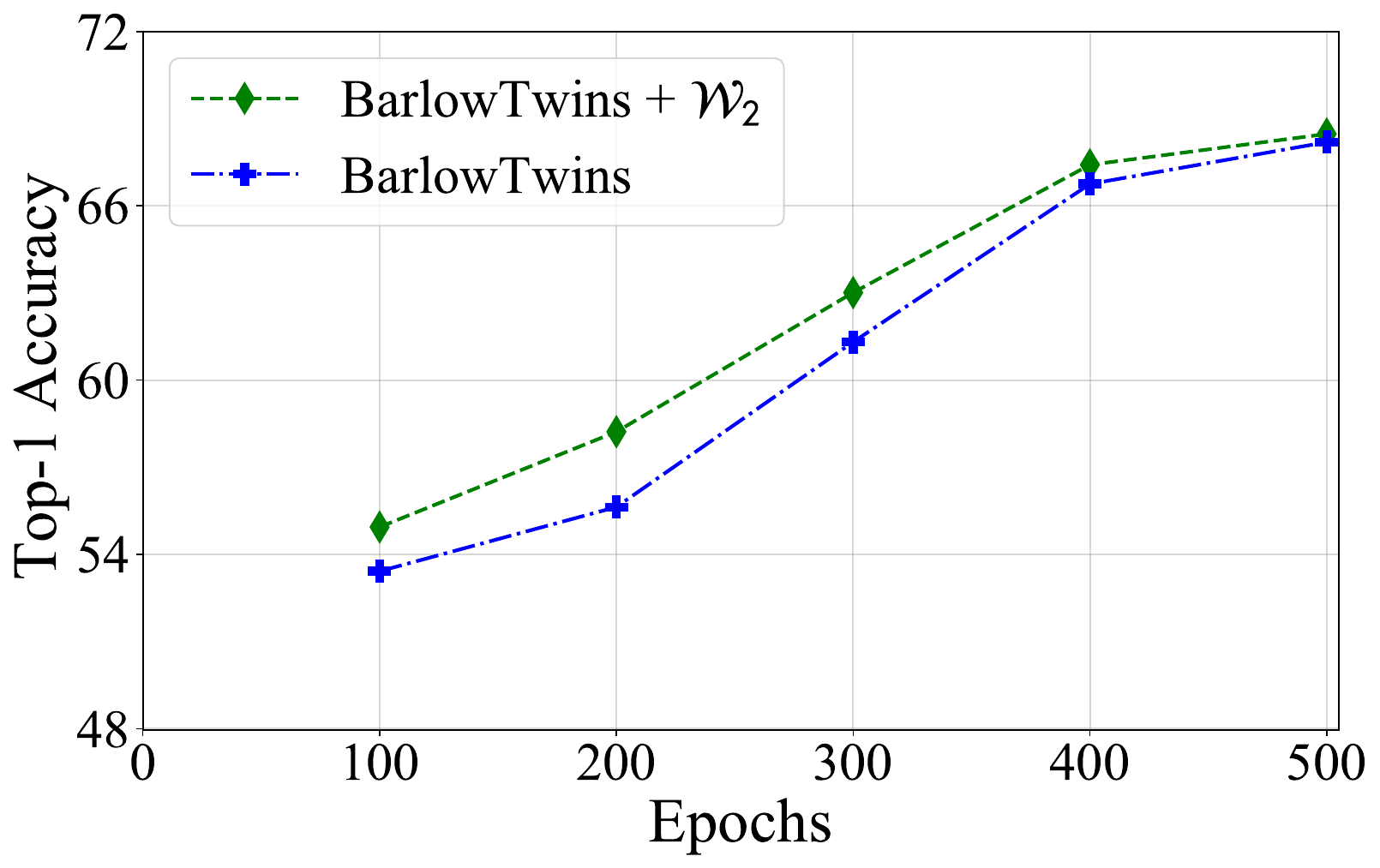}
	}
	\caption{{Convergence analysis for Top-1 accuracy during training.}}
	\label{fig:convergence}
        \vspace{-4mm}
\end{figure*}

\begin{figure*}[!h]
    \vspace{-4mm}
	\centering
	\subfigure[MoCo v2 on CIFAR-10]{ 
		\label{fig:moco v2 on cifar-10 (uniformity)}  
		\includegraphics[width=0.3\textwidth]{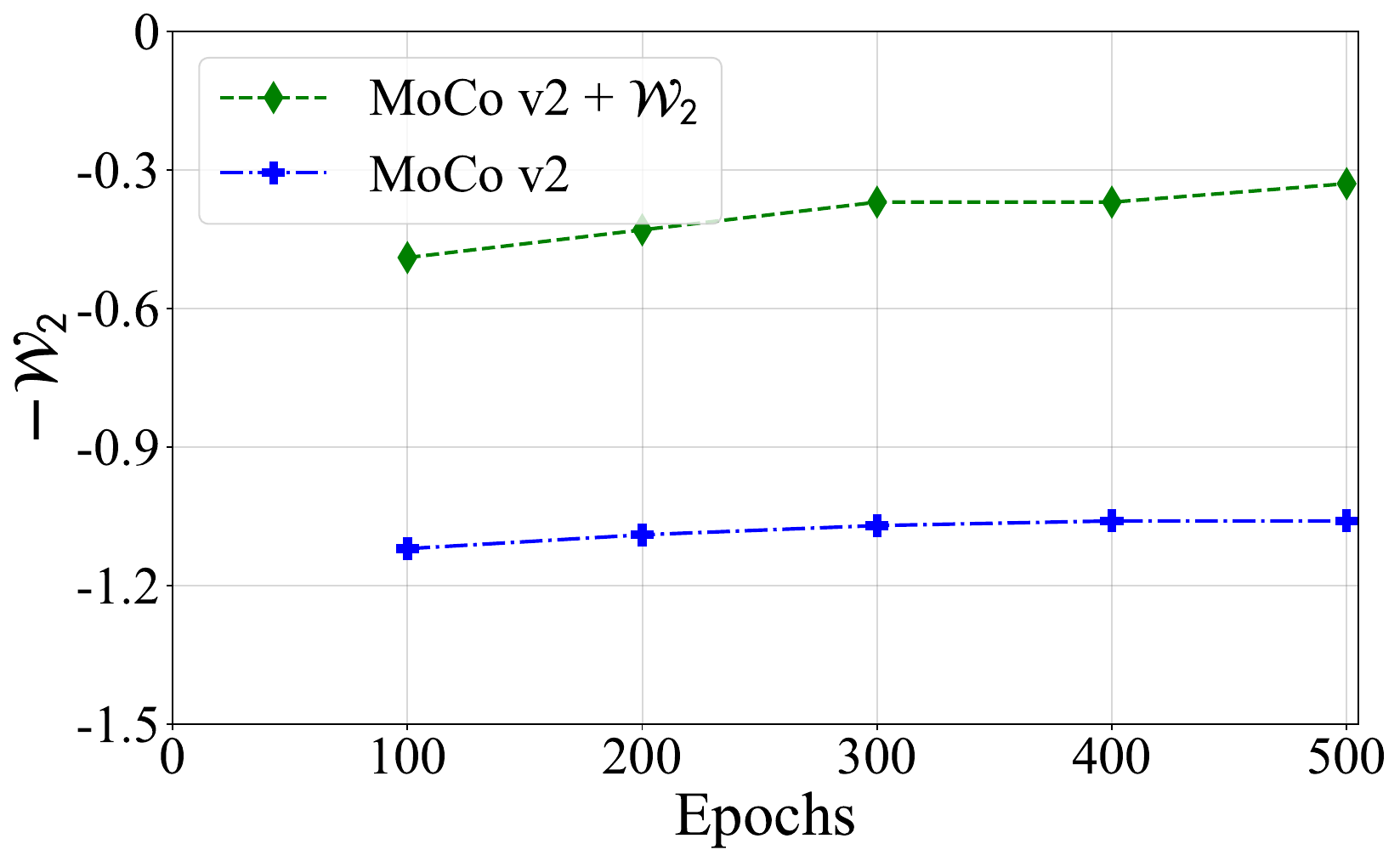}
	}
	\hspace{-0.2cm}
	\subfigure[BYOL on CIFAR-10]{
		\label{fig:byol on cifar-10 (uniformity)}
		\includegraphics[width=0.3\textwidth]{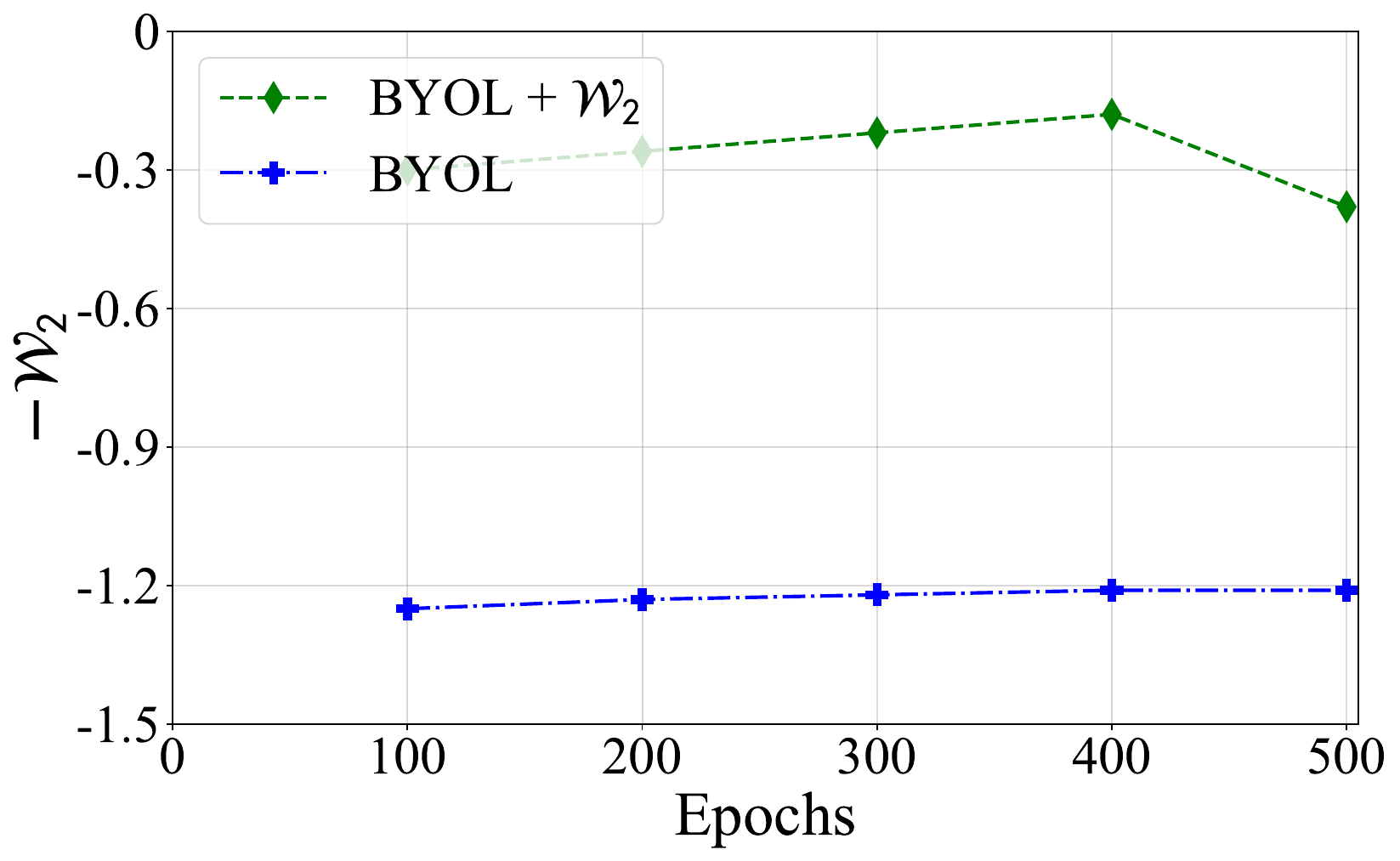}
	}
	\hspace{-0.2cm}
	\subfigure[BarlowTwins on CIFAR-10]{
		\label{fig:barlowtwins on cifar-10 (uniformity)}
		\includegraphics[width=0.3\textwidth]{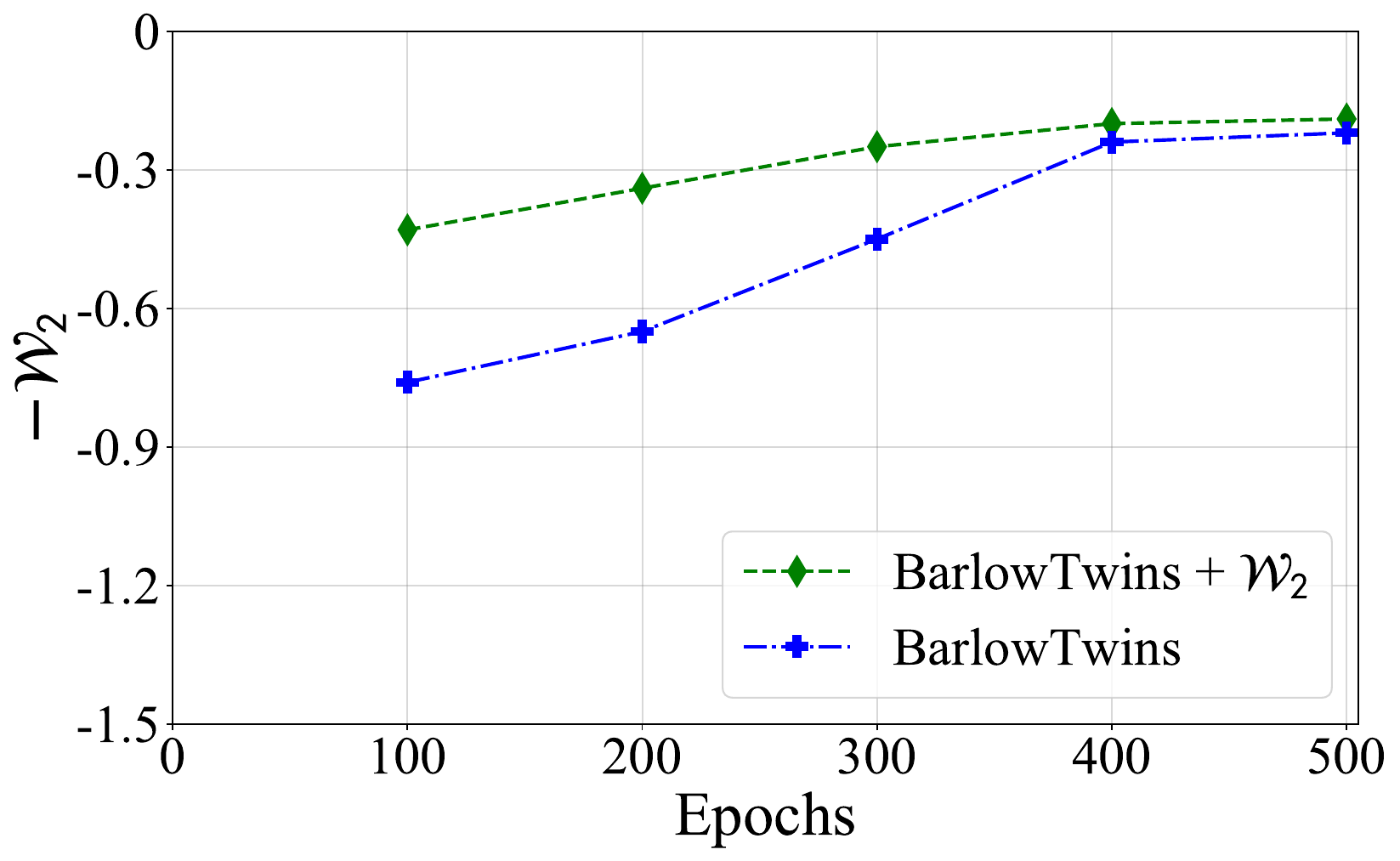}
	}
	\vspace{-4mm}
	\centering
	\subfigure[MoCo v2 on CIFAR-100]{ 
		\label{fig:moco v2 on cifar-100 (uniformity)}  
		\includegraphics[width=0.3\textwidth]{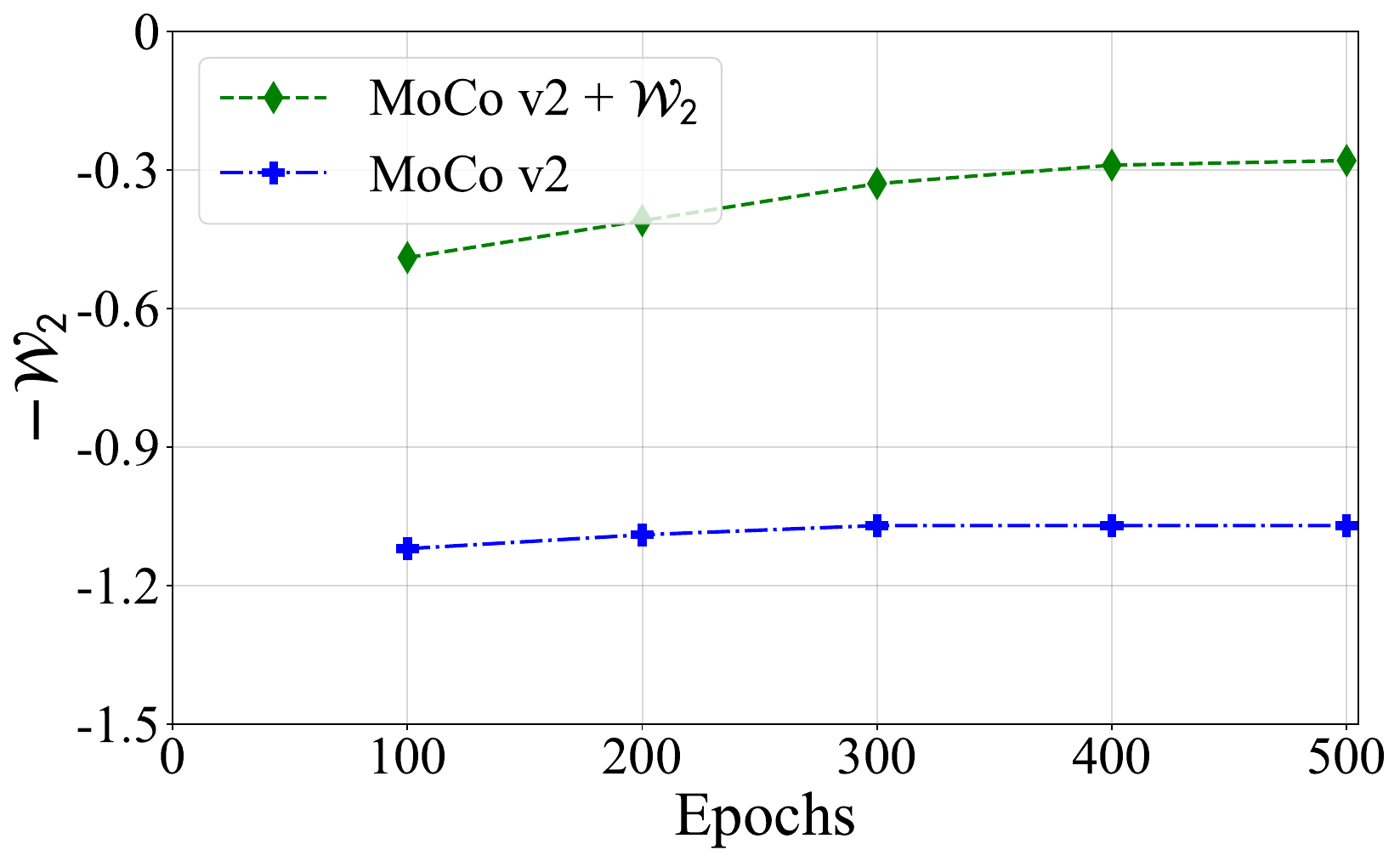}
	}
	\hspace{-0.2cm}
	\subfigure[BYOL on CIFAR-100]{
		\label{fig:byol on cifar-100 (uniformity)}
		\includegraphics[width=0.3\textwidth]{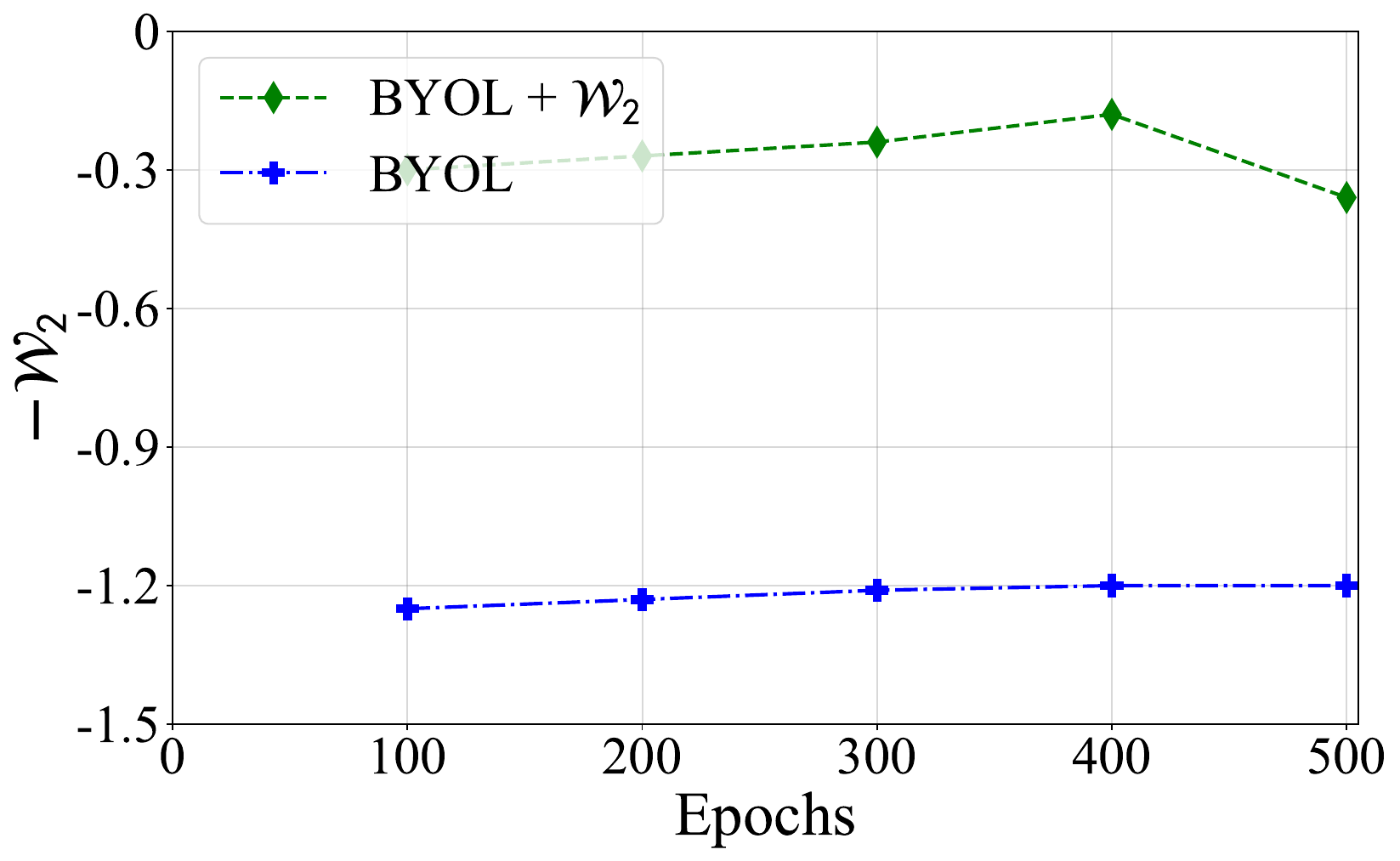}
	}
	\hspace{-0.2cm}
	\subfigure[BarlowTwins on CIFAR-100]{
		\label{fig:barlowtwins on cifar-100 (uniformity)}
		\includegraphics[width=0.3\textwidth]{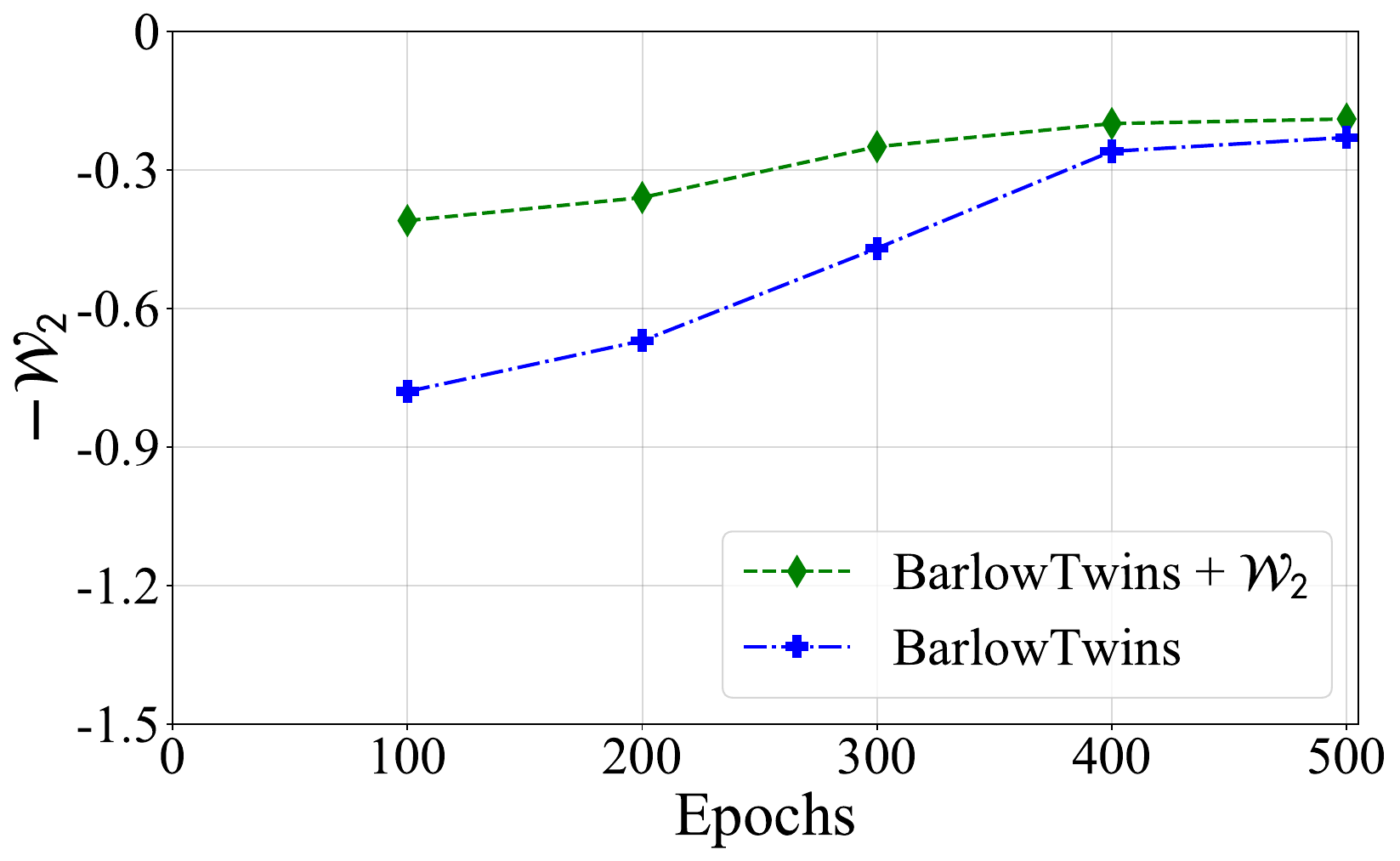}
	}
	\caption{{Visualizing  uniformity during training}}
	\label{fig:convergence on uniformity}
	\vspace{-4mm}
\end{figure*}

\begin{figure*}[!h]
    \vspace{-4mm}
	\centering
	\subfigure[MoCo v2 on CIFAR-10]{ 
		\label{fig:moco v2 on cifar-10 (alignment)}  
		\includegraphics[width=0.3\textwidth]{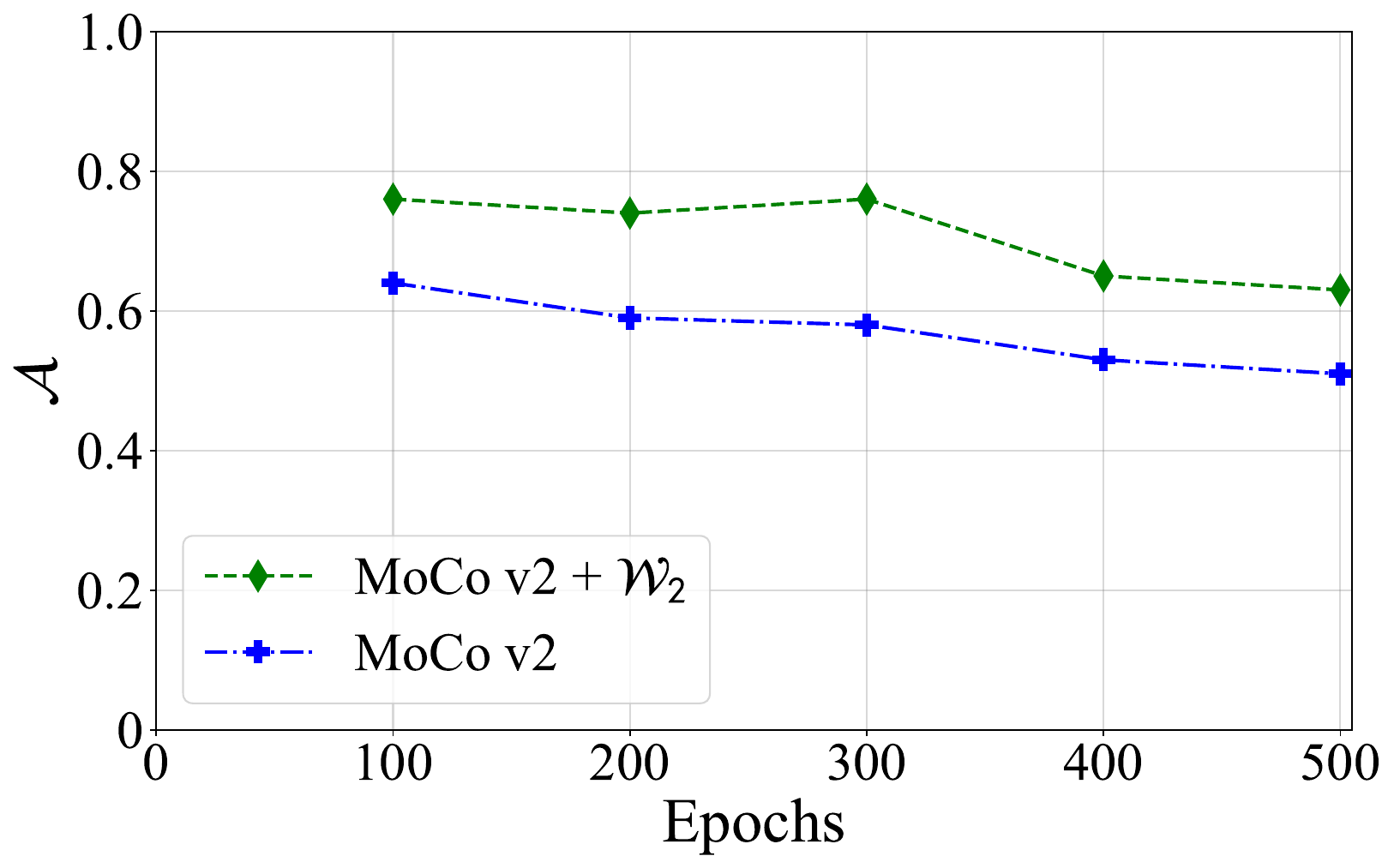}
	}
	\hspace{-0.2cm}
	\subfigure[BYOL on CIFAR-10]{
		\label{fig:byol on cifar-10 (alignment)}
		\includegraphics[width=0.3\textwidth]{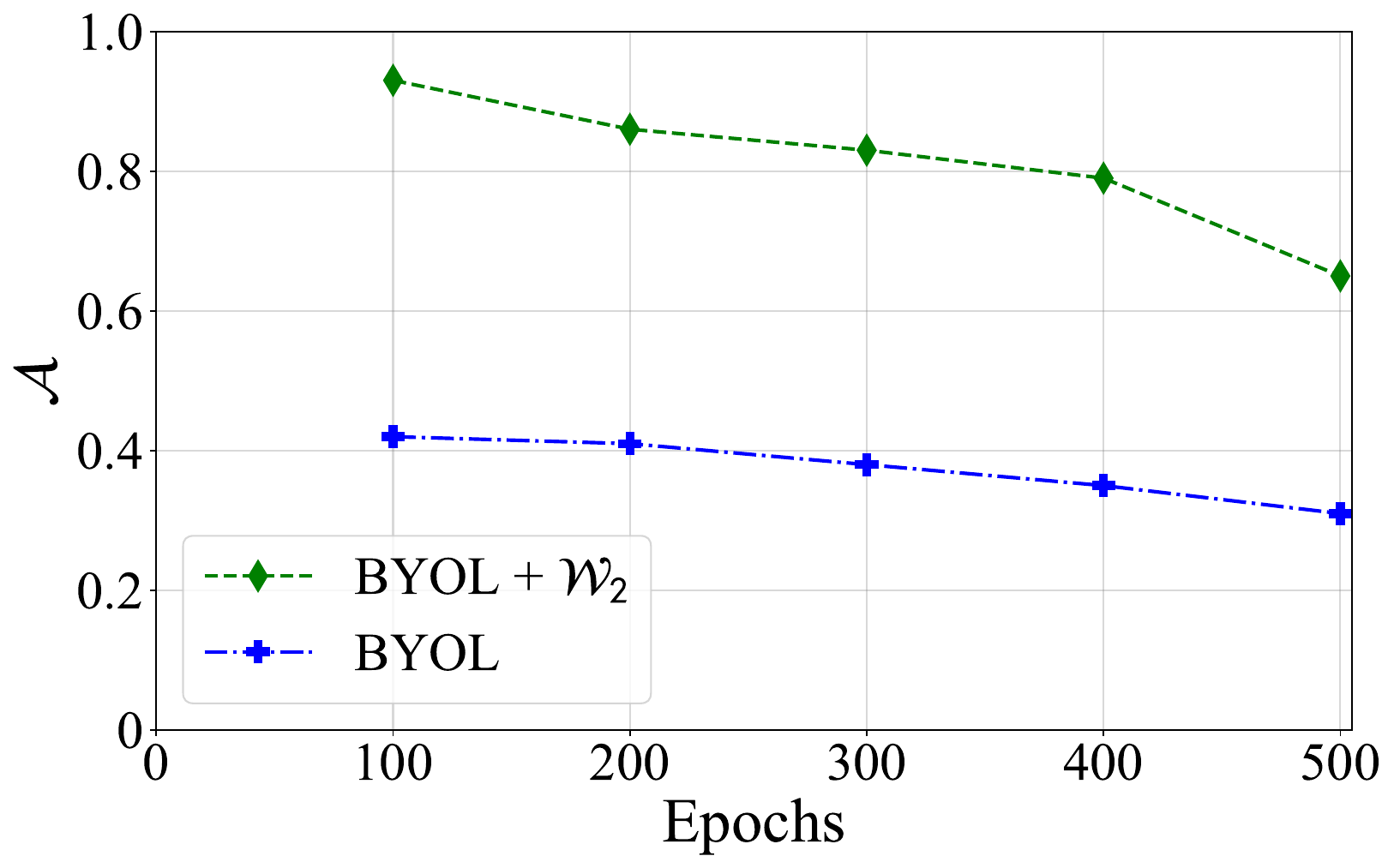}
	}
	\hspace{-0.2cm}
	\subfigure[BarlowTwins on CIFAR-10]{
		\label{fig:barlowtwins on cifar-10 (alignment)}
		\includegraphics[width=0.3\textwidth]{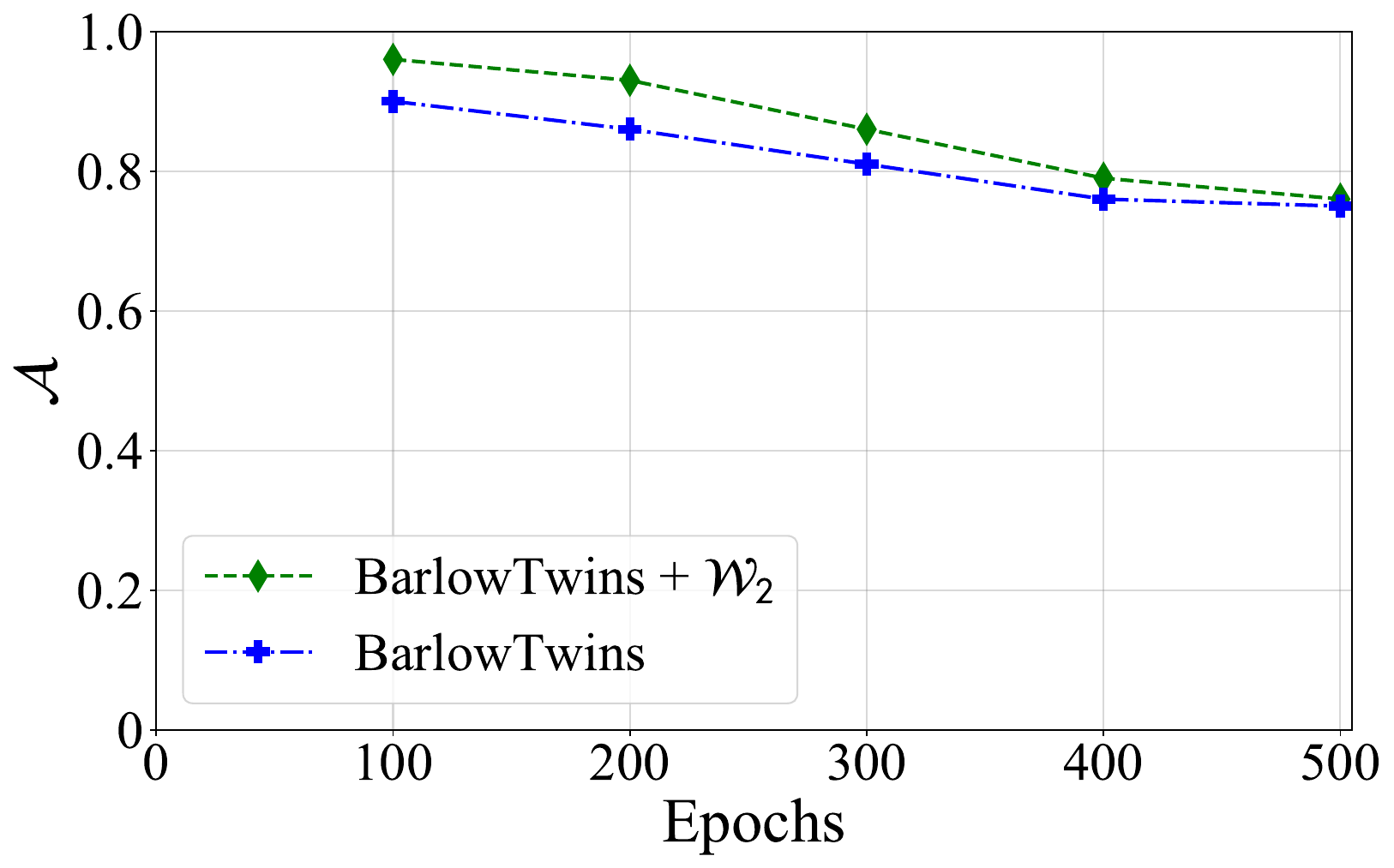}
	}
	\vspace{-4mm}
	\centering
	\subfigure[MoCo v2 on CIFAR-100]{ 
		\label{fig:moco v2 on cifar-100 (alignment)}  
		\includegraphics[width=0.3\textwidth]{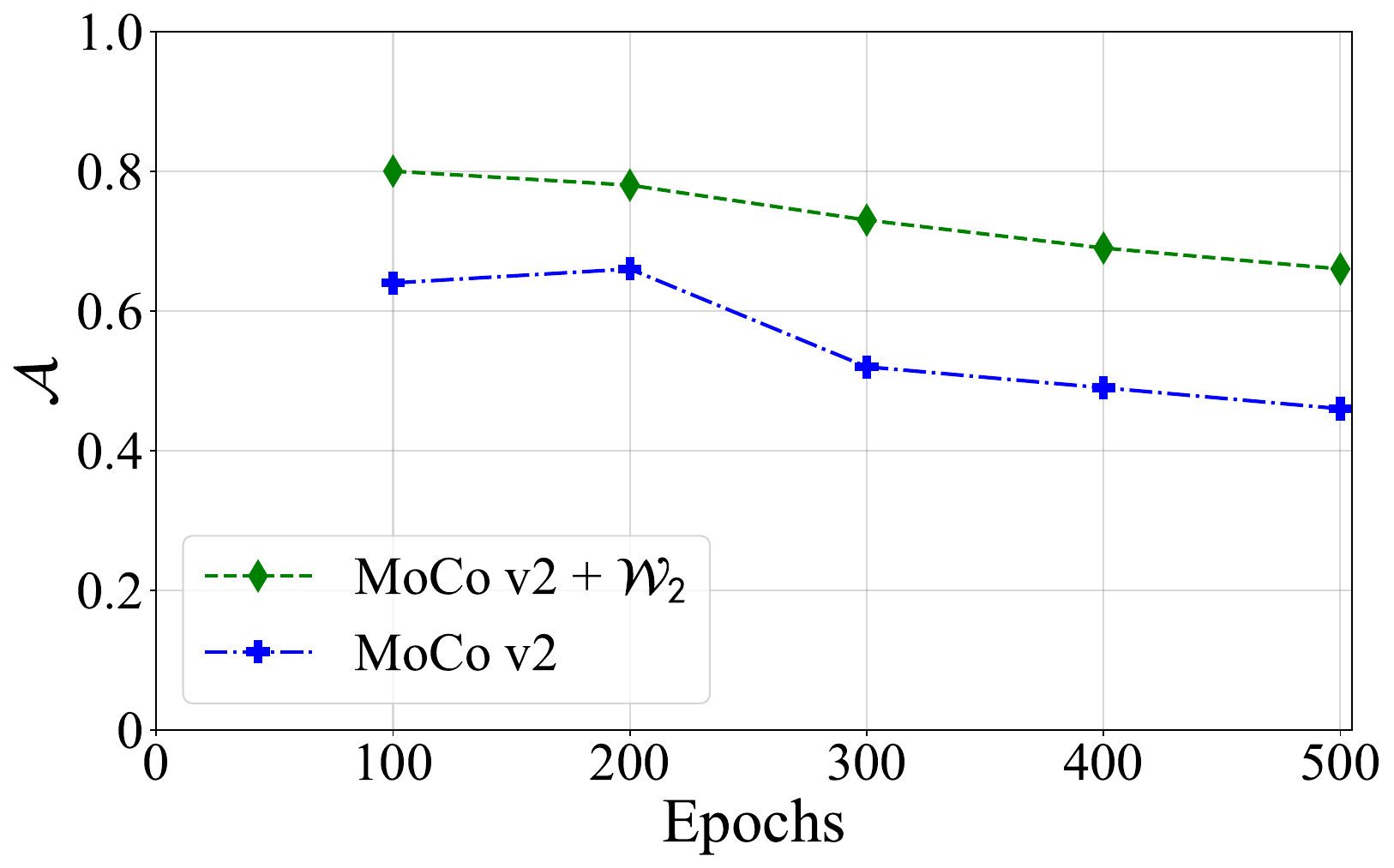}
	}
	\hspace{-0.2cm}
	\subfigure[BYOL on CIFAR-100]{
		\label{fig:byol on cifar-100 (alignment)}
		\includegraphics[width=0.3\textwidth]{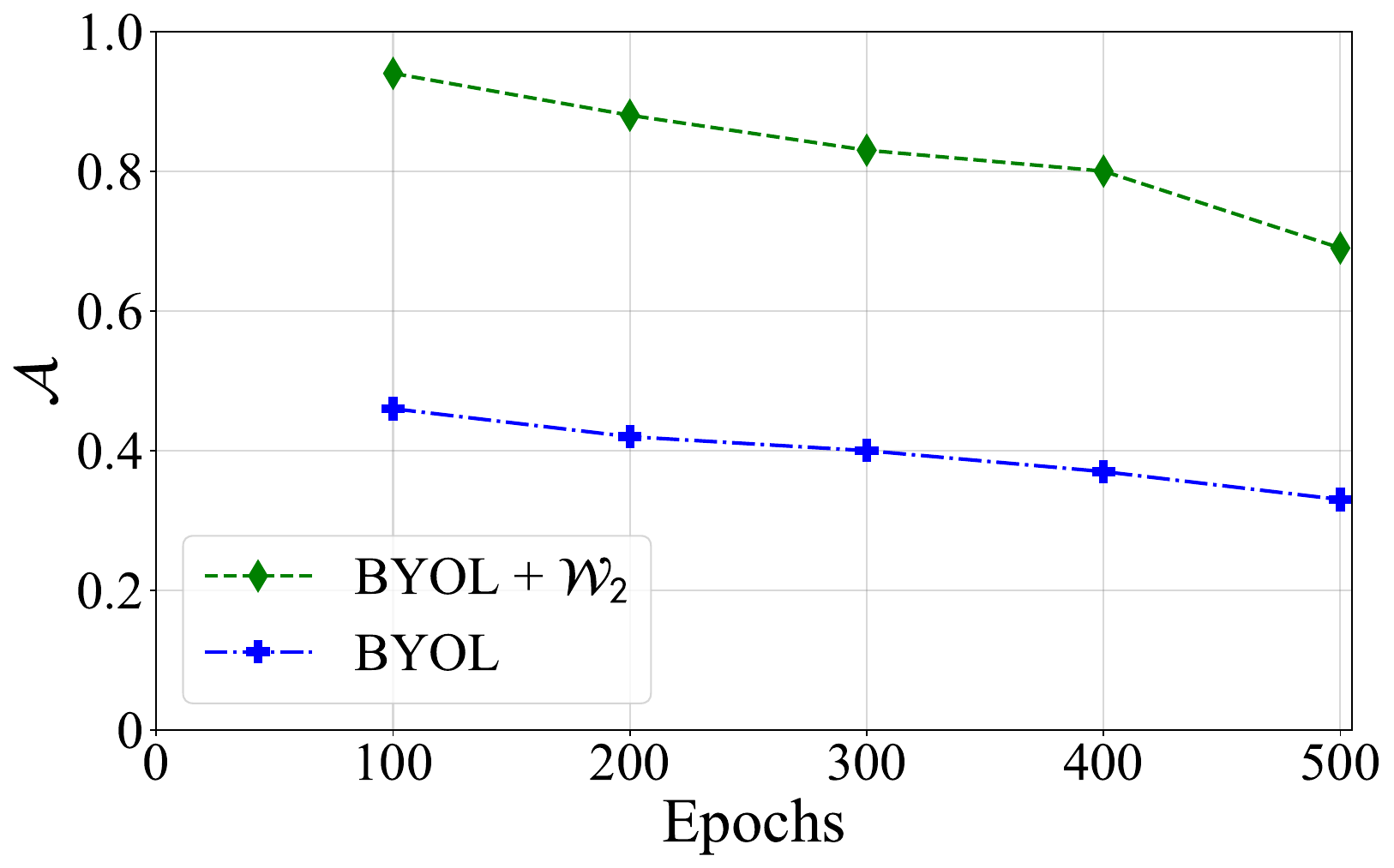}
	}
	\hspace{-0.2cm}
	\subfigure[BarlowTwins on CIFAR-100]{
		\label{fig:barlowtwins on cifar-100 (alignment)}
		\includegraphics[width=0.3\textwidth]{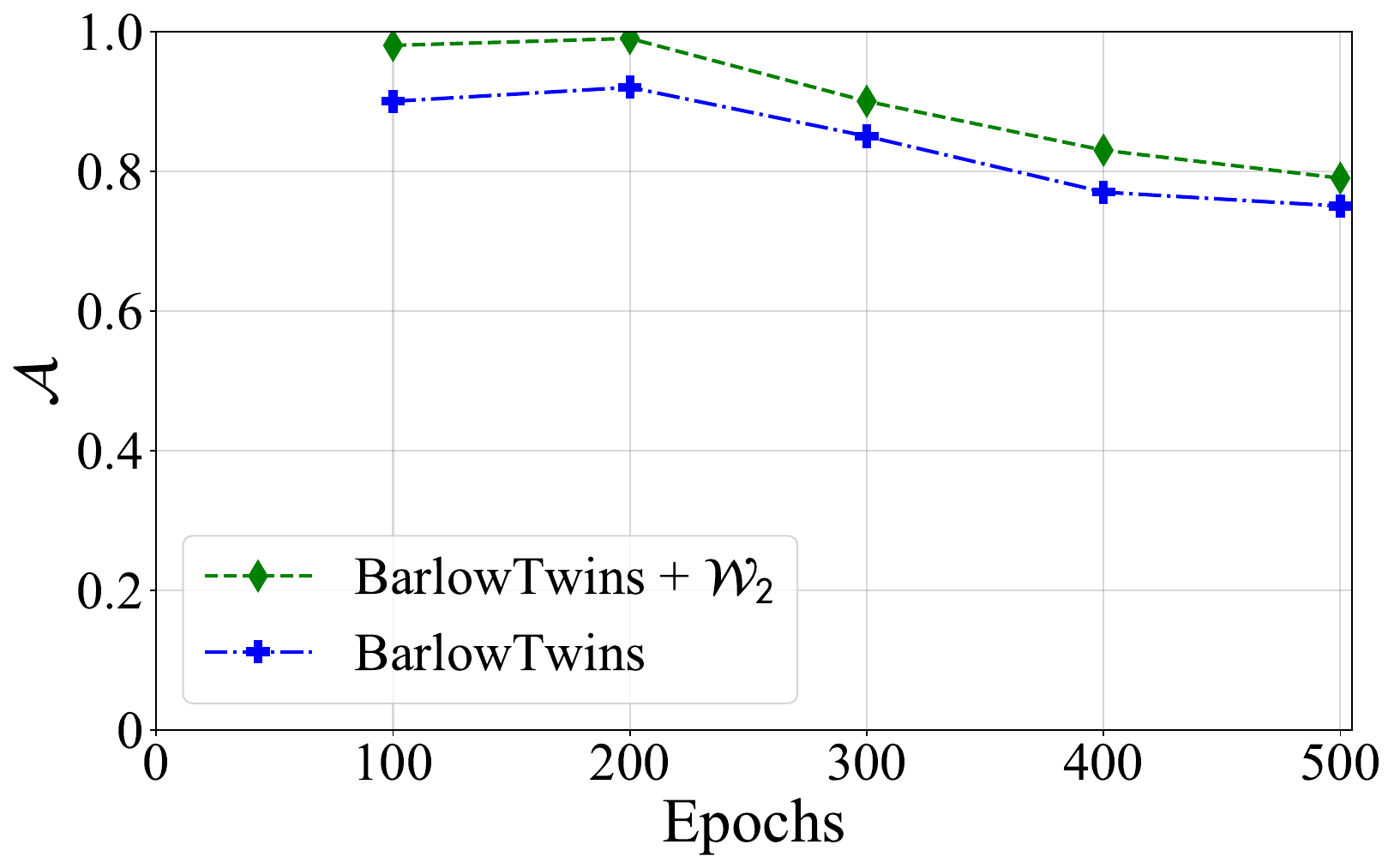}
	}
	\caption{{Visualizing alignment during training.}}
	\label{fig:convergence on alignment}
        \vspace{-4mm}
\end{figure*}

%% file: iclr2024_conference.bbl
\begin{thebibliography}{36}
\providecommand{\natexlab}[1]{#1}
\providecommand{\url}[1]{\texttt{#1}}
\expandafter\ifx\csname urlstyle\endcsname\relax
  \providecommand{\doi}[1]{doi: #1}\else
  \providecommand{\doi}{doi: \begingroup \urlstyle{rm}\Url}\fi

\bibitem[Arora et~al.(2019)Arora, Khandeparkar, Khodak, Plevrakis, and Saunshi]{Arora2019ATA}
Sanjeev Arora, Hrishikesh Khandeparkar, Mikhail Khodak, Orestis Plevrakis, and Nikunj Saunshi.
\newblock A theoretical analysis of contrastive unsupervised representation learning.
\newblock In \emph{ICML}, 2019.

\bibitem[Bardes et~al.(2022)Bardes, Ponce, and LeCun]{Bardes2021VICRegVR}
Adrien Bardes, Jean Ponce, and Yann LeCun.
\newblock Vicreg: Variance-invariance-covariance regularization for self-supervised learning.
\newblock In \emph{ICLR}, 2022.

\bibitem[Bhattacharyya(1943)]{Bhattacharyya1943OnAM}
A.~Bhattacharyya.
\newblock On a measure of divergence between two statistical populations defined by their probability distributions.
\newblock \emph{Bulletin of the Calcutta Mathematical Society}, 1943.

\bibitem[Caron et~al.(2020)Caron, Misra, Mairal, Goyal, Bojanowski, and Joulin]{Caron2020UnsupervisedLO}
Mathilde Caron, Ishan Misra, Julien Mairal, Priya Goyal, Piotr Bojanowski, and Armand Joulin.
\newblock Unsupervised learning of visual features by contrasting cluster assignments.
\newblock In \emph{NeurIPS}, 2020.

\bibitem[Caron et~al.(2021)Caron, Touvron, Misra, J'egou, Mairal, Bojanowski, and Joulin]{Caron2021EmergingPI}
Mathilde Caron, Hugo Touvron, Ishan Misra, Herv'e J'egou, Julien Mairal, Piotr Bojanowski, and Armand Joulin.
\newblock Emerging properties in self-supervised vision transformers.
\newblock In \emph{ICCV}, 2021.

\bibitem[Chandrasekaran et~al.(2012)Chandrasekaran, Recht, Parrilo, and Willsky]{chandrasekaran2012convex}
Venkat Chandrasekaran, Benjamin Recht, Pablo~A Parrilo, and Alan~S Willsky.
\newblock The convex geometry of linear inverse problems.
\newblock \emph{Foundations of Computational mathematics}, 12:\penalty0 805--849, 2012.

\bibitem[Chen et~al.(2020)Chen, Kornblith, Norouzi, and Hinton]{Chen2020ASF}
Ting Chen, Simon Kornblith, Mohammad Norouzi, and Geoffrey~E. Hinton.
\newblock A simple framework for contrastive learning of visual representations.
\newblock In \emph{ICML}, 2020.

\bibitem[Chen \& He(2021)Chen and He]{Chen2021ExploringSS}
Xinlei Chen and Kaiming He.
\newblock Exploring simple siamese representation learning.
\newblock In \emph{CVPR}, 2021.

\bibitem[Cohn \& Kumar(2007)Cohn and Kumar]{Cohn2006UniversallyOD}
Henry Cohn and Abhinav Kumar.
\newblock Universally optimal distribution of points on spheres.
\newblock \emph{Journal of the American Mathematical Society}, 2007.

\bibitem[da~Costa et~al.(2022)da~Costa, Fini, Nabi, Sebe, and Ricci]{Costa2022SololearnAL}
Victor Guilherme~Turrisi da~Costa, Enrico Fini, Moin Nabi, N.~Sebe, and Elisa Ricci.
\newblock Solo-learn: A library of self-supervised methods for visual representation learning.
\newblock \emph{JMLR}, 2022.

\bibitem[Dwibedi et~al.(2021)Dwibedi, Aytar, Tompson, Sermanet, and Zisserman]{Dwibedi2021WithAL}
Debidatta Dwibedi, Yusuf Aytar, Jonathan Tompson, Pierre Sermanet, and Andrew Zisserman.
\newblock With a little help from my friends: Nearest-neighbor contrastive learning of visual representations.
\newblock In \emph{ICCV}, 2021.

\bibitem[Gao et~al.(2021)Gao, Yao, and Chen]{Gao2021SimCSESC}
Tianyu Gao, Xingcheng Yao, and Danqi Chen.
\newblock Simcse: Simple contrastive learning of sentence embeddings.
\newblock In \emph{ArXiv}, 2021.

\bibitem[Grill et~al.(2020)Grill, Strub, Altch'e, Tallec, Richemond, Buchatskaya, Doersch, Pires, Guo, Azar, Piot, Kavukcuoglu, Munos, and Valko]{Grill2020BootstrapYO}
Jean-Bastien Grill, Florian Strub, Florent Altch'e, Corentin Tallec, Pierre~H. Richemond, Elena Buchatskaya, Carl Doersch, Bernardo~{\'A}vila Pires, Zhaohan~Daniel Guo, Mohammad~Gheshlaghi Azar, Bilal Piot, Koray Kavukcuoglu, R{\'e}mi Munos, and Michal Valko.
\newblock Bootstrap your own latent: A new approach to self-supervised learning.
\newblock In \emph{NeurIPS}, 2020.

\bibitem[Hadsell et~al.(2006)Hadsell, Chopra, and LeCun]{Hadsell2006DimensionalityRB}
Raia Hadsell, Sumit Chopra, and Yann LeCun.
\newblock Dimensionality reduction by learning an invariant mapping.
\newblock In \emph{CVPR}, 2006.

\bibitem[He et~al.(2016)He, Zhang, Ren, and Sun]{He2016DeepRL}
Kaiming He, X.~Zhang, Shaoqing Ren, and Jian Sun.
\newblock Deep residual learning for image recognition.
\newblock In \emph{CVPR}, 2016.

\bibitem[He et~al.(2020)He, Fan, Wu, Xie, and Girshick]{He2020MomentumCF}
Kaiming He, Haoqi Fan, Yuxin Wu, Saining Xie, and Ross~B. Girshick.
\newblock Momentum contrast for unsupervised visual representation learning.
\newblock In \emph{CVPR}, 2020.

\bibitem[Hinton et~al.(2015)Hinton, Vinyals, and Dean]{Hinton2015DistillingTK}
Geoffrey~E. Hinton, Oriol Vinyals, and Jeffrey Dean.
\newblock Distilling the knowledge in a neural network.
\newblock \emph{ArXiv}, abs/1503.02531, 2015.

\bibitem[Hua et~al.(2021)Hua, Wang, Xue, Ren, Wang, and Zhao]{hua2021feature}
Tianyu Hua, Wenxiao Wang, Zihui Xue, Sucheng Ren, Yue Wang, and Hang Zhao.
\newblock On feature decorrelation in self-supervised learning.
\newblock In \emph{ICCV}, 2021.

\bibitem[Jing et~al.(2022)Jing, Vincent, LeCun, and Tian]{Jing2021UnderstandingDC}
Li~Jing, Pascal Vincent, Yann LeCun, and Yuandong Tian.
\newblock Understanding dimensional collapse in contrastive self-supervised learning.
\newblock In \emph{ICLR}, 2022.

\bibitem[Li et~al.(2021)Li, Zhou, Xiong, Socher, and Hoi]{Li2021PrototypicalCL}
Junnan Li, Pan Zhou, Caiming Xiong, Richard Socher, and Steven C.~H. Hoi.
\newblock Prototypical contrastive learning of unsupervised representations.
\newblock In \emph{ICLR}, 2021.

\bibitem[Li et~al.(2022)Li, Liang, Zhao, Cui, Ouyang, Shao, Yu, and Yan]{li2022supervision}
Yangguang Li, Feng Liang, Lichen Zhao, Yufeng Cui, Wanli Ouyang, Jing Shao, Fengwei Yu, and Junjie Yan.
\newblock Supervision exists everywhere: A data efficient contrastive language-image pre-training paradigm.
\newblock In \emph{ICLR}, 2022.

\bibitem[Lindley \& Kullback(1959)Lindley and Kullback]{Lindley1959InformationTA}
David Lindley and Solomon Kullback.
\newblock Information theory and statistics.
\newblock \emph{Journal of the American Statistical Association}, 54:\penalty0 825, 1959.

\bibitem[Loshchilov \& Hutter(2017)Loshchilov and Hutter]{Loshchilov2017SGDRSG}
Ilya Loshchilov and Frank Hutter.
\newblock Sgdr: Stochastic gradient descent with warm restarts.
\newblock In \emph{ICLR}, 2017.

\bibitem[Olkin \& Pukelsheim(1982)Olkin and Pukelsheim]{Olkin1982TheDB}
Ingram Olkin and Friedrich Pukelsheim.
\newblock The distance between two random vectors with given dispersion matrices.
\newblock \emph{Linear Algebra and its Applications}, 48:\penalty0 257--263, 1982.

\bibitem[Radford et~al.(2021)Radford, Kim, Hallacy, Ramesh, Goh, Agarwal, Sastry, Askell, Mishkin, Clark, Krueger, and Sutskever]{Radford2021LearningTV}
Alec Radford, Jong~Wook Kim, Chris Hallacy, Aditya Ramesh, Gabriel Goh, Sandhini Agarwal, Girish Sastry, Amanda Askell, Pamela Mishkin, Jack Clark, Gretchen Krueger, and Ilya Sutskever.
\newblock Learning transferable visual models from natural language supervision.
\newblock In \emph{ICML}, 2021.

\bibitem[Tian et~al.(2021)Tian, Chen, and Ganguli]{Tian2021UnderstandingSL}
Yuandong Tian, Xinlei Chen, and Surya Ganguli.
\newblock Understanding self-supervised learning dynamics without contrastive pairs.
\newblock In \emph{ICML}, 2021.

\bibitem[Van~Handel(2016)]{van2016probability}
Ramon Van~Handel.
\newblock Probability in high dimension.
\newblock \emph{Lecture Notes (Princeton University)}, 2016.

\bibitem[Wang \& Isola(2020)Wang and Isola]{Wang2020UnderstandingCR}
Tongzhou Wang and Phillip Isola.
\newblock Understanding contrastive representation learning through alignment and uniformity on the hypersphere.
\newblock In \emph{ICML}, 2020.

\bibitem[Wang et~al.(2021)Wang, Zhang, Shen, Kong, and Li]{Wang2021DenseCL}
Xinlong Wang, Rufeng Zhang, Chunhua Shen, Tao Kong, and Lei Li.
\newblock Dense contrastive learning for self-supervised visual pre-training.
\newblock In \emph{CVPR}, 2021.

\bibitem[Xie et~al.(2021)Xie, Ding, Wang, Zhan, Xu, Li, and Luo]{Xie2021DetCoUC}
Enze Xie, Jian Ding, Wenhai Wang, Xiaohang Zhan, Hang Xu, Zhenguo Li, and Ping Luo.
\newblock Detco: Unsupervised contrastive learning for object detection.
\newblock In \emph{ICCV}, 2021.

\bibitem[Yang et~al.(2021)Yang, Wu, Zhou, and Lin]{Yang2021InstanceLF}
Ceyuan Yang, Zhirong Wu, Bolei Zhou, and Stephen Lin.
\newblock Instance localization for self-supervised detection pretraining.
\newblock In \emph{CVPR}, 2021.

\bibitem[You et~al.(2017)You, Gitman, and Ginsburg]{You2017ScalingSB}
Yang You, Igor Gitman, and Boris Ginsburg.
\newblock Scaling sgd batch size to 32k for imagenet training.
\newblock \emph{ArXiv}, 2017.

\bibitem[Zbontar et~al.(2021)Zbontar, Jing, Misra, LeCun, and Deny]{Zbontar2021BarlowTS}
Jure Zbontar, Li~Jing, Ishan Misra, Yann LeCun, and St{\'e}phane Deny.
\newblock Barlow twins: Self-supervised learning via redundancy reduction.
\newblock In \emph{ICML}, 2021.

\bibitem[Zhang et~al.(2022{\natexlab{a}})Zhang, Zhang, Zhang, Pham, Yoo, and Kweon]{Zhang2022HowDS}
Chaoning Zhang, Kang Zhang, Chenshuang Zhang, Trung~X. Pham, Chang~D. Yoo, and In~So Kweon.
\newblock How does simsiam avoid collapse without negative samples? a unified understanding with self-supervised contrastive learning.
\newblock In \emph{ICLR}, 2022{\natexlab{a}}.

\bibitem[Zhang et~al.(2022{\natexlab{b}})Zhang, Zhu, Yan, Zhao, and Yang]{zhang2022zerocl}
Shaofeng Zhang, Feng Zhu, Junchi Yan, Rui Zhao, and Xiaokang Yang.
\newblock Zero-{CL}: Instance and feature decorrelation for negative-free symmetric contrastive learning.
\newblock In \emph{ICLR}, 2022{\natexlab{b}}.

\bibitem[Zhao et~al.(2021)Zhao, Vemulapalli, Mansfield, Gong, Green, Shapira, and Wu]{Zhao2021ContrastiveLF}
Xiangyu Zhao, Raviteja Vemulapalli, P.~A. Mansfield, Boqing Gong, Bradley Green, Lior Shapira, and Ying Wu.
\newblock Contrastive learning for label efficient semantic segmentation.
\newblock In \emph{ICCV}, 2021.

\end{thebibliography}
